\documentclass[twoside]{article}

\usepackage[accepted]{aistats2026}
\usepackage[round]{natbib}

\usepackage{amsmath}
\usepackage{xcolor}
\usepackage{algorithm}
\usepackage{algpseudocode}   
\newcommand{\mathcolorbox}[2]{\colorbox{#1}{$\displaystyle #2$}}
\usepackage[utf8]{inputenc} 
\usepackage[T1]{fontenc}    
\usepackage{hyperref}       
\usepackage{url}            
\usepackage{booktabs}       
\usepackage{amsfonts}       
\usepackage{nicefrac}       
\usepackage{microtype}      
\usepackage{xcolor}         
\usepackage{enumitem}
\usepackage{enumitem}
\usepackage{microtype}
\usepackage{graphicx}
\usepackage{subfigure}
\usepackage{booktabs} 
\usepackage{xcolor}
\newcommand{\argmin}{\mathop{\arg\!\min}}
\usepackage{enumitem}

\usepackage{amssymb}
\usepackage{mathtools}
\usepackage{amsthm}
\usepackage{nicefrac}
\usepackage{bbding}
\usepackage{threeparttable}
\usepackage[capitalize,noabbrev]{cleveref}
\usepackage{caption}

\usepackage[capitalize,noabbrev]{cleveref}
\usepackage{nicefrac}
\usepackage{bbding}
\newtheorem{theorem}{Theorem}[section]

\newtheorem{lemma}[theorem]{Lemma}
\newtheorem{corollary}[theorem]{Corollary}

\newtheorem{definition}[theorem]{Definition}
\newtheorem{assumption}[theorem]{Assumption}

\newtheorem{remark}[theorem]{Remark}

\newcommand{\algname}[1]{{\color{black}\small \sf #1}}

\newcommand{\eqdef}{\stackrel{\text{def}}{=}}

\begin{document}

%
\runningtitle{Double Momentum for Byzantine Robust Learning}

%

\twocolumn[

\aistatstitle{Accelerating Byzantine-Robust Distributed Learning with Compressed Communication via Double Momentum and Variance Reduction}

\aistatsauthor{ Yanghao Li \And Changxin Liu\textsuperscript{\textdagger} \And  Yuhao Yi\textsuperscript{\textdagger}}

\aistatsaddress{Sichuan University \And  East China University of Science and Technology \And Sichuan University } ]

\begin{abstract}
In collaborative and distributed learning, Byzantine robustness reflects a major facet of optimization algorithms. Such distributed algorithms are often accompanied by transmitting a large number of parameters, so communication compression is essential for an effective solution. In this paper, we propose Byz-DM21, a novel Byzantine-robust and communication-efficient stochastic distributed learning algorithm. Our key innovation is a novel gradient estimator based on a double-momentum mechanism, integrating recent advancements in error feedback techniques. Using this estimator, we design both standard and accelerated algorithms that eliminate the need for large batch sizes while maintaining robustness against Byzantine workers. We prove that the Byz-DM21 algorithm has a smaller neighborhood size and converges to $\varepsilon$-stationary points in $\mathcal{O}(\varepsilon^{-4})$ iterations. To further enhance efficiency, we introduce a distributed variant called Byz-VR-DM21, which incorporates local variance reduction at each node to progressively eliminate variance from random approximations. We show that Byz-VR-DM21 provably converges to $\varepsilon$-stationary points in $\mathcal{O}(\varepsilon^{-3 })$ iterations. Additionally, we extend our results to the case where the functions satisfy the Polyak-{\L}ojasiewicz condition. Finally, numerical experiments demonstrate the effectiveness of the proposed method.

\end{abstract}
\vskip -0.5in
\section{Introduction}
\vskip -0.1in
\label{submission}

Distributed learning has garnered significant attention due to its widespread applications. Specific applications include large-scale training of deep neural networks, collaborative learning in edge computing, and distributed optimization in federated learning \cite{haddadpour2019trading,jaggi2014communication,lee2017distributed,yu2019linear}. Traditional distributed learning often assumes an environment free from faults or attacks. However, in certain real-world scenarios, such as edge computing \cite{shi2016edge} and federated learning \cite{brendan2017federated}, service providers (also known as the server) typically have limited control over computational nodes (also known as workers). In such cases, workers may experience various software and hardware failures \cite{xie2019zeno}. Worse still, some workers may be compromised by malicious third parties and intentionally send erroneous information to disrupt the distributed learning process \cite{kairouz2021advances}. Workers affected by such failures or attacks are referred to as \emph{Byzantine workers}\footnote{Following the standard terminology in the literature~\cite{lamport2019byzantine,su2016fault}, we refer to a worker as Byzantine if it may, either maliciously or unintentionally, send incorrect information to other workers or to the server. Such workers are assumed to be omniscient: they can access the vectors transmitted by other workers, are aware of the server-side aggregation rule, and may coordinate their actions with one another.}. Distributed learning in the presence of Byzantine workers, also known as Byzantine-Robust Distributed Learning (BRDL), has recently emerged as a prominent research topic \cite{yin2018byzantine,bernstein2018signsgd,diakonikolas2019recent,konstantinidis2021byzshield}. 

A common approach to achieving Byzantine robustness is to replace the standard mean aggregator with robust alternatives, such as Krum \cite{blanchard2017machine}, geometric median \cite{chen2017distributed}, coordinate-wise median \cite{yin2018byzantine}, and trimmed mean \cite{yin2018byzantine}, among others. However, in the presence of Byzantine workers, even robust aggregators inevitably introduce aggregation error, defined as the discrepancy between the aggregated result and the true mean.  Moreover, even with independent and identically distributed (i.i.d.) data, the aggregation error can be significant due to the high variance of stochastic gradients \cite{karimireddy2021learning}, which are typically sent from workers to the server for parameter updates.  Large aggregation errors may lead to the failure of BRDL methods \cite{xie2020fall}.

In addition to Byzantine robustness, the efficiency of distributed learning systems is considered a major performance metric. Due to its distributed nature, a bottleneck lies in the communication between workers and the server, particularly the transmission of local stochastic gradients. This challenge becomes more pronounced with high-dimensional models, resulting in substantial communication overhead.  To mitigate this issue, several strategies have been proposed, including reducing communication frequency by performing multiple local updates \cite{chen2018lag} and compressing the transmitted messages.  Common compression techniques include quantization, which encodes vectors using a limited number of bits \cite{alistarh2017qsgd}, and sparsification, which reduces the number of non-zero elements in transmitted vectors \cite{wangni2018gradient}.  In this work, we primarily focus on the latter approach.  Specifically, at each iteration, workers compress their local gradients before transmission, and the server aggregates these compressed gradients to update the model parameters.

Byzantine robustness and communication efficiency are both crucial properties in distributed learning, yet their simultaneous exploration has been relatively limited in the existing literature, with current methods facing notable challenges. \cite{zhu2021broadcast} proposed Byzantine-robust variants of compressed SGD (\algname{BR-CSGD}) and \algname{SAGA} (\algname{BR-CSAGA}), as well as \algname{BROADCAST}, which integrates \algname{DIANA} \cite{mishchenko2019distributed} with \algname{BR-CSAGA}. However, their convergence analysis is confined to strongly convex problems and relies on stringent assumptions. Similarly, \cite{gorbunov2023variance} studied the Byzantine-tolerant \algname{Byz-VR-MARINA}, which achieves fast convergence but occasionally triggers uncompressed message communication and full gradient computation. Moreover, most existing Byzantine-robust methods utilize unbiased compressors, whereas biased contractive compressors combined with error feedback often yield superior empirical performance \cite{rammal2024communication}.

In this work, we comprehensively address these limitations by building on the recently proposed Byzantine-robust stochastic distributed learning method with error feedback, \algname{Byz-EF21-SGDM} \cite{liu2026byzantine}. We first propose \algname{Byz-DM21}, a Byzantine-robust stochastic distributed learning algorithm that utilizes Double Momentum with Error Feedback-21, an enhanced variant of \algname{Byz-EF21-SGDM}, featuring improved convergence properties over the original method. A double momentum estimator $u_i^{(t)}$ has richer "memory" of the past gradients compared to \algname{SGDM} \cite{fatkhullin2023momentum}. Building on \algname{Byz-DM21}, we further introduce \algname{Byz-VR-DM21}, which incorporates local variance reduction at each node to progressively eliminate variance arising from stochastic approximations. We summarize our main contributions as follows:

\begin{itemize}[itemsep=0ex,leftmargin=1em]
    \item  We propose a novel Byzantine-robust and communication-efficient stochastic distributed learning method, \algname{Byz-DM21}. Our new algorithm is batch-free and employs a Double-Momentum mechanism to simultaneously suppress variance from stochastic gradients and bias introduced by compression, which leads to improved sample complexity for \algname{Byz-EF21-SGDM} in the non-asymptotic regime (see Remark \ref{remark:compare}). Moreover, we prove that the variance of the double-momentum estimator is strictly smaller than that of the single-momentum estimator, with an asymptotic reduction factor of approximately $1/2$ when $\eta$ is small. We also show that \algname{Byz-DM21} converges to an $\varepsilon$-stationary point in $\mathcal{O}(\varepsilon^{-4})$ iterations.
    \item We propose \algname{Byz-VR-DM21} as an extension to \algname{Byz-DM21}. The improved algorithm incorporates local variance reduction on all nodes to progressively eliminate variance from stochastic approximations. By combining worker momentum based variance reduction with a Byzantine robust aggregator, we obtain a faster Byzantine robust algorithm. We prove that \algname{Byz-VR-DM21} achieves an accelerated convergence to $\varepsilon$-stationary points in $\mathcal{O}(\varepsilon^{-3})$ iterations. Additionally, we analyze \algname{Byz-DM21} and \algname{Byz-VR-DM21} for problems that satisfy the Polyak-{\L}ojasiewicz condition \cite{polyak1963gradient}.
    \item We derive complexity bounds for \algname{Byz-DM21} under standard assumptions and further extend these results to \algname{Byz-VR-DM21}. These complexity bounds demonstrate that our algorithm outperforms the state-of-the-art, specifically in the full gradient setting \cite{rammal2024communication}, in terms of convergence speed. Notably, our results are tight and align with established lower bounds in both stochastic and full gradient scenarios when no Byzantine workers are present.
    \item Under the $\zeta^2$-heterogeneity assumption, our algorithm converges to a tighter neighborhood around the optimal solution and also matches the established lower bound \cite{karimireddy2021learning,allouah2023fixing}. For a detailed comparison, see Table \ref{table1}. Furthermore, experiments demonstrate that the proposed algorithm not only converges faster but also asymptotically reaches a model with a smaller error.

\end{itemize}

\begin{table*}[t]
	\centering
	\caption{Summary of related works on Byzantine-robust and communication-efficient distributed methods. 
    "Complexity (NC)" and "Complexity (P{\L})": represents the total number of communication rounds required for each worker to find $x$ such that $\mathbb{E}\left[\lVert \nabla f ({x}) \rVert \right] \leq \varepsilon$ in the general non-convex case and such $x$ that $\mathbb{E}[ f(x)-f(x^*)] \leq \varepsilon $ in P{\L} case respectively. 
		$\sigma^2$ represents the variance of local stochastic gradients, $\kappa$ refers to the parameter of robust aggregators, $\alpha \in (0,1]$ and $\omega \geq 0$ are parameters for biased contractive and unbiased compressors, respectively, $\zeta^2$ denotes the heterogeneity bound among honest workers, and $c$ denotes the heterogeneity constant. The parameter $p\in(0,1]$ is the sampling probability used in \algname{Byz-VR-MARINA} and \algname{Byz-DASHA-PAGE}. $m$ represents the local dataset size for workers in \algname{Byz-VR-MARINA},\algname{Byz-DASHA-PAGE} and \algname{BROADCAST}. }
	\label{table1}
   \resizebox{\linewidth}{!}{
	\begin{threeparttable}
		\centering
		\begin{tabular}{cccccc}
			\hline
			Method &
			Setting  &
			Batch-free? &
			Complexity (NC) & Complexity (P{\L})& Accuracy
			\\
			\hline
			\begin{tabular}{c}
				\algname{BROADCAST} \\
				\cite{zhu2021broadcast}   
			\end{tabular} &
   
			\begin{tabular}{c}
				{finite-sum} 
			\end{tabular}
			 &
			\XSolidBrush& - &$\frac{m^2{(1+\omega)^{\nicefrac{3}{2}}}G}{\mu^2(n-2B)} \tnote{\color{blue}(1)}$
			& { ${\kappa (1+\omega) \zeta^2}$} \\
   
			\hline
			\begin{tabular}{c}
				\algname{Byz-VR-MARINA} \tnote{\color{blue}(2)}\\
				\cite{gorbunov2023variance} 
			\end{tabular} &

			\begin{tabular}{c}
				
				finite-sum 
			\end{tabular}
			 &
			\XSolidBrush& 
			{ $\frac{\left( 1+ \sqrt{\max\{ \omega^2, {m\omega } \}} \left( \sqrt{\frac{1}{G}}+\sqrt{\kappa \max\{ \omega,m\}}  \right)\right)}{\varepsilon^2}$} & { $\frac{\left( 1+ \sqrt{\max\{ \omega^2, {m\omega } \}} \left( \sqrt{\frac{1}{G}}+\sqrt{\kappa \max\{ \omega,m\}}  \right)\right)+\mu(m+\omega)}{\mu}$} & $\frac{\kappa \zeta^2}{p-c\kappa}$ \\

            \hline
			\begin{tabular}{c}
				\algname{Byz-DASHA-PAGE} \tnote{\color{blue}(2)}\\
				\cite{rammal2024communication} 
			\end{tabular} &

			\begin{tabular}{c}
				
				finite-sum 
			\end{tabular}
			 &
			\XSolidBrush& 
			{ $\frac{\left( 1+ \left( \omega+\frac{\sqrt{m}}{\omega } \right) \left( \sqrt{\frac{1}{G}}+\sqrt{\kappa}  \right)\right)}{\varepsilon^2}$} & - & $\frac{\kappa \zeta^2}{1-c\kappa}$ \\
			\hline
			\begin{tabular}{c}
				\algname{Byz-EF21} \tnote{\color{blue}(2)}\\
				\cite{rammal2024communication} 
			\end{tabular}
			&
			\begin{tabular}{c}
				full\\
				gradient
			\end{tabular}  &
			\XSolidBrush &
			$\frac{1+\sqrt{\kappa}}{\alpha \varepsilon^2}$& - & {$\frac{(\kappa+\sqrt{\kappa})\zeta^2}{1-c(\kappa+\sqrt{\kappa})}$} \\

            \hline
			\begin{tabular}{c}
				\algname{Byz-EF21-SGDM} \\
				\cite{liu2026byzantine} 
			\end{tabular}
			&
			\begin{tabular}{c}
				stochastic \\
				gradient
			\end{tabular}  &
			\CheckmarkBold &
			$\frac{\sigma^2 }{G\varepsilon^4}+ \frac{\kappa\sigma^2 }{\varepsilon^4} $& - & {$\kappa\zeta^2$} \\
            
			\hline

			\begin{tabular}{c}
				\algname{Byz-DM21}\\
				(This work)
			\end{tabular} & 
			\begin{tabular}{c}
				stochastic \\
				gradient
			\end{tabular}
		
			&
			\CheckmarkBold &
			\begin{tabular}{c}
				$\frac{ \sqrt{\kappa+1}\sigma^2 }{G\varepsilon^4}+ \frac{(\kappa+1)^{\nicefrac{3}{2}} \sigma^2 }{\varepsilon^4} $ \\
				{$\frac{\sqrt{\kappa+1}}{\alpha \varepsilon^2} $} (full gradient)
			\end{tabular} & $\frac{(G(\kappa+1)+1)\sigma^2(\mu+\sqrt{\kappa+1})}{\mu^2\varepsilon G} $&$\kappa \zeta^2$
			\\
			\hline
   \begin{tabular}{c}
				\algname{Byz-VR-DM21}\\
				(This work)
			\end{tabular} & 
			\begin{tabular}{c}
				stochastic \\
				gradient
			\end{tabular}
			& 
			\CheckmarkBold &
			\begin{tabular}{c}
				$\frac{ \sqrt{\kappa+1}\sigma }{\sqrt{G}\varepsilon^3}+ \frac{(\kappa+1) \sigma }{\varepsilon^3} $ \\
				{$\frac{\sqrt{\kappa+1}}{\alpha \varepsilon^2} $} (full gradient)
			\end{tabular} & 
            $\frac{(G(\kappa+1)+1)\sigma^2}{\mu\varepsilon G}$&$\kappa \zeta^2$
			\\
			\hline
		\end{tabular}
		\begin{tablenotes} \item[{\color{blue} (1)}]The rate is derived under the strong convexity assumption.  Strong convexity implies the P{\L}-condition, but the converse is not true: there exist non-convex P{\L} functions \cite{karimi2016linear}.
			\item[{\color{blue} (2)}] For comparison, the complexity results of \algname{Byz-VR-MARINA} and \algname{Byz-EF21} are derived by exploring the relationship between $(\delta,c)$-agnostic robust aggregator and $(B,\kappa)$-robust aggregator. 
			See Remark \ref{remark:full_gradient} for details.
		\end{tablenotes}
		
\end{threeparttable}
}
\vskip -0.1in
\end{table*}
\vskip 0.1in
\section{Preliminaries}\label{preli}
\vskip -0.05in
 We consider a distributed learning system comprising a central server and $n$ workers, denoted as the set $[n] = \mathcal{G} \cup \mathcal{B}$. In this setup, $\mathcal{G}$ represents the subset of reliable or honest workers and $G = \left| \mathcal{G} \right|$, while $\mathcal{B}$ consists of malicious or Byzantine workers and $B = \left| \mathcal{B} \right|$. Notably, the identities of the honest workers and Byzantine workers are unknown beforehand. The Byzantine workers are assumed to be omniscient \cite{baruch2019little} and capable of colluding with each other to send arbitrary malicious messages to the server. The primary objective is to find the optimal solution to this \emph{distributed stochastic optimization problem}
\begin{equation}
\label{eq:original_P}
\min_{x \in \mathbb{R}^d} \Big\{ f(x) = \frac{1}{G} \sum_{i \in \mathcal{G}} f_i(x) \Big\},
\end{equation}
where $f_i(x) = \mathbb{E}_{\xi_i \sim \mathcal{D}_i} f_i(x,\xi_i), \forall i\in \mathcal{G}\,$.
 $x \in \mathbb{R}^d$ represents the model parameters to be optimized, while $f_i(x)$ denotes the (typically nonconvex) loss function of the model parameterized by $x$ on the dataset $\mathcal{D}_i$ held by client $i$. We allow the distributions of malicious nodes, $\mathcal{D}_1,\ldots,\mathcal{D}_n$, to vary arbitrarily. Our objective is to solve the optimization problem \eqref{eq:original_P} in the presence of arbitrary malicious messages sent by Byzantine workers, while ensuring communication efficiency. 

The following assumptions will be used throughout the analysis of our algorithms.
\begin{assumption}[{$L$-Smoothness}]\label{assump:smoothness}
We assume that function $f: \mathbb{R} ^d\rightarrow \mathbb{R}$ is $L$-smooth, meaning that for all $x, y \in \mathbb{R}^d$, the following inequality holds: 
\begin{equation}
    \lVert\nabla f_i(x) -\nabla f_i(y)\rVert \leq L_i\lVert x-y \rVert,
\end{equation}
and
\begin{equation}
    \lVert\nabla f(x) -\nabla f(y)\rVert \leq L\lVert x-y \rVert.
\end{equation}
\end{assumption}
\begin{assumption}[{Individual Smoothness}]\label{assump:individual_smoothness}
For each $i = 1,\ldots,n$, every realization of $\xi_i \sim \mathcal{D}_i$, the stochastic gradient $\nabla f_i(x,\xi_i)$ is $\ell_i$-$Lipschitz$, i.e., for all $x,y\in\mathbb{R}^d$, we have: 
    \begin{equation}
        \lVert\nabla f_i(x,\xi_i) -\nabla f_i(y,\xi_i)\rVert\leq\ell_i\lVert x-y \rVert.
    \end{equation}

We denote the averaged smoothness constants as $\widetilde L^2 = G^{-1}\sum_{i\in \mathcal{G}} L_i^2$ and $\widetilde \ell^2 = G^{-1}\sum_{i\in \mathcal{G}} \ell_i^2$. Finally, we assume that $f$ is lower bounded, i.e., $f^*:=\min_{x\in \mathbb{R}^d} f(x) \ge -\infty.$
\end{assumption}

In scenarios with arbitrary heterogeneity, distinguishing between regular and Byzantine workers becomes infeasible. Therefore, we adopt a common assumption regarding the heterogeneity of the gradient of local loss functions. 
\begin{assumption}[{$\zeta^2$}-heterogeneity]\label{assump:heterogeneity}
    We assume that good workers have $\zeta^2$-heterogeneous local loss functions for some $\zeta\geq 0$, i.e., 
    \begin{equation}
        \frac{1}{G}\sum_{i\in \mathcal{G}}\lVert \nabla f_i(x)-\nabla f(x)\rVert^2\leq \zeta^2, \forall x \in \mathbb{R}^d\,.
    \end{equation}
\end{assumption}

To model the stochastic noise introduced at each honest worker, we adopt the following assumption.
\begin{assumption}[{Bounded variance (BV)}]\label{assump:bound_variance}
        There exists $\sigma $\textgreater 0 such that 
        \begin{equation}
            \mathbb{E} [\lVert  \nabla f_i(x,\xi_i)-\nabla f_i(x)\rVert^2] \leq \sigma^2, \forall x \in \mathbb{R}^d\,,
        \end{equation}
        where $\mathbb{E}[\cdot]$ is defined over the randomness of the algorithm and $\xi_i \sim \mathcal{D}_i$ are i.i.d. random samples for each $i\in \mathcal{G}$.
\end{assumption}
We define a Byzantine-robust algorithm as an algorithm guaranteed to find an $\varepsilon$-approximate stationary point for $f_i$ despite the presence of $B$ Byzantine workers. In particular, we introduce the formal definition of Byzantine robustness as follows.
\begin{definition}[($B,\varepsilon$)-Byzantine robustness]\label{def:Byzantine_robustness}
    A learning algorithm is said $(B, \varepsilon)$-Byzantine robust if, even in the presence of $B$ Byzantine workers, it outputs $\hat{x}$ satisfying
    \begin{align}
        \mathbb{E}[\lVert \nabla f(\hat{x})\rVert^2]\leq\varepsilon\,.
    \end{align}
\end{definition}
Achieving ($B,\varepsilon$)-Byzantine robustness is generally not feasible for any $\varepsilon$ when the number of Byzantine workers $B$ is at least half the total workers ($B \geq n/2$). Therefore, in this work, we assume an upper bound on $B$, specifically $B < n/2$, to ensure robustness.

Several robust aggregation methods have been proposed, including the coordinate-wise trimmed mean (\algname{CWTM}) \cite{yin2018byzantine} and centered clipping \cite{karimireddy2021learning}. To formally assess the robustness of aggregation techniques, we adopt the concept of $(B, \kappa)$-robustness \cite{allouah2023fixing}. This property ensures that for any subset of inputs of size $G$, the output of the aggregation rule remains close to the average of those inputs. This concept serves as a useful metric for comparing the robustness of different aggregation methods. For a detailed analysis and formal quantification of commonly used aggregation rules, refer to \cite{allouah2023fixing}.
\begin{definition}[{$(B,\kappa)$-robustness}]
Given an integer $B< n/2$ and a real number $\kappa \geq 0$, an aggregation rule $F$ is $(B,\kappa)$-robust if for any set of $n$ vectors $\{g_1,g_2,\dots, g_n\}$, 
and any subset $S\subseteq [n]$ with $\left|S\right| = G$,
\begin{equation}
	\lVert{F(g_1,g_2,\dots, g_n) - \overline{g}_S}\rVert^2 \leq \frac{\kappa}{{\left|S\right|}}\sum_{i\in S}\lVert{g_i - \overline{g}_S}\rVert^2\,,
\end{equation}
where $\overline{g}_S := {G^{-1}}\sum_{i\in S}g_i.$
\end{definition}
To facilitate communication-efficient learning, we introduce compression techniques for message transmission. A general compression operator is defined as follows.

\begin{definition}[ Contractive compressors]
\label{def:contractive}
A (possibly randomized) mapping $\mathcal{C}:\mathbb{R}^d\rightarrow \mathbb{R}^d$ is called a contractive compression operator if there exists a constant $\alpha \in (0,1]$ such that 
\begin{equation}
	\mathbb{E} [ \lVert \mathcal{C}(x)-x \rVert^2] \leq (1-\alpha) \lVert x \rVert^2, \,\, \forall x\in\mathbb{R}^d.
\end{equation}
\end{definition}
In this paper, we focus on the $\text{Top}_k$ sparsification-based compression operator. This operator sorts the input vector, selects the largest $k$ elements, and transmits these elements along with their original indices. Another commonly used sparsification operator is $\text{Rand}_k$, which randomly sets 
$d-k$ elements of the input vector to zero, where $k$ is an integer between 1 and $d$. This approach achieves a sparsity ratio of $d/k$. For a detailed overview of both biased and unbiased compressors, please refer to \cite{beznosikov2023biased}.

\begin{algorithm*}[t]
\caption{\algname{Byzantine-DM21}}
\label{alg:SGD2M}
\textbf{Input}: initial model $x^{(0)}$, step-size $\gamma>0$, momentum coefficient $\eta\in(0,1]$, robust aggregation $F$, initial batch size $b$, the number of rounds $T$\\
\textbf{Initialization}: 
for every honest worker $i\in \mathcal{G}$, $v_i^{(0)}=u_i^{(0)} = g_i^{(0)} = \nabla f_i(x^{(0)},\xi_{i}^{(0)})$, each worker $i\in[n]$ sends $g_i^{(0)}$ to the server, $g^{(0)} = F(\{g_1^{(0)}, \ldots, g_n^{(0)}\})$
\begin{algorithmic}[1] 
\For{$t= 0,1,\ldots, T-1$}
\State Server computes
$x^{(t+1)} = x^{(t)} - \gamma g^{(t)}$ and broadcasts $x^{(t+1)}$ to all workers
\For{every honest worker $i\in \mathcal{G}$ in parallel}
\State Compute the first momentum estimator:
\State $v_i^{(t+1)}\!\!=\!\!
\begin{cases}
        (1 - \eta) v_i^{(t)} +  \eta \nabla f_i(x^{(t+1)}, \xi_i^{(t+1)}) \hfill \textnormal{\algname{(Byz-DM21)}}\\
     \nabla f_i(x^{(t+1)}, \xi_i^{(t+1)}) + (1-\eta)(v_i^{(t)}-\nabla f_i(x^{(t)}, \xi_i^{(t+1)})) \quad\quad\quad\hfill \textnormal{\algname{(Byz-VR-DM21)}}
\end{cases}$

\State Compute the second momentum estimator:
 $ u_i^{(t+1)} = (1 - \eta) u_i^{(t)} + \eta(v_i^{(t+1)}) $
\State Compress $c_i^{(t+1)} = \mathcal{C}(u_i^{(t+1)}- g_i^{(t)})$ and send $c_i^{(t+1)}$ to the server
\State Update local state 
$  g_i^{(t+1)} = g_i^{(t)}  + c_i^{(t+1)} $
\EndFor
\State Server updates $g^{(t+1)}= F(\{g_1^{(t+1)}, \ldots, g_n^{(t+1)}\}) $ via $g_i^{(t+1)} = g_i^{(t)}  + c_i^{(t+1)}, \forall i\in [n]$
\EndFor
\State
\textbf{Return}: A random $\hat{x}^{(T)}$ from $\{x^{(t)}\}_{t=0}^{T-1}$
\end{algorithmic}
\end{algorithm*}

\section{Byzantine-Robust Distributed Learning}\label{DM21}
In this section, we introduce our main results on methods employing biased compression.

\subsection{Byzantine-DM21} 
\vskip -0.05in
We summarize our new method \algname{Byz-DM21} in Algorithm \ref{alg:SGD2M}. At each iteration of \algname{Byz-DM21}, the model parameter $x^{(t+1)}$ is updated by the parameter server using the formula: $x^{(t+1)} = x^{(t)} - \gamma g^{(t)}$,  where $\gamma$ represents the step size and $g^{(t)}$ is the gradient estimator received from the parameter server. Following the update, the server broadcasts the updated model $x^{(t+1)}$ to all workers. Upon receiving $x^{(t+1)}$, the good workers proceed with the following steps: Update their first momentum estimate $v_i^{(t+1)}$ and the second momentum estimate $u_i^{(t+1)}$(lines 5 and 6). The change in $g_i^{(t)}$ is then compressed using the expression: $c_i^{(t+1)} = \mathcal{C}(u_i^{(t+1)}- g_i^{(t)})$, these compressed vectors are sent back to the server (line 7). Simultaneously, the honest workers update their local state according to: $g_i^{(t+1)} = g_i^{(t)}  + c_i^{(t+1)}$ reflecting the adjustments made due to compression (line 8). After that, the server gathers the results of computations from the workers and applies a $(B, \kappa)$-robust aggregator to compute the next estimator $g^{(t+1)}$.

Two crucial aspects of the \algname{Byz-DM21} are introduced as follows.

\emph{First}, honest workers transmit only compressed updates of their local momentum variables and estimated momentum variables to the server. This approach not only reduces communication overhead but also helps to identify and exclude Byzantine workers who attempt to subvert the algorithm by transmitting dense, inconsistent vectors. \emph{Second}, the double momentum method enhances the efficiency of the optimization process by combining first and second momentum estimations. On the one hand, the first momentum helps smooth gradient fluctuations, reduces the influence of noise, and accelerates convergence. On the other hand, the second momentum further optimizes the update direction, prevents gradient oscillations, and improves local convergence. This double momentum mechanism not only accelerates the training process but also reduces the impact of gradient noise and enhances robustness.

\subsection{Convergence of Byz-DM21 for General Non-Convex Problems}

\vskip -0.05in
The convergence analysis of \algname{Byz-DM21} hinges on the monotonicity of the following Lyapunov function:
\begin{align}
\label{eq:lyapunov}
    \Phi^{(t)}&=\delta_{t}+\frac{4\gamma}{\eta}\lVert\widetilde M^{(t)}\rVert^2+( \frac{48\gamma(8\kappa+1)(\eta^4+6\eta^2)}{\eta\alpha^2 G}+\nonumber\\
    &\quad\frac{8\gamma(28\kappa+3)}{\eta G}) \sum_{i\in \mathcal{G}} \lVert M_i^{(t)}\rVert^2,
\end{align}
where $\delta_{t}= f(x^{(t)})-f(x^*), M_i^{(t)} = v_i^{(t)}-\nabla f_i(x^{(t)})$, and $\widetilde M^{(t)}=G^{-1}\sum_{i\in\mathcal{G}}v_i^{(t)}-\nabla f_\mathcal{G}(x^{(t)})$. The primary convergence result for general non-convex functions is articulated in Theorem \ref{thm:rate}, with the corresponding proof detailed in Appendix \ref{proof_DM21}.
\begin{theorem}\label{thm:rate}
Assuming that Assumptions \ref{assump:smoothness}, \ref{assump:heterogeneity}, and \ref{assump:bound_variance} hold, we consider Algorithm \ref{alg:SGD2M} for solving the distributed learning problem \eqref{eq:original_P} with $B $\textless${n/2}$ Byzantine workers and communication compression characterized by the parameter $\alpha \in \left(0,1\right]$ as per Definition \ref{def:contractive}. If $\eta \leq 1$, and $\gamma \leq \min \Big\{  \frac{\alpha}{4\widetilde L\sqrt{234(8\kappa+1)}},\frac{\eta}{4\sqrt{(56\kappa+6)\widetilde L^2+L^2}} \Big\}$, then 
    \begin{equation*}
    \label{eq:bound}
    \begin{split}
        &\mathbb{E}\Big[\lVert \nabla f(\hat{x}^{(T)})\rVert^2\Big]\leq \frac{\Phi^{(0)}}{\gamma T}+(\frac{48(\eta^5+7\eta^3)(8\kappa+1)}{\alpha^2 }\\
        &+ 4\eta(64\kappa+7) +\frac{4\eta}{G}+\frac{8(8\kappa+1)\eta^4}{\alpha})\sigma^2+32\kappa \zeta^2,
    \end{split}
    \end{equation*}
    where $\hat{x}^{(T)}$ is sampled uniformly at random from the iterations of the method, $\Phi^{(0)}$ is defined in \eqref{eq:lyapunov}. Then choice $\eta \leq \mathcal{O}\Big(\min \Big\{    \Big( \frac{\alpha^2\delta_0 \widehat{L}}{(\kappa+1) \sigma^2 T}\Big)^{\nicefrac{1}{6}},\Big( \frac{\alpha^2\delta_0 \widehat{L}}{(\kappa+1) \sigma^2 T}\Big)^{\nicefrac{1}{4}},
    $
    $
    \Big( \frac{\delta_0 \widehat{L} }{(\kappa+1) \sigma^2 T}\Big)^{\nicefrac{1}{2}},
    $
    $
     \Big( \frac{\delta_0 G \widehat{L}}{ \sigma^2 T}\Big)^{\nicefrac{1}{2}},
     $
     $
     \Big( \frac{\alpha \widehat{L}\delta_0 }{(\kappa+1) \sigma^2 T}\Big)^{\nicefrac{1}{5}}   \Big\}\Big)$, where $\widehat{L}\eqdef8\sqrt{(56\kappa+6)\widetilde L^2+L^2} $, we obtain
     \begin{align*}
             &\mathbb{E}\Big[\lVert \nabla f(\hat{x}^{(T)})\rVert^2\Big]= \mathcal{O}\Big(( \frac{(\kappa+1)^{\nicefrac{1}{5}}\sigma^{\nicefrac{2}{5}}\widehat{L}\delta_0}{ \alpha^{\nicefrac{2}{5}} T})^{\nicefrac{5}{6}}+ \kappa \zeta^2\\
             &\quad+( \frac{(\kappa+1)^{\nicefrac{1}{3}}\sigma^{\nicefrac{2}{3}}\widehat{L}\delta_0 }{ \alpha^{\nicefrac{2}{3}} T})^{\nicefrac{3}{4}}+( \frac{(\kappa+1)\sigma^2 \widehat{L}\delta_0 }{  T})^{\nicefrac{1}{2}}\\
             &\quad+( \frac{(\kappa+1)^{\nicefrac{1}{4}}\sigma^{\nicefrac{1}{2}}\widehat{L}\delta_0 }{ \alpha^{\nicefrac{1}{4}}T})^{\nicefrac{4}{5}}+\frac{\Phi^{(0)}}{\gamma T}+( \frac{\sigma^2 \widehat{L}\delta_0 }{ G T})^{\nicefrac{1}{2}}\Big).
         \end{align*}
\end{theorem}
Theorem \ref{thm:rate} establishes that the proposed algorithm, \algname{Byz-DM21}, converges in $\mathbb{E}[\lVert \nabla f(\hat{x}^T)\rVert^2]$, adhering to the Byzantine robustness criterion specified in Definition \ref{def:Byzantine_robustness}. Consequently, \algname{Byz-DM21} achieves a convergence rate comparable to that of standard \algname{SGD} with momentum \cite{cutkosky2020momentum}. Furthermore, prior studies \cite{cutkosky2020momentum,arjevani2023lower} have shown that, under standard Assumptions, the convergence rate of $\mathcal{O}(\nicefrac{1}{T^{\nicefrac{1}{2}}})$ is optimal for \algname{SGD}. Although \algname{Byz-VR-MARINA} \cite{gorbunov2023variance} attains a faster convergence rate by intermittently using full gradients, this approach incurs significant computational costs, particularly in real-world scenarios involving large-scale training data.

\begin{remark}\label{remark:compare}
Omitting constants and higher-order terms, Theorem \ref{thm:rate} gives $\mathbb{E}[\lVert \nabla f(\hat{x}^{(T)})\rVert^2]\lesssim \frac{\sigma}{\sqrt{GT}}+\frac{\sqrt{\kappa}\sigma}{\sqrt{T}}$. When there is no Byzantine adversary, i.e., $B=0$, employing the standard mean aggregator (which is (0,0)-robust) yields a convergence rate of $\mathcal{O}\left(\tfrac{\sigma}{\sqrt{GT}}\right)$, which improves with the number of workers $G$. 
Additionally, \algname{Byz-DM21} attains better sample complexity than \algname{Byz-EF21-SGDM} because its bound in Theorem \eqref{eq:bound} involves $\eta^4/\alpha$, which is more favorable compared to the terms $\eta^2/\alpha$ in \algname{Byz-EF21-SGDM}. Consequently, this term becomes dominated by other terms and vanishes in Corollary \ref{coro:stochastic_gradient}.
\end{remark}
\begin{corollary}\label{coro:stochastic_gradient}
    To guarantee $\mathbb{E}[\lVert \nabla f(\hat{x}^{(T)})\rVert^2]\leq \varepsilon^2$ for $\varepsilon^2\geq 64\kappa\zeta^2$, we obtain
    \begin{align*}
        &T=\mathcal{O}\Big( \frac{\widetilde L\sqrt{\kappa+1}}{\alpha\varepsilon^2}\!+\!\frac{(\kappa+1)^{\nicefrac{1}{5}}\sigma^{\nicefrac{2}{5}}\widehat{L}}{\alpha^{\nicefrac{2}{5}}\varepsilon^{\nicefrac{12}{5}}}\!+\!\frac{(\kappa+1)^{\nicefrac{1}{4}}\sigma^{\nicefrac{1}{2}}\widehat{L}}{\alpha^{\nicefrac{1}{4}}\varepsilon^{\nicefrac{5}{2}}}\!\\
        &\quad+\!\frac{(\kappa+1)^{\nicefrac{1}{3}}\sigma^{\nicefrac{2}{3}}\widehat{L}}{\alpha^{\nicefrac{2}{3}}\varepsilon^{\nicefrac{8}{3}}}\!+\!\frac{(G(\kappa+1)+1)\sigma^2 \widehat{L}}{G\varepsilon^4}\Big).
    \end{align*}
\end{corollary}

Next, we consider a special case where local full gradients are available to workers, i.e., $\sigma = 0$.
\begin{corollary}\label{coro:full_gradient}
    If $\sigma=0$, then $\mathbb{E} \left[ \lVert \nabla f(\hat{x} ^{(T)}) \rVert \right]\leq \varepsilon$ after $T=\mathcal{O}(\nicefrac{\tilde{L}\sqrt{\kappa+1}}{\alpha \varepsilon^2})$ iterations.
\end{corollary}

\begin{remark}\label{remark:full_gradient}
In the special case of full gradients, \cite{rammal2024communication} established a complexity of $\mathcal{O}(\nicefrac{1+\sqrt{c\delta}}{\alpha \varepsilon^2})$, where $\delta=\nicefrac{B}{n}$ and $c$ are parameters characterizing the agnostic robust aggregator (ARAgg) \cite{karimireddy2021learning}.
Note that an $(B,\kappa)$-robust aggregation rule also qualifies as a $(\delta,c)$-ARAgg with $c=\nicefrac{\kappa n}{2B}$ \cite{allouah2023fixing}. As a result, Corollary \ref{coro:full_gradient} provides a slight improvement over the complexity bound $\mathcal{O}(\nicefrac{1+\sqrt{\kappa}}{\alpha \varepsilon^2})$ established in \cite{rammal2024communication}. 
Moreover, in the absence of Byzantine faults, we adopt the standard mean aggregator, which is (0,0)-robust. Then, the iteration complexity is $T=\mathcal{O}(\nicefrac{\widetilde{L}}{(\alpha \varepsilon^2)})$, achieving the lower bound for Byzantine-free distributed learning with communication compression \cite{huang2022lower}. We note that this asymptotic complexity bound also holds for any constant $\kappa$, achieved by many other aggregators unless $n$ is big and $\nicefrac{B}{n} \to \nicefrac{1}{2}$.

\end{remark}

\subsection{Convergence of Byz-DM21 under the Polyak-{\L}ojasiewicz Condition}
In this section, we provide complexity bounds for \algname{Byz-DM21} under the Polyak-{\L}ojasiewicz (PŁ) condition.

\begin{assumption}[Polyak-{\L}ojasiewicz condition]\label{as:pl}
    The function $f$ satisfies Polyak-Łojasiewicz (PŁ) condition with parameter $\mu$, i.e., for all $x \in \mathbb{R}^d$ there exists $x^*\in\arg\min_{x\in\mathbb{R}^d} f(x)$ such that
    \begin{equation}
        2\mu(f(x) - f(x^*)) \leq \lVert\nabla f(x) \rVert^2, \quad\forall x \in \mathbb{R},
    \end{equation}
    where $f(x^*) = \inf_{x\in \mathbb{R}^d} f(x) > -\infty$.
    Here we use a different notion of an $\varepsilon$-solution: it is a (random) point $\hat{x}$, such that $\mathbb{E}[f(\hat{x}^{(T)})-f(x^*)]\leq \varepsilon$. 
    Under this and previously introduced assumptions, we derive the following result.
\end{assumption}
\begin{theorem}\label{thm:pl_DM21}
    Let Assumptions \ref{assump:smoothness}, \ref{assump:heterogeneity}, \ref{assump:bound_variance} and \ref{as:pl} be satisfied, then choose momentum $\eta \leq \mathcal{O}\Big(\min\Big\{ \frac{\mu \varepsilon G}{(G(\kappa+1)+1)\sigma^2},\frac{\mu\alpha^2\varepsilon^{\nicefrac{1}{3}}}{(\kappa+1)\sigma^2},\frac{\mu\alpha\varepsilon^{1/4}}{(\kappa+1)\sigma^2}\Big\}\Big)$, after
    \begin{align*}
            T =&\mathcal{O}\Big( \frac{(G(\kappa+1)+1)\sigma^2}{\mu\varepsilon G} + \frac{(\kappa+1)\sigma^2}{\mu\alpha^2\varepsilon^{\nicefrac{1}{3}}}+\frac{(\kappa+1)\sigma^2}{\mu\alpha\varepsilon^{\nicefrac{1}{4}}}\nonumber\\
            &\quad+\frac{L}{\mu}+\frac{L\sigma^2(G(\kappa+1)+1)\sqrt{(\kappa+1)}}{\mu^2\varepsilon G} \Big)
    \end{align*}
    iterations  with stepsize $\gamma \!\leq \!\min \Big\{ {\frac{\eta}{2\mu}},\frac{\alpha}{4\mu},L^{-1}\Big(1+\sqrt{\frac{104\alpha^2(8\kappa+1)+48\eta^2(8\kappa+1)(2\eta^4+28\eta^2+\alpha^2+48)}{\alpha^2\eta^2}}\Big)^{-1}
    \Big\}$, \algname{Byz-DM21} produces a point $\hat{x}^{(T)}$ for which $\mathbb{E}[f(\hat{x}^{(T)})-f(x^*)]\leq \varepsilon$.
\end{theorem}

\section{Incorporating Variance Reduction}\label{VRDM21}
\vskip -0.1in
In this work, we employ a gradient estimator inspired by those utilized in \cite{cutkosky2019momentum} and \cite{tran2019hybrid}. The proposed gradient estimator of each node integrates the advantages of the widely-used \algname{SARAH} estimator \cite{nguyen2017sarah} and the unbiased \algname{SGD} gradient estimator. Formally, the gradient estimator of node $i\in \mathcal{G}$ at time step $t$ is expressed as:
\begin{align}
    & v_i^{(t)}=\eta \underbrace{\nabla f_i(x^{(t)}, \xi_i^{(t)})}_{\algname{SGD}} +\\
    &+ (1-\eta)\underbrace{\Big(v_i^{(t-1)}+\nabla f_i(x^{(t)}, \xi_i^{(t)})\!-\!\nabla f_i(x^{(t-1)}, \xi_i^{(t)})\Big)}_{\algname{SARAH}},\nonumber
\end{align}
with the parameter $\eta \in (0,1]$ as the momentum parameter. Next, we discuss the algorithm.

We extend \algname{Byz-DM21} to \algname{Byz-VR-DM21}, incorporating local variance reduction on all nodes to progressively eliminate variance from stochastic approximations. This new algorithm is depicted in Algorithm \ref{alg:SGD2M}. Its main difference from the original \algname{Byz-DM21} is that in the momentum update rule (Line 5), an additional term of $(1 - \eta) (\nabla f(x^{(t+1)}, \xi_i^{(t+1)}) - \nabla f(x^{(t)}, \xi_i^{(t+1)}))$ is added, inspired by the \algname{STORM} algorithm \cite{cutkosky2019momentum}. This term corrects the bias of $v_i^{(t)}$ so that it is an unbiased estimate of $\nabla f(x^{(t)})$ in the current iteration condition, i.e., $\mathbb{E}[v_i^{(t)}|x^{(t)}] = \nabla f_i(x^{(t)})$. We will also show that it reduces the variance and accelerates the convergence.

\subsection{Convergence of Byz-VR-DM21 for General Non-Convex
Problems}

We now present a variant of \algname{Byz-VR-DM21}, outlined in Algorithm \ref{alg:SGD2M}. The convergence analysis of this method hinges on the monotonicity of the following Lyapunov function:
\begin{align}\label{eq:lyapunov_vr}
\Phi^{(t)}=&\delta_{t}+\frac{4\gamma}{\eta}\lVert\widetilde M^{(t)}\rVert^2+\Bigg( \frac{48\gamma(8\kappa+1)(\eta^3+6\eta)}{\alpha^2 G}\nonumber\\
&+\frac{8\gamma(28\kappa+3)}{\eta G}\Bigg) \sum_{i\in \mathcal{G}}\lVert M_i^{(t)}\rVert^2 ,
\end{align}
where $\delta_{t}= f(x^{(t)})-f(x^*), M_i^{(t)} = v_i^{(t)}-\nabla f_i(x^{(t)})$, and $\widetilde M^{(t)}=G^{-1}\sum_{i\in\mathcal{G}}v_i^{(t)}-\nabla f_\mathcal{G}(x^{(t)})$. The primary convergence result for general non-convex functions is articulated in Theorem \ref{thm:rate_vr}, with the corresponding proof detailed in Appendix \ref{proof_VR_DM21}.

\begin{theorem}\label{thm:rate_vr}
    Let Assumptions \ref{assump:smoothness}, \ref{assump:individual_smoothness}, \ref{assump:heterogeneity} and \ref{assump:bound_variance} be satisfied, consider Algorithm 1 for solving the distributed learning problem \eqref{eq:original_P} with $B <{n/2}$ Byzantine workers and communication compression characterized by the parameter $\alpha \in \left(0,1\right]$ as per Definition \ref{def:contractive}. If $\eta \leq 1$, and
            $\gamma \leq \min \Big\{  \frac{\alpha}{10\sqrt{(8\kappa+1)(49\widetilde\ell^2+15\widetilde L^2)}},
            \frac{\sqrt{\eta}}{20\sqrt{(8\kappa+1)\widetilde \ell^2}}
            \Big\}$
            , then 
    \begin{align*}\label{eq:bound_vr}
        &\mathbb{E}\Big[\lVert \nabla f(\hat{x}^{(T)})\rVert^2\Big]\leq \frac{\Phi^{(0)}}{\gamma T}+\Bigg(\frac{96(8\kappa+1)(\eta^5+7\eta^3)}{\alpha^2 }\\
        &+ 8\eta(60\kappa+7)+\frac{8\eta}{G}+\frac{16\eta^4(8\kappa+1)}{\alpha}\Bigg)\sigma^2+32\kappa \zeta^2,
    \end{align*}
    where $\hat{x}^{(T)}$ is sampled uniformly at random from the iterates of the method. By setting $\eta \leq \mathcal{O}\Big(\min \Big\{   \Big( \frac{\alpha^2\delta_0 \widehat{\ell}}{(\kappa+1) \sigma^2 T}\Big)^{\nicefrac{2}{11}},
    $
    $\Big( \frac{\alpha^2\delta_0 \widehat{\ell}}{(\kappa+1) \sigma^2 T}\Big)^{\nicefrac{2}{7}},
    $
    $
    \Big( \frac{\delta_0 \widehat{\ell} }{(\kappa+1) \sigma^2 T}\Big)^{\nicefrac{2}{3}},
    $
    $
     \Big( \frac{\delta_0G \widehat{\ell}}{ \sigma^2 T}\Big)^{\nicefrac{2}{3}},
     $
     $
     \Big( \frac{\alpha\delta_0 \widehat{\ell}}{(\kappa+1) \sigma^2 T}\Big)^{\nicefrac{2}{9}}   \Big\}\Big) $, where $\widehat{\ell}\eqdef 40\sqrt{(8\kappa+1)\widetilde\ell^2}$, we obtain
     \begin{equation*}
         \begin{split}
             &\mathbb{E}\Big[\lVert \nabla f(\hat{x}^{(T)})\rVert^2\Big]=\mathcal{O}\Big(( \frac{(\kappa+1)^{\nicefrac{1}{10}}\sigma^{\nicefrac{2}{10}}\delta_0 \widehat{\ell}}{ \alpha^{\nicefrac{2}{10}} T})^{\nicefrac{10}{11}}\!\!+\!\!\frac{\Phi^{(0)}}{\gamma T}\\
             &\quad+\kappa \zeta^2+( \frac{(\kappa+1)^{\nicefrac{1}{6}}\sigma^{\nicefrac{2}{6}}\delta_0 \widehat{\ell}}{ \alpha^{\nicefrac{2}{6}} T})^{\nicefrac{6}{7}}+( \frac{\sigma\delta_0 \widehat{\ell}}{ G^{\nicefrac{1}{2}} T})^{\nicefrac{2}{3}}\\
             &\quad+( \frac{(\kappa+1)^{\nicefrac{1}{2}}\sigma\delta_0 \widehat{\ell}}{  T})^{\nicefrac{2}{3}}+( \frac{(\kappa+1)^{\nicefrac{1}{8}}\sigma^{\nicefrac{1}{4}}\delta_0 \widehat{\ell}}{ \alpha^{\nicefrac{1}{8}}T})^{\nicefrac{8}{9}}\Big).
         \end{split}
     \end{equation*}
\end{theorem}

\begin{corollary}\label{coro:stochastic_gradient_vr}
    To guarantee $\mathbb{E}[\lVert \nabla f(\hat{x}^{(T)})\rVert^2]\leq \varepsilon^2$ for $\varepsilon^2\geq64\kappa\zeta^2$, then
    \begin{align}
        &T=\mathcal{O}\Big( \frac{(G(\kappa+1)^{\nicefrac{1}{2}}+1)\sigma \widehat{\ell}}{G^{1/2}\varepsilon^3}\!+\!\frac{(\kappa+1)^{\nicefrac{1}{10}}\sigma^{\nicefrac{1}{5}}\widehat{\ell}}{\alpha^{\nicefrac{2}{5}}\varepsilon^{\nicefrac{11}{5}}}\nonumber\!\\
        &\quad\quad+\!\frac{(\kappa+1)^{\nicefrac{1}{6}}\sigma^{\nicefrac{1}{3}}\widehat{\ell}}{\alpha^{\nicefrac{1}{3}}\varepsilon^{\nicefrac{7}{3}}}\!+\!\frac{(\kappa+1)^{\nicefrac{1}{8}}\sigma^{\nicefrac{1}{4}}\widehat{\ell}}{\alpha^{\nicefrac{1}{8}}\varepsilon^{\nicefrac{9}{4}}}\!+\!\frac{\widetilde{L}\sqrt{\kappa+1}}{\alpha\varepsilon^2}\Big)\nonumber.
    \end{align}
\end{corollary}
\begin{corollary}\label{coro:full_gradient_vr}
    If $\sigma=0$, then $\mathbb{E} \left[ \lVert \nabla f(\hat{x} ^{(T)}) \rVert \right]\leq \varepsilon$ after $T=\mathcal{O}(\nicefrac{\widetilde{L}\sqrt{\kappa+1}}{\alpha \varepsilon^2})$ iterations.
\end{corollary}
\begin{remark}
    We emphasize several key properties of Theorem \ref{thm:rate} and Theorem \ref{thm:rate_vr}. These theorems provide the theoretical guarantee for the convergence of the error feedback method with stochastic gradients in the presence of Byzantine attacks. In the heterogeneous case (i.e.,  $\zeta > 0$), the algorithm does not ensure that $\mathbb{E} \left[ \lVert \nabla f(\hat{x} ^{(T)}) \rVert\right]$ can be made arbitrarily small. This limitation is characteristic of all Byzantine-robust algorithms in heterogeneous environments. Specifically, with an order-optimal robustness coefficient \(\kappa = \mathcal{O}(\nicefrac{B}{n})\), as achieved by \algname{CWTM}, the result is consistent with the lower bound \(\Omega(\nicefrac{B}{n}\zeta^2)\) established by \cite{allouah2023fixing}. Moreover, the highest attainable accuracy of \algname{Byz-DM21} and \algname{Byz-VR-DM21} is tighter than that of \algname{Byz-VR-MARINA} and \algname{Byz-EF21} (see Table \ref{table1}).
\end{remark}

\begin{figure}[t]
\vskip -0.01in
\begin{center}
\centerline{\includegraphics[width=\linewidth]{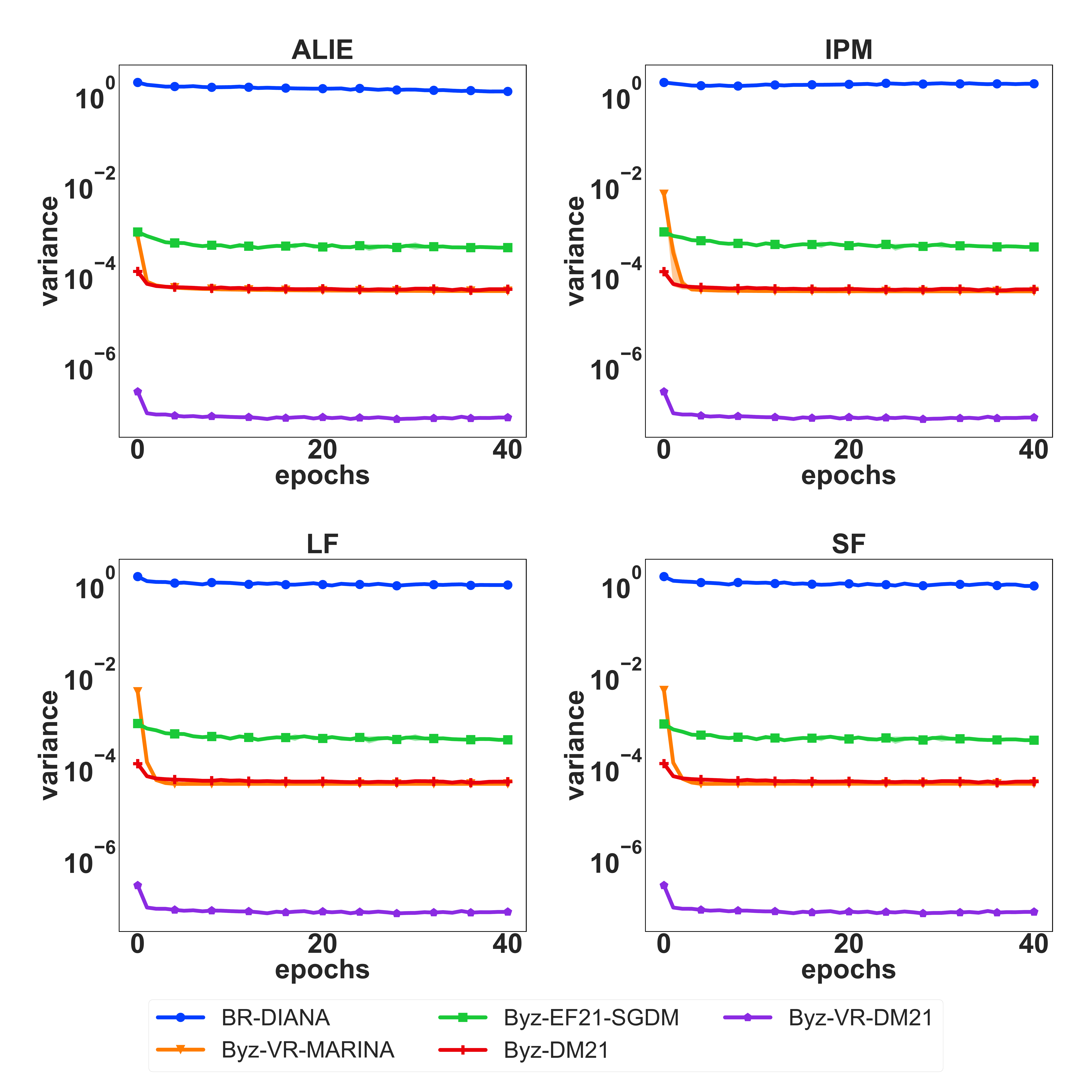}}
\caption{The training variance of honest messages under four attack scenarios on the a9a dataset.}
\label{train_variance}
\end{center}
\vskip -0.35in
\end{figure}

\begin{figure*}[h]
\vskip -0.1in
\begin{center}
\includegraphics[width=\textwidth]{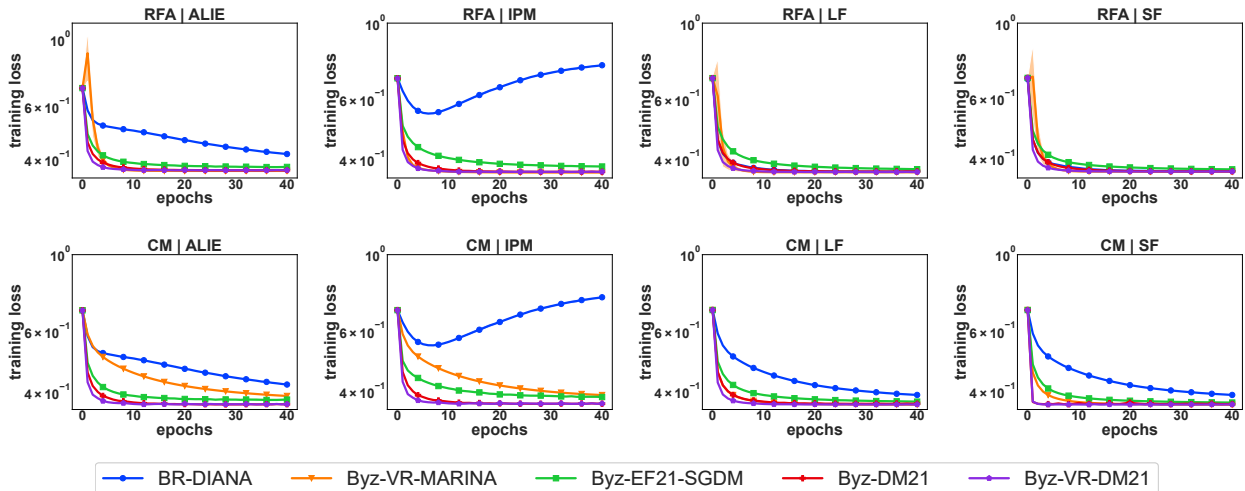}
\caption{The training loss of \algname{RFA} and \algname{CM} under four attack scenarios (SF, IPM, LF, ALIE) on the a9a dataset in a heterogeneous setting. We use $k=0.1d$ for both $\text{Rand}_k$ and $\text{Top}_k$ compressors.}
\label{train_loss}
\end{center}
\vskip -0.2in
\end{figure*}

\subsection{Convergence of Byz-VR-DM21 under the Polyak-{\L}ojasiewicz Condition}
\begin{theorem}\label{thm:pl_VRDM21}
    Let Assumptions \ref{assump:smoothness}, \ref{assump:individual_smoothness}, \ref{assump:heterogeneity}, \ref{assump:bound_variance} and \ref{as:pl} be satisfied, then choose momentum $\eta \leq \mathcal{O}\Big(\min\Big\{ \frac{\mu G\varepsilon}{(G(\kappa+1)+1)\sigma^2},\frac{\mu\alpha^2\varepsilon^{\nicefrac{1}{3}}}{(\kappa+1)\sigma^2},\frac{\mu\alpha\varepsilon^{1/4}}{(\kappa+1)\sigma^2}\Big\}\Big)$, after
    \begin{align*}
            T =&\mathcal{O}\Big( \frac{(G(\kappa+1)+1)\sigma^2}{\mu\varepsilon G} + \frac{(\kappa+1)\sigma^2}{\mu\alpha^2\varepsilon^{\nicefrac{1}{3}}}+\frac{(\kappa+1)\sigma^2}{\mu\alpha\varepsilon^{\nicefrac{1}{4}}}\\
            &\quad+\frac{L}{\mu}+\frac{L\sigma^2(G(\kappa+1)+1)}{\mu^2\varepsilon G}\Big)
    \end{align*}
    iterations  with stepsize 
    $\gamma \!\leq\! \min\! \Big\{\! \Big(L+ 4\sqrt{8\kappa\!+\!1}\\\sqrt{\frac{16\alpha^2( L^2\!+\!7\eta\ell^2)\!+\!\eta^2(L^2(3\eta^2\!+\!24)\!+\!\ell^2(12\eta^3+\alpha\eta^2+156\eta))}{\eta^2\alpha^2}}\Big)^{-1},\\
    \frac{\eta}{2\mu},\frac{\alpha}{4\mu}\Big\}$, \algname{Byz-VR-DM21} produces a point $\hat{x}^{(T)}$ for which $\mathbb{E}[f(\hat{x}^{(T)})-f(x^*)]\leq \varepsilon$.
\end{theorem}

\section{Numerical Experiments}\label{Experiments}

In this section, we demonstrate the performance of the proposed method. The goal of our experimental evaluation is to showcase the benefits of double momentum to mitigate Byzantine workers. We consider two binary classification tasks: a9a and w8a from the LIBSVM \cite{chang2011libsvm} dataset, and image classification on CIFAR-10 \cite{krizhevsky2009learning} and FEMNIST \cite{caldas2018leaf}. Due to space limitations, we present results only on a9a and defer the rest to the Appendix
\ref{extra_erperiments}.

\textbf{Adversarial attacks. } 
We address a binary logistic regression problem with a regularizer, using the a9a dataset from LIBSVM \cite{chang2011libsvm}. The data is distributed across $n$ = 20 workers, 8 of which are Byzantine. For aggregation, we use Robust Federated Averaging (RFA) \cite{pillutla2022robust}, Coordinate-wise Median (CM) \cite{yin2018byzantine}, and the \algname{NNM} algorithm \cite{allouah2023fixing}. The experiments test four Byzantine attack strategies: Sign Flipping (SF) \cite{allen-zhu2021byzantineresilient}, Label Flipping (LF) \cite{allen-zhu2021byzantineresilient}, Inner Product Manipulation (IPM) \cite{xie2020fall}, and A Little Is Enough (ALIE) \cite{baruch2019little} (details in Appendix \ref{app_agg}). We compare our algorithm with \algname{BR-DIANA} \footnote{\algname{BR-DIANA} is a version of \algname{BROADCAST} with the \algname{SGD} estimator instead of the \algname{SAGA} estimator. \algname{BROADCAST} consumes a large amount of memory, which scales linearly with the number of data points. We compare \algname{Byrd-SAGA} in Appendix \ref{extra_erperiments}.} \cite{mishchenko2019distributed}, \algname{Byz-VR-MARINA} \cite{gorbunov2023variance}, and \algname{Byz-EF21-SGDM} \cite{liu2026byzantine}. For the contractive compressor, we use $\text{Top}_k$, while \algname{Byz-EF21-SGDM} uses it too, and all other algorithms use $\text{Rand}_k$ \footnote{ To ensure a fair comparison, each method employs a theoretically compatible compressor.}.

\textbf{Empirical results. } Figure \ref{train_variance} illustrates the training variance of honest messages for the compared algorithms. The results demonstrate that \algname{Byz-VR-DM21} effectively reduces the variance of the stochastic gradient, maintaining a consistently low level of variance even after convergence. This highlights \algname{Byz-VR-DM21}'s enhanced robustness in mitigating noise and interference from Byzantine nodes. On the other hand, \algname{Byz-DM21} achieves a comparable variance level to \algname{Byz-VR-MARINA} despite not employing explicit variance reduction techniques, showcasing the advantages conferred by its double-momentum design. Figure \ref{train_loss} depicts the training loss across the four methods under various attack scenarios. Both \algname{Byz-DM21} and \algname{Byz-VR-DM21} exhibit rapid convergence and strong resilience to Byzantine interference, outperforming the other algorithms. In contrast, \algname{Byz-VR-MARINA}, although capable of quick convergence, suffers from significant fluctuations when exposed to Byzantine attacks. Notably, \algname{Byz-VR-DM21} achieves a faster convergence rate compared to \algname{Byz-DM21}, which can be attributed to the effectiveness of its variance reduction mechanism.

\textbf{Reproducibility.} To ensure reproducibility, all experiments were conducted using three different random seeds. We report the mean training loss along with one standard error.


\section{Conclusion}
We introduce \algname{Byz-DM21} and \algname{Byz-VR-DM21}, Byzantine-tolerant schemes designed to harness the empirical benefits of double momentum and variance reduction, while preserving communication efficiency. Unlike most existing Byzantine-tolerant methods, our new algorithm leverages stochastic gradients and is batch-free, meaning it does not require computing full gradients or additional tuning. Through theoretical analysis, we prove that our new algorithm has tight lower bounds and a smaller neighborhood size. It matches the upper bound results in both stochastic and full gradient scenarios when the problem is Byzantine-free, and also converges to a smaller neighborhood around the optimal solution. We further show that, under the P{\L} condition, our algorithm admits a faster convergence guarantee in terms of the optimality gap. Moreover, we demonstrate the robustness of our algorithm against adversarial attacks through experiments on binary and image classification tasks.

\section*{Acknowledgements}
Yanghao Li and Yuhao Yi acknowledge support from the National Natural Science Foundation of China under Grant 62303338, while Changxin Liu acknowledges support under Grant 62573196.

\bibliography{ref}

\onecolumn
\newpage
\tableofcontents
\clearpage
\appendix

\section{Related Work}
\textbf{Byzantine-robust distributed learning}. A distributed learning algorithm is Byzantine-robust if its performance remains reliable in the face of Byzantine workers, which may act arbitrarily. The framework for Byzantine-robust distributed learning generally includes three main steps \cite{guerraoui2024byzantine}: (i) pre-processing of vectors submitted by workers (e.g., models or stochastic gradients), (ii) robust aggregation of these vectors, and (iii) the application of an optimization method. Existing works in this domain often vary in their treatment of one or more of these steps. For pre-processing, techniques such as Bucketing \cite{karimireddy2022byzantinerobust} and Nearest Neighbor Mixing (\algname{NNM}) \cite{allouah2023fixing} have been proposed. Many robust aggregation rules have been proposed. The most commonly used ones include coordinate-wise median (\algname{CWMed}), coordinate-wise trimmed mean (\algname{CWTM}) \cite{yin2018byzantine}, centered clipping \cite{karimireddy2021learning}. Beyond the classical setting where the server outputs a single model and robustness typically relies on a sufficiently large honest fraction, list-decodable federated learning maintains a list of candidate models and guarantees that at least one model performs well even under a malicious majority \cite{liu2025lid}. Additionally, Byzantine attack identification strategies have demonstrated enhanced robustness in distributed computing tasks such as matrix multiplication \cite{hong2022hierarchical}. Various optimization algorithms have been applied to the considered problem, including Stochastic Gradient Descent (\algname{SGD}) \cite{yang2021basgd}, Polyak momentum \cite{karimireddy2022byzantinerobust,el2021distributed}, \algname{SAGA} \cite{wu2020federated}, and \algname{VR-MARINA} \cite{gorbunov2023variance}.

\textbf{Variance Reduction.} Variance reduction is a powerful technique for accelerating the convergence of stochastic methods, particularly when a good approximation of the solution is required. The earliest variance-reduced methods were introduced by \cite{schmidt2017minimizing,johnson2013accelerating,defazio2014saga}. Optimal methods for variance reduction in (strongly) convex problems were later proposed by \cite{lan2018optimal,allen2018katyusha,lan2019unified}, while methods for non-convex optimization were developed by \cite{nguyen2017sarah,fang2018spider,li2021page}. Although variance reduction has attracted significant attention \cite{gower2020variance}, there has been limited research on combining it with Byzantine robustness, with only a few studies addressing this area \cite{wu2020federated,zhu2021broadcast,karimireddy2021learning,gorbunov2023variance}.

\textbf{Compressed Communications and Error Feedback.}  Research on distributed methods with communication compression can generally be divided into two main categories. The first focuses on methods utilizing unbiased compression operators, such as Rand$K$ sparsification ($\text{Rand}_k$) \cite{horvoth2022natural}, while the second explores methods that employ biased compressors, like Top$K$ sparsification ($\text{Top}_k$) \cite{alistarh2018convergence}. A detailed summary of the most commonly used compression operators can be found in \cite{safaryan2022uncertainty,beznosikov2023biased}. Biased compression combined with error feedback has demonstrated strong practical performance \cite{seide20141,vogels2019powersgd}. In the non-convex setting, which is the focus of our work, standard error feedback has been analyzed by \cite{karimireddy2019error,beznosikov2023biased} and \cite{sahu2021rethinking}. However, the complexity bounds for standard error feedback typically depend on the heterogeneity parameter $\zeta^2$ or require bounded gradients. \cite{richtarik2021ef21} address these challenges by introducing a novel variant of error feedback, known as EF21. This approach has since been extended in various directions by \cite{fatkhullin2021ef21}.

\newpage
\section{Intuition Behind Double Momentum}
In this section, we show that double momentum can reduce the noise variance in a one-dimensional toy example.

We consider a simple one-dimensional stochastic gradient model
\[
g_t = \mu_t + \xi_t,
\]
where \(\mu_t = \mathbb{E}[g_t \mid \mathcal{F}_{t-1}]\) is the (possibly time-varying) true gradient at step \(t\), and \(\xi_t\) is the noise term with \(\mathbb{E}[\xi_t \mid \mathcal{F}_{t-1}] = 0\) and \(\mathrm{Var}(\xi_t \mid \mathcal{F}_{t-1}) \le \sigma^2\).

Here \(\mathcal{F}_{t-1}\) denotes the sigma-field generated by all randomness up to time \(t-1\). We do not need to assume that \(\mu_t\) is constant or independent of the noise. The only structural assumption is the standard one for SGD-type methods: the noise sequence \((\xi_t)\) is a martingale difference with bounded conditional variance. For the exact variance formulas below, we further consider the homoskedastic case \(\mathrm{Var}(\xi_t \mid \mathcal{F}_{t-1}) = \sigma^2\).

\textbf{Single momentum (conditional analysis)}

For convenience of analysis, we set \(v_0=0\); however, choosing \(v_0=g_0\) instead does not change the asymptotic variance. The single-momentum update is
\[
v_{t+1} = (1-\eta)v_t + \eta g_{t+1},\quad 0<\eta<1,\; v_0=0.
\]

Unrolling the recursion yields the identity
\[
v_t = \eta\sum_{k=0}^{t-1}(1-\eta)^k g_{t-k}.
\]
Substituting \(g_{t-k} = \mu_{t-k} + \xi_{t-k}\), we decompose
\[
v_t
= \underbrace{\eta\sum_{k=0}^{t-1}(1-\eta)^k\mu_{t-k}}_{\text{signal part}}
 + \underbrace{\eta\sum_{k=0}^{t-1}(1-\eta)^k\xi_{t-k}}_{\text{noise part}}.
\]

Now fix any realization of \(\mu_{1:t}=\{\mu_1,\dots,\mu_t\}\) and take conditional expectation with respect to the noise:
\[
\mathbb{E}[v_t \mid \mu_{1:t}]
= \eta\sum_{k=0}^{t-1}(1-\eta)^k\mu_{t-k}.
\]
This shows that in the time-varying case, momentum tracks an exponentially weighted average of past true gradients, not the instantaneous \(\mu_t\). This is the usual bias of momentum methods.

For the conditional variance, define the centered noise part
\[
\widetilde v_t
:= v_t - \mathbb{E}[v_t\mid \mu_{1:t}]
= \eta\sum_{k=0}^{t-1}(1-\eta)^k\xi_{t-k}.
\]
Using the martingale-difference property of the noise terms \(\xi_s\) and the assumption \(\mathrm{Var}(\xi_s\mid\mathcal{F}_{s-1}) = \sigma^2\), we obtain
\[
\mathbb{E}[\widetilde v_t^2\mid \mu_{1:t}]
= \eta^2\sigma^2\sum_{k=0}^{t-1}(1-\eta)^{2k}.
\]
Letting \(t\to\infty\) gives
\[
\mathrm{Var}_{\text{noise}}(v_\infty \mid \mu_{1:\infty})
= \eta^2\sigma^2\sum_{k=0}^{\infty}(1-\eta)^{2k}
= \frac{\eta}{2-\eta}\sigma^2.
\]
Importantly, this conditional variance formula is identical to the constant-\(\mu\) case and does not depend on how \(\mu_t\) is generated.

\textbf{Double momentum (conditional analysis)}

Double momentum maintains two EMA variables. Starting from \(v_0=u_0=0\), the updates are
\[
\begin{aligned}
v_{t+1} &= (1-\eta)v_t + \eta g_{t+1},\\
u_{t+1} &= (1-\eta)u_t + \eta v_{t+1}.
\end{aligned}
\]
Unrolling the second recursion gives
\[
u_t = \eta^2\sum_{r=0}^{t-1}(r+1)(1-\eta)^r g_{t-r}.
\]
Again substituting \(g_{t-r} = \mu_{t-r} + \xi_{t-r}\), we obtain
\[
u_t
= \underbrace{\eta^2\sum_{r=0}^{t-1}(r+1)(1-\eta)^r\mu_{t-r}}_{\text{signal part}}
 + \underbrace{\eta^2\sum_{r=0}^{t-1}(r+1)(1-\eta)^r\xi_{t-r}}_{\text{noise part}}.
\]

Conditioning on \(\mu_{1:t}\), the centered noise component is
\[
\widetilde u_t
:= u_t - \mathbb{E}[u_t\mid\mu_{1:t}]
= \eta^2\sum_{r=0}^{t-1}(r+1)(1-\eta)^r\xi_{t-r},
\]
and its conditional variance is
\[
\mathbb{E}[\widetilde u_t^2\mid\mu_{1:t}]
= \eta^4\sigma^2\sum_{r=0}^{t-1}(r+1)^2(1-\eta)^{2r}.
\]
Letting \(t\to\infty\) and using the standard series formula
\(\sum_{r=0}^{\infty}(r+1)^2b^r = \frac{1+b}{(1-b)^3}\) with \(b=(1-\eta)^2\), we get
\[
\mathrm{Var}_{\text{noise}}(u_\infty \mid \mu_{1:\infty})
= \sigma^2\eta\,\frac{2-2\eta+\eta^2}{(2-\eta)^3}.
\]

\textbf{Conditional variance comparison}

Therefore,
\[
\mathrm{Var}_{\text{noise}}(v_\infty \mid \mu_{1:\infty})
= \frac{\eta}{2-\eta}\sigma^2,\qquad
\mathrm{Var}_{\text{noise}}(u_\infty \mid \mu_{1:\infty})
= \sigma^2\eta\,\frac{2-2\eta+\eta^2}{(2-\eta)^3},
\]
and the ratio
\[
\frac{\mathrm{Var}_{\text{noise}}(u_\infty \mid \mu_{1:\infty})}{\mathrm{Var}_{\text{noise}}(v_\infty \mid \mu_{1:\infty})}
= \frac{2-2\eta+\eta^2}{(2-\eta)^2}
\in\big[\tfrac12,1\big),\quad 0<\eta<1.
\]

Thus, the conditional noise variance of the double-momentum estimator is strictly smaller than that of the single-momentum estimator, with an asymptotic reduction factor of about \(1/2\) when \(\eta\) is small. The time variation of \(\mu_t\) only affects the conditional mean
\(\mathbb{E}[v_t\mid\mu_{1:t}]\) and \(\mathbb{E}[u_t\mid\mu_{1:t}]\), i.e., the bias, but not the above variance formulas for the noise component.

\newpage
\section{Further Details on Robust Aggregation and Byzantine Attacks}\label{app_agg}

\subsection{Robust Aggregation}
In Section \ref{preli}, we apply robust aggregation rules satisfying Definition \ref{def:Byzantine_robustness} after using NNM proposed by \cite{allouah2023fixing} (see Algorithm \ref{alg:NNM}). This algorithm can enhance the robustness of aggregation rules. In particular, \cite{allouah2023fixing} shows that Algorithm \ref{alg:NNM} makes Robust Federated Averaging (\algname{RFA}) \cite{pillutla2022robust} (also known as geometric median), Coordinate-wise Median (\algname{CM}) \cite{chen2017distributed} robust, and Coordinate-wise Trimmed Mean (\algname{CWTM}) in view of the definition from \cite{allouah2023fixing}.

\begin{algorithm}
\caption{NNM: Nearest Neighbor Mixing \cite{allouah2023fixing}}
\label{alg:NNM}
\begin{algorithmic}[1]
    \State\textbf{Input:} number of inputs $n$, number of Byzantine inputs $B$\textless $\nicefrac{n}{2}$, vectors $x_1,\ldots, x_n \in \mathbb{R}^d$.
    \For{$i = 1 \ldots n$}
    \State Sort inputs to get $( x_{i:1},\ldots,x_{i:n})$ such that $\lVert x_{i:1}-x_i \rVert \leq \ldots\leq\lVert x_{i:n}-x_i\rVert$;
    \State Average the $G$ nearest neighbors of $x_i$, i.e., $y_i = \nicefrac{1}{G}\sum_{j=1}^{G}x_{i:j}$;
    \EndFor
    \State \textbf{Return:} $y_1,\ldots,y_n;$
    
\end{algorithmic}
\end{algorithm}
Our main goal in this section is to show that \algname{RFA} $\circ$ \algname{NNM}, \algname{CM} $\circ$ \algname{NNM}, and \algname{CWTM} $\circ$ \algname{NNM} satisfy Deﬁnition \ref{def:Byzantine_robustness}. Before we prove this fact, we need to introduce \algname{RFA}, \algname{CM}, \algname{CWTM}.

\paragraph{Robust Federated Averaging.} \text{RFA}-estimator finds a geometric median:
\label{eq:RFA_estimator}
\begin{equation}
	\text{RFA}(x_1,\ldots, x_n) \eqdef \argmin\limits_{x\in \mathbb{R}^d} \sum\limits_{i=1}^n \|x - x_i\|. 
\end{equation}
The above problem has no closed-form solution. However, one can compute approximate \text{RFA} using several steps of smoothed Weiszfeld algorithm having $\mathcal{O}(n)$ computation cost of each iteration \cite{weiszfeld1937point, pillutla2022robust}.

\paragraph{Coordinate-wise Median.} \text{CM}-estimator computes a median of each component separately. That is, for the $t$-th coordinate,  it is defined as
\begin{equation}
	\left[\text{CM}(x_1,\ldots, x_n)\right]_t \eqdef \text{Median}([x_1]_t,\ldots, [x_n]_t) = \argmin\limits_{u \in \mathbb{R}} \sum\limits_{i=1}^n |u - [x_i]_t|, \label{eq:CM_estimator}
\end{equation}

\paragraph{Coordinate-wise Trimmed Mean.} The CWTM-estimator for input vectors $x_1, \ldots, x_n \in \mathbb{R}^d$ is defined as:
\begin{equation}
[\text{CWTM}(x_1, \ldots, x_n)]_t \eqdef \argmin\limits_{v \in \mathbb{R}}\sum_{i=B+1}^{n-B} ([x_i]_t-v)^2,
\end{equation}
where \([x_i]_t\) represents the $t$-th coordinate values of the input vectors, sorted in non-decreasing order. CWTM finds the value $v$ that minimizes the sum of squared differences with the middle $(n-2B)$ values after discarding the smallest 
$B$ and largest $B$.

\subsection{Byzantine Attacks}
Byzantine workers send sophisticated malicious updates to the server. In this scenario, we model an omniscient attacker \cite{baruch2019little}, who knows the data of all clients. Therefore, the attacker can mimic the statistical properties of honest updates and then craft a malicious one.
\begin{itemize}
    \item \textbf{Sign Flipping (SF)} \cite{allen-zhu2021byzantineresilient}: Byzantine workers compute $-c_i^{t+1}$ and send it to the server.
    \item \textbf{Label Flipping (LF)} \cite{allen-zhu2021byzantineresilient}: Byzantine workers compute their gradients using poisoned labels (i.e., $y_i \rightarrow -y_i$).
    \item \textbf{A Little Is Enough (ALIE)} \cite{baruch2019little}: Byzantine workers compute the empirical mean $\mu_{\mathcal{G}}$ and standard deviation $\sigma_{\mathcal{G}}$ of $\{c_i^{t+1} \}_{i\in\mathcal{G}}$ and send $\mu_{\mathcal{G}} - z\sigma_{\mathcal{G}}$ to the server, where $z$ is a constant that controls the strength of the attack.
    \item \textbf{Inner Product Manipulation (IPM)} \cite{xie2020fall}: Byzantine workers send $-\frac{z}{G} \sum_{i\in \mathcal{G}} c_i^{t+1}$ to the server, where $z > 0$ is a constant that controls the strength of the attack.

\end{itemize}

\section{Extra Experiments and Experimental Details}\label{extra_erperiments}
\subsection{General Setup}
Our running environment has the following setup:
\begin{itemize}
\item CPU: AMD Ryzen Threadripper PRO 5975WX 32-Cores,
\item GPU: NVIDIA GeForce RTX 4090 with CUDA version 11.7,
\item PyTorch version: 2.0.1.
\end{itemize}

\subsection{Datasets}

We apply the proposed algorithm to two types of classification tasks to evaluate its robustness, namely binary classification and image classification. 
For the binary classification task, we utilize the a9a and w8a datasets from LIBSVM \cite{chang2011libsvm}. The information of the a9a and the w8a dataset is summarized in Table \ref{table:dataset}). For the image classification task, we use the CIFAR-10 \cite{krizhevsky2009learning} and FEMNIST \cite{caldas2018leaf} datasets. The introduction of the data set and the distribution of the data are as follows:

\begin{table}[H]
\centering
\caption{Overview of the LibSVM datasets used.}
\label{table:dataset}
\begin{tabular}{c c c} 
 \hline
 Dataset & N (\text{\# of datapoints}) & d (\text{\# of features}) \\
 \hline
 \texttt{a9a} & 32,561 & 123 \\
 \texttt{w8a} & 49,749 & 300 \\
 \hline
\end{tabular}
\end{table}

\begin{itemize}[itemsep=0ex,leftmargin=1em]
    \item \textbf{CIFAR-10. } The CIFAR-10 dataset consists of 60,000 32 × 32 color images across 10 classes, with 6,000 images per class. The sampled images are evenly distributed among the clients. For CIFAR-10, we train a Residual Network with 20 layers (ResNet-20) \cite{he2016deep}.
    \item \textbf{FEMNIST. } The Federated Extended MNIST (FEMNIST) dataset is a widely used benchmark for federated learning, constructed by partitioning the EMNIST dataset \cite{cohen2017emnist}. FEMNIST consists of 805,263 images across 62 classes. For this study, we randomly sample 5\% of the images from the original dataset and distribute them to clients in an independent and identically distributed (i.i.d.) manner. It is important to note that we simulate an imbalanced data distribution, where the number of training samples varies across clients. The implementation is based on LEAF \cite{caldas2018leaf}. For FEMNIST, we train a Convolutional Neural Network (CNN) \cite{krizhevsky2009learning} with two convolutional layers.
\end{itemize}

In each case, the data is distributed among $n = 20$ workers, out of which 8 are Byzantine.

\subsection{Experiments Setup}
For each experiment, we select the step size from the candidate set 
\{0.5, 0.05, 0.005\} and fix it throughout the training process. No learning rate warmup or decay is applied. Each experiment is repeated with three different random seeds, and we report the mean training loss or testing accuracy along with one standard error. Additionally, we adopt the same set of robust aggregation hyperparameters as outlined in \cite{karimireddy2022byzantinerobust}, i.e.,
\begin{itemize}
\item \textbf{Byz-EF21-SGDM, Byz-DM21, and Byz-VR-DM21 :}the momentum parameter is set as $\eta = 0.1$,
    \item \textbf{Byz-VR-MARINA :} the probability to compute full gradient as $p = \min\left\{\nicefrac{b}{m}, \nicefrac{1}{(1+\omega)}\right\}$,
    \item \textbf{BR-DIANA :} the compressed difference parameter $\beta = 0.01$,
    \item \textbf{RFA :} the number of steps of smoothed Weisﬁeld algorithm $T$ = 8,
    \item \textbf{ALIE :} a small constant that controls the strength of the attack $z$ is chosen according to \cite{baruch2019little},
    \item \textbf{IPM :} a small constant that controls the strength of the attack $\epsilon = 0.1$.
\end{itemize}

\newpage
\begin{figure*}[t]
\vskip 0.02in
\begin{center}
\includegraphics[width=\textwidth]{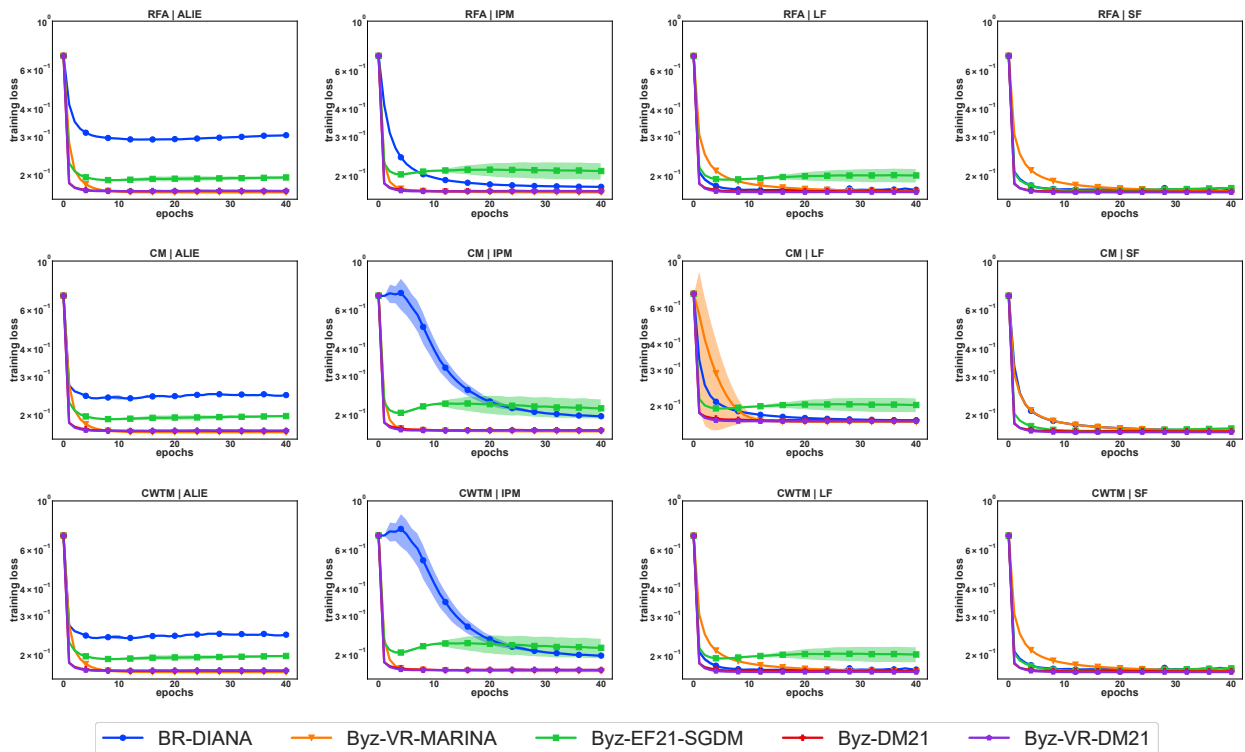}
\caption{The training loss of \algname{RFA}, \algname{CM}, and \algname{CWTM} aggregation rules under four attack scenarios (SF, IPM, LF, ALIE) on the w8a dataset in a heterogeneous setting. \algname{BR-DIANA} and \algname{Byz-VR-MARINA} use the $\text{Rand}_k$ compressor with $k = 0.1d$, while \algname{Byz-EF21-SGDM}, \algname{Byz-DM21}, and \algname{Byz-VR-DM21} use the $\text{Top}_k$ compressor with $k = 0.1d$.}
\label{loss_w8a}
\end{center}
\vskip -0.1in
\end{figure*}

\subsection{Empirical Results on logistic regression}
For this task, we consider solving a logistic regression problem with $l_2$-regularization:
\begin{align}
    f(x,\xi_i)= \log (1+\exp(-b_i a_i x))+\lambda \lVert x\rVert^2\nonumber
\end{align}
where $\xi_i = (a_i,b_i) \in \mathbb{R}^{d} \times \{ -1,1\}$
denotes each data point, with $a_i\in \mathbb{R}^d$ and $b_i\in \{ -1,1\}$, and $\lambda > 0 $ is the regularization parameter. We set $\lambda = 1/m$ for these experiments, where $m$ is the number of samples in the local datasets. For all methods, we use a batch size of $b=1$, and select the step-size from the following candidates: $\gamma \in \{0.5, 0.05, 0.005\}$. We use $k = 0.1d$ for both $\text{Top}_k$ and $\text{Rand}_k$ compressors. Finally, the number of epochs is set to 40. To ensure reproducibility, all experiments were conducted using three different random seeds. We report the mean training loss along with one standard error.

Figure \ref{loss_w8a} illustrates the training loss for five algorithms—\algname{BR-DIANA}, \algname{Byz-VR-MARINA}, \algname{Byz-EF21-SGDM}, \algname{Byz-DM21}, and \algname{Byz-VR-DM21}—under a 0.4 adversarial setting on the w9a dataset in a heterogeneous setting. The results show that \algname{Byz-DM21} and \algname{Byz-VR-DM21} overall have a slight advantage, being able to converge more quickly to the optimal solution, even when there is a high proportion of Byzantine workers. The improvement in performance is attributed to the implementation of the double momentum mechanism and variance reduction techniques, which effectively reduce the impact of malicious updates.

\newpage
\begin{figure*}[t]
\vskip 0.02in
\begin{center}
\includegraphics[width=\textwidth]{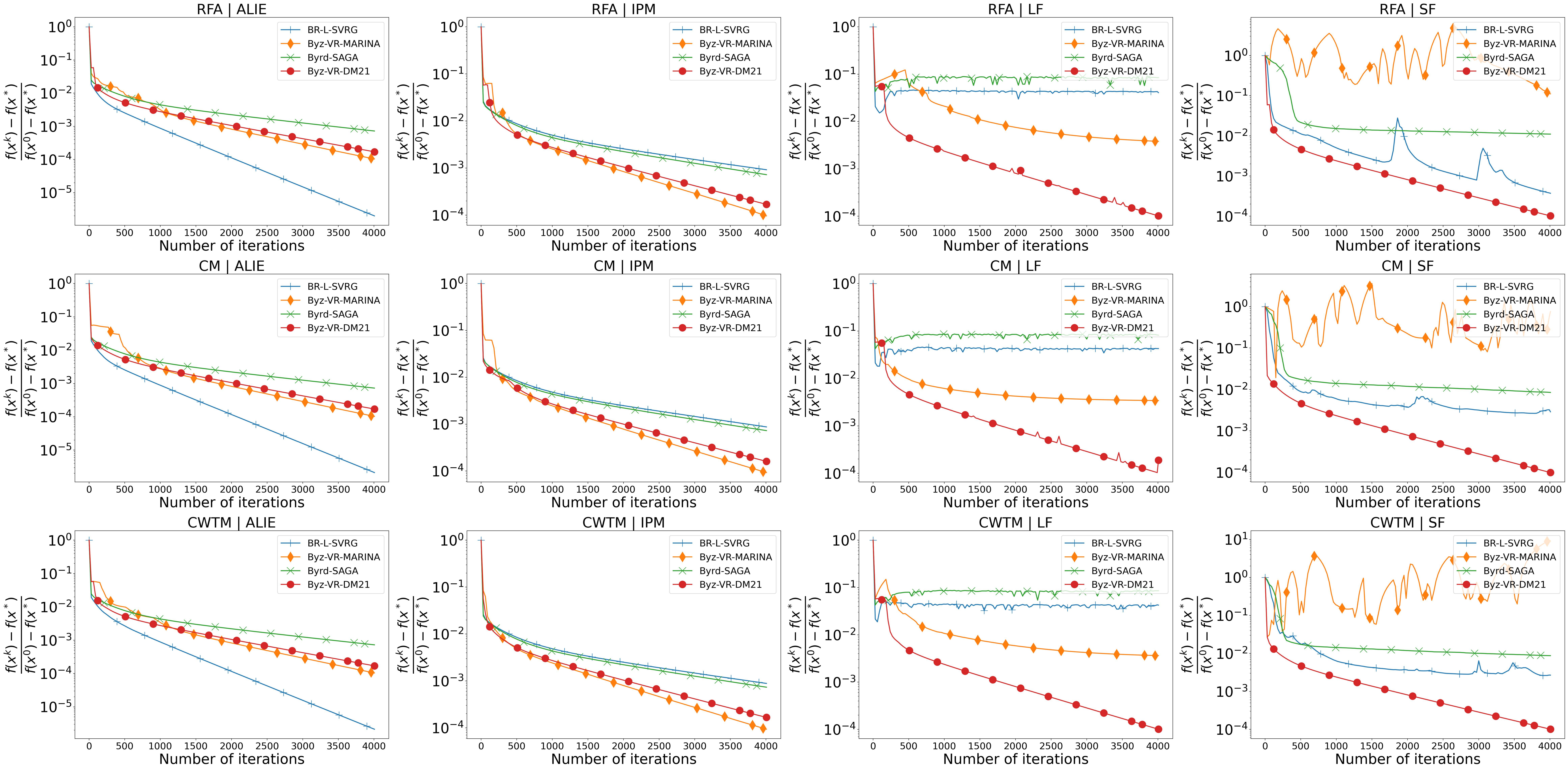}
\caption{The relative error curve of 3 aggregation rules ({RFA}, {CM}, {CWTM}) under 4 attacks (ALIE, IPM, LF, SF) on the a9a dataset in a heterogeneous setting. The dataset is uniformly split over 12 honest workers with 8 Byzantine workers. {\algname{BR-LSVRG}}, {\algname{Byz-VR-MARINA}}, {\algname{Byrd-SAGA}} and {\algname{Byz-VR-DM21}} with batchsize $b=0.01m$ and step size $\gamma=1/2L$.}
\label{variance_method}
\end{center}
\vskip -0.1in
\end{figure*}

\subsection{Comparison of Variance Reduction Methods}

In this experiment, we compare our \algname{Byz-VR-DM21} with several well-known variance-reduction algorithms, including \algname{BR-LSVRG} \cite{fedin2023byzantine}, \algname{Byz-VR-MARINA} \cite{gorbunov2023variance}, and \algname{Byrd-SAGA} \cite{wu2020federated}. We evaluate these methods on the a9a dataset from the LIBSVM library, using logistic regression with $l_2$ regularization. Specifically, we set $l_2 = L / 1000$ in our experiments, where $L$ denotes the $L$-smoothness constant.  The total number of workers is $n=20$, including 8 Byzantine workers. All methods were evaluated with a step size $\gamma=\nicefrac{1}{2L}$ and a batch size of $b=0.01m$, i.e., $b = 325$.

Figure \ref{variance_method} presents the relative error curve of four different algorithms—\algname{BR-LSVRG}, \algname{Byz-VR-MARINA}, \algname{Byrd-SAGA}, and \algname{Byz-VR-DM21}—under a 0.4 adversarial setting on the a9a dataset in a heterogeneous setting. The results show that, under both Sign Flipping (SF) and Label Flipping (LF) attacks, \algname{Byz-VR-DM21} consistently outperforms the other algorithms, achieving faster convergence to the optimal solution. This underscores \algname{Byz-VR-DM21}'s superior robustness in minimizing the impact of adversarial updates, enabling the model to stay on course toward optimal performance despite the presence of Byzantine workers. \algname{BR-LSVRG} is particularly effective in the A Little Is Enough (ALIE) attacks, and \algname{Byz-VR-MARINA} is particularly effective in the Inner Product Manipulation (IPM) attacks. In all cases, \algname{Byz-VR-DM21} demonstrates convergence to very high accuracy, while none of the algorithms dominate across all the attacks. In addition, \algname{Byz-VR-DM21} consistently achieves a functional suboptimality of at least $10^{-4}$, representing reasonably good accuracy. This experiment demonstrates that \algname{Byz-VR-DM21} can converge to high accuracy even with small or moderate batch sizes.

\newpage

\begin{figure*}[t]
\vskip 0.02in
\begin{center}
\includegraphics[width=\textwidth]{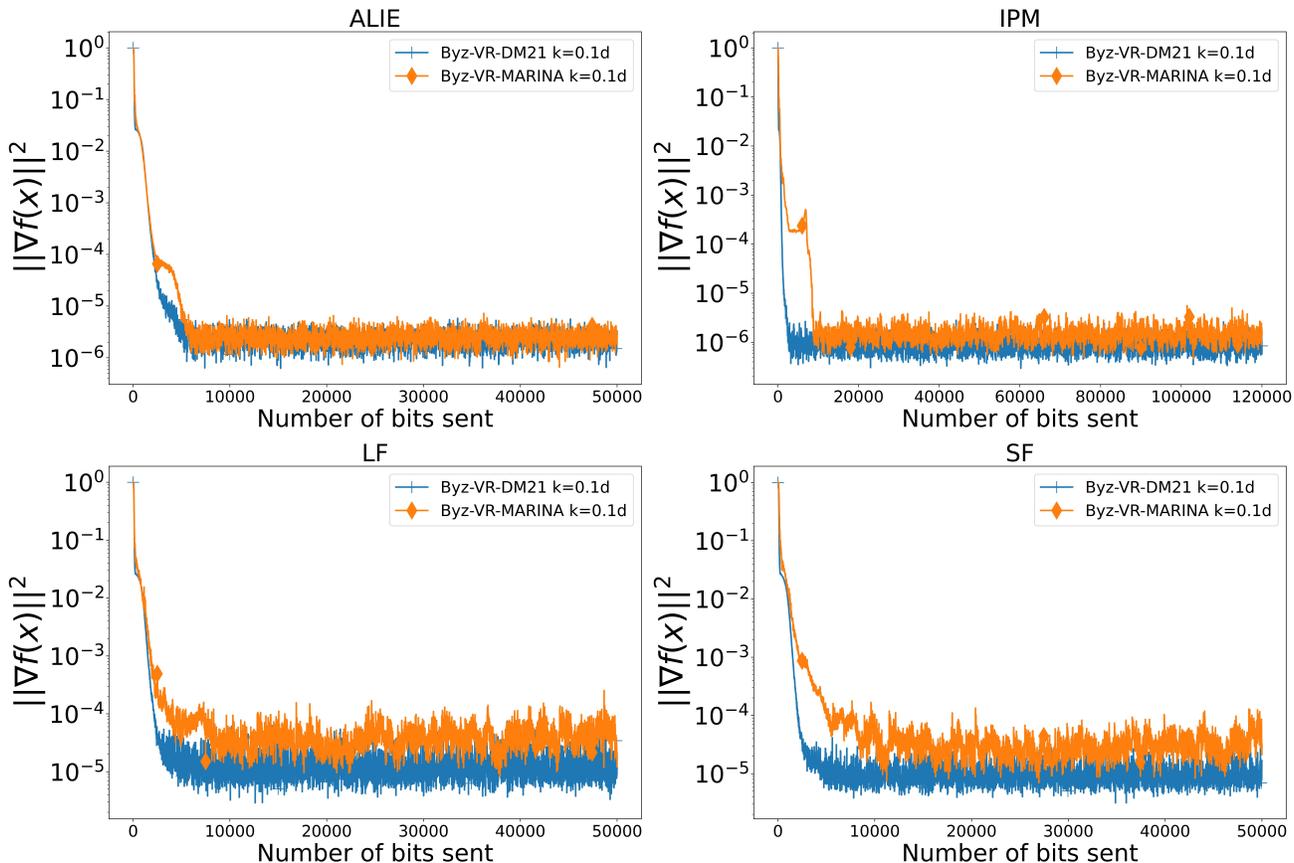}
\caption{The communication complexity comparison under 4 attacks (ALIE, IPM, LF, SF) on the a9a dataset in a heterogeneous setting. \algname{Byz-VR-MARINA} uses the $\text{Rand}_k$ compressor, while \algname{Byz-VR-DM21} uses the $\text{Top}_k$ compressor, with $k=0.1d$, batch size $b=1$, and step size $\gamma=1/2L$.}
\label{errof_feedback}
\end{center}
\vskip -0.2in
\end{figure*}

\subsection{Error Feedback Experiments}
We next compare the empirical performance of \algname{Byz-VR-DM21} and \algname{Byz-VR-MARINA} on the a9a dataset in a heterogeneous setting. The total number of workers is $n = 20$, including 2 Byzantine workers. All methods were evaluated with a step size of $\gamma = 1/2L$ and a batch size of $b = 1$. We use $k = 0.1d$ for both the $\text{Top}_k$ and $\text{Rand}_k$ compressors.

The experiments unveil promising potential for error feedback of communication improvement. As shown in Figures \ref{errof_feedback}, \algname{Byz-VR-DM21} converges slightly faster than \algname{Byz-VR-MARINA}, before both reach a point of stagnation.

\newpage

\begin{figure*}[ht]
\vskip 0.02in
\begin{center}
\includegraphics[width=\textwidth]{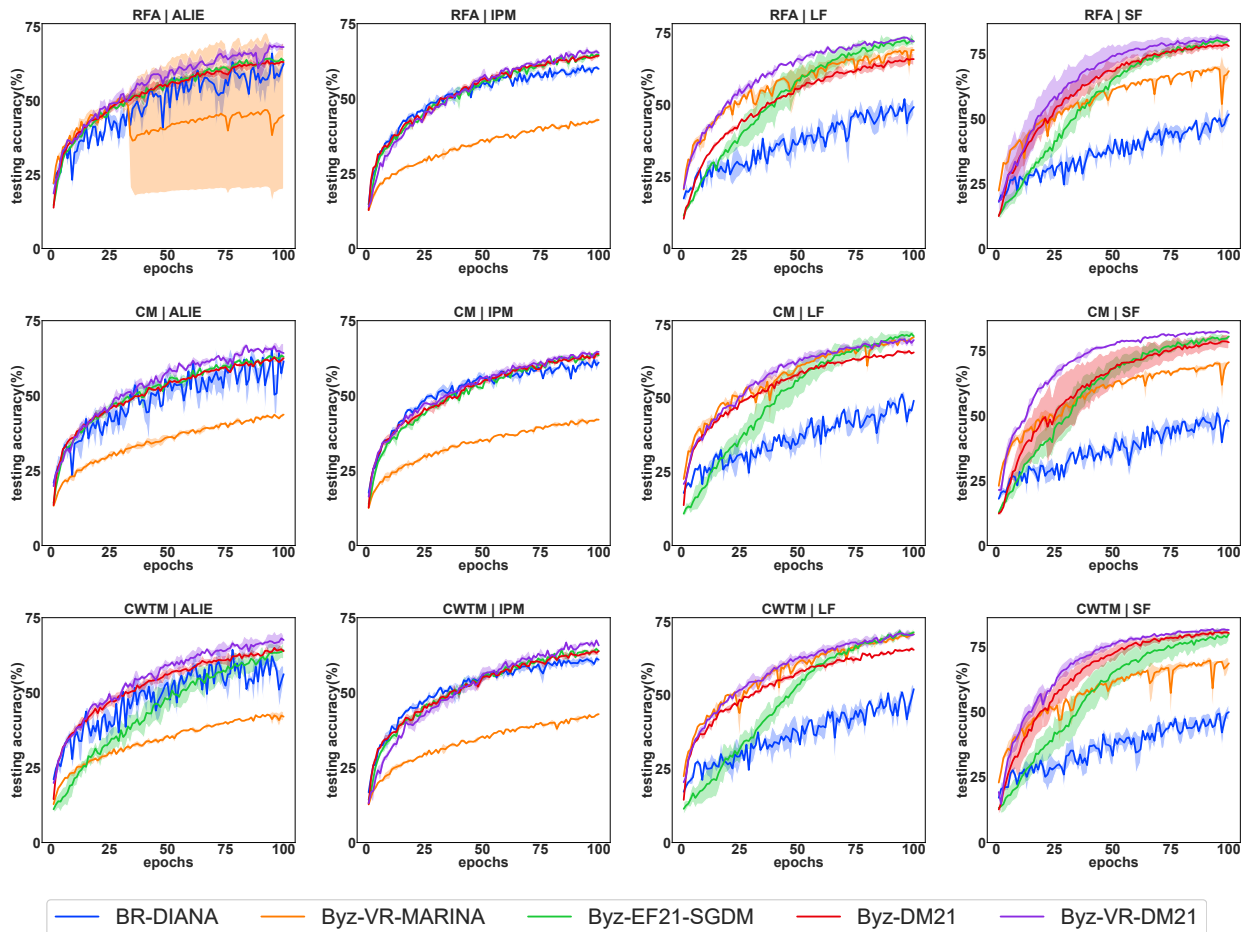}
\caption{The testing accuracy (\%) of 3 aggregation rules (\algname{RFA, CM, CWTM}) under 4 attacks (ALIE, IPM, LF, SF) on the CIFAR-10 dataset. \algname{BR-DIANA} and \algname{Byz-VR-MARINA} use the $\text{Rand}_k$ compressor, while \algname{Byz-EF21-SGDM}, \algname{Byz-DM21}, and \algname{Byz-VR-DM21} use the $\text{Top}_k$ compressor, where $k=0.1d$.}
\label{cifar_10}
\end{center}
\vskip -0.1in
\end{figure*}

\subsection{Empirical Results on CIFAR-10}
We evaluate our algorithms on an image classiﬁcation task using the CIFAR-10 dataset \cite{krizhevsky2009learning} with the ResNet-20 deep neural network \cite{he2016deep}. For all methods, we use a batch size of $b=64$ and select the step size from the following candidates: $\gamma \in \{0.5,0.05,0.005 \}$. We use $k=0.1d $ for both $\text{Top}_k$ and $\text{Rand}_k$ compressors. The training process is carried out over 100 epochs, iterating over 5000 iterations. To ensure reproducibility, all experiments were conducted using three different random seeds. We report the mean testing accuracy along with one standard error. 

Figure \ref{cifar_10} presents the testing accuracy of five different algorithms—\algname{BR-DIANA}, \algname{Byz-VR-MARINA}, \algname{Byz-EF21-SGDM}, \algname{Byz-DM21}, and \algname{Byz-VR-DM21}—under a 0.4 adversarial setting on the CIFAR-10 dataset. The results show that \algname{Byz-VR-DM21} demonstrates a slight advantage over the other algorithms in all the attack scenarios considered, even in the case of a high proportion of Byzantine workers.
\newpage

\begin{figure*}[ht]
\vskip 0.02in
\begin{center}
\includegraphics[width=\textwidth]{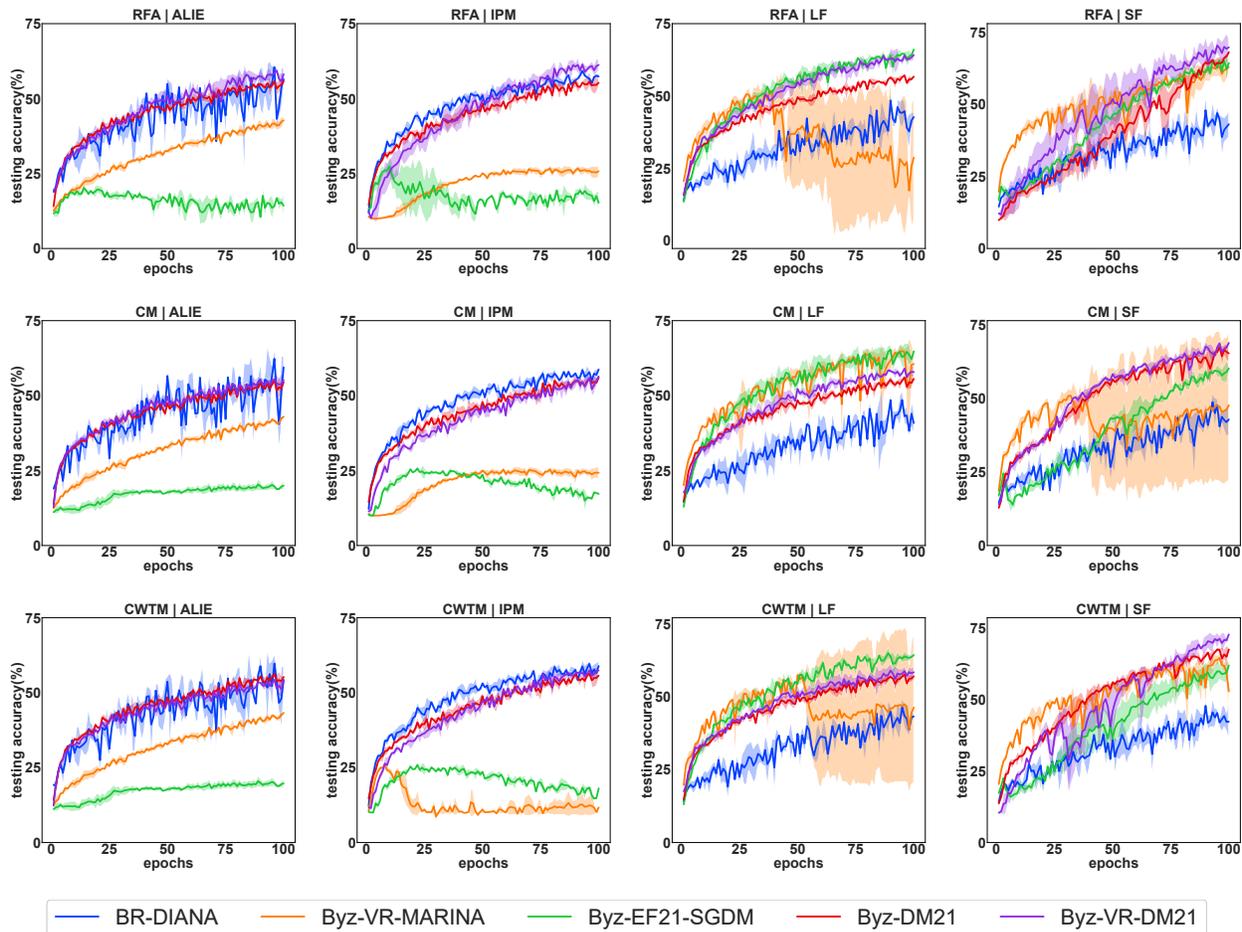}
\caption{The testing accuracy (\%) of 3 aggregation rules (\algname{RFA, CM, CWTM}) under 4 attacks (ALIE, IPM, LF, SF) on the CIFAR-10 dataset with a heterogeneity regime of $a = 0.25$ (high). \algname{BR-DIANA} and \algname{Byz-VR-MARINA} use the $\text{Rand}_k$ compressor, while \algname{Byz-EF21-SGDM}, \algname{Byz-DM21}, and \algname{Byz-VR-DM21} use the $\text{Top}_k$ compressor, where $k=0.1d$.}
\label{cifar_niid}
\end{center}
\vskip -0.1in
\end{figure*}

Figure \ref{cifar_niid} presents the testing accuracy of five different algorithms—\algname{BR-DIANA}, \algname{Byz-VR-MARINA}, \algname{Byz-EF21-SGDM}, \algname{Byz-DM21}, and \algname{Byz-VR-DM21}—under a 0.4 adversarial setting on the CIFAR-10 dataset with a heterogeneity regimes of $a = 0.25$ (high). The results show that \algname{Byz-VR-DM21} demonstrates a slight advantage over the other algorithms in all the attack scenarios considered, even in the case of a high proportion of Byzantine workers and high heterogeneity.

In Table \ref{table:results_cifar}, we present the performance of five algorithms—\algname{BR-DIANA}, \algname{Byz-VR-MARINA}, \algname{Byz-EF21-SGDM}, \algname{Byz-DM21}, and \algname{Byz-VR-DM21}—on the CIFAR-10 dataset. For each algorithm and under every attack, we highlight in bold the algorithm that results in the highest accuracy for the considered scenario.

In Table \ref{table:results_cifar_niid}, we present the performance of five algorithms—\algname{BR-DIANA}, \algname{Byz-VR-MARINA}, \algname{Byz-EF21-SGDM}, \algname{Byz-DM21}, and \algname{Byz-VR-DM21}—on the CIFAR-10 dataset with a heterogeneity regimes of $a = 0.25$ (high). For each algorithm and under every attack, we highlight in bold the algorithm that results in the highest accuracy for the considered scenario.

\newpage

\begin{table*}[ht!]
\centering
\resizebox{\textwidth}{!}{%
\def\arraystretch{1.2}
\begin{tabular}{c|lccccc|c} 
\toprule
\textbf{Aggregation}&\textbf{Method} & \textbf{ALIE} & \textbf{IPM} & \textbf{LF} & \textbf{SF} & \textbf{N.A.} & \textbf{Worst Case}\\ \midrule
 & \algname{BR-DIANA} & $65.86 \pm 02.03$ & $61.03 \pm 00.83$ & 51.96 $\pm$ 02.66 & 51.62 $\pm$ 02.07 & $67.64 \pm 01.63$  & $\mathcolorbox{orange!0}{51.62 \pm 02.07}$\\
& \algname{Byz-VR-MARINA} & 50.76 $\pm$ 24.86 & 42.86 $\pm$ 00.79 & 69.03 $\pm$ 01.17 &  69.63 $\pm$ 02.44 & 71.07 $\pm$ 00.18 & $\mathcolorbox{orange!0}{42.86 \pm 00.79}$\\
\hspace{-0.2cm}\algname{RFA+NNM}\hspace{-0.2cm}
& \algname{Byz-EF21-SGDM} & 64.14 $\pm$ 00.33 & 64.90 $\pm$ 01.35 & 72.75 $\pm$ 01.55 & 80.12 $\pm$ 00.87 & $80.94 \pm 01.12$  & $\mathcolorbox{orange!0}{64.14 \pm 00.33}$\\
& \algname{Byz-DM21} & 63.47 $\pm$ 01.41 & 64.26 $\pm$ 01.00 & 65.98 $\pm$ 02.28 & 78.42 $\pm$ 01.63 & $80.10 \pm 02.14$ & $\mathcolorbox{orange!0}{63.47 \pm 01.41}$\\
& \algname{Byz-VR-DM21} & \textbf{68.63 $\pm$ 01.37 }& \textbf{66.20 $\pm$ 00.64} & \textbf{73.34 $\pm$ 00.51 }& \textbf{81.19 $\pm$ 01.21}  & \textbf{82.33 $\pm$ 00.87} & $\mathcolorbox{orange!0}{66.20 \pm 00.64}$\\
 \midrule
 &\algname{BR-DIANA} & $64.99 \pm 02.05$ & 61.93 $\pm$ 01.02 & 51.27 $\pm$ 01.72 & 50.91 $\pm$ 04.73 & $67.95 \pm 02.05$ & $\mathcolorbox{orange!0}{50.91 \pm 04.73}$\\
 &\algname{Byz-VR-MARINA} & 43.66 $\pm$ 01.24 & 42.02 $\pm$ 00.28 & 70.71 $\pm$ 00.36 & 70.46 $\pm$ 00.37 & $70.17 \pm 01.66$  & $\mathcolorbox{orange!0}{42.02 \pm 00.28}$\\
  \hspace{-0.2cm}\algname{CM+NNM}\hspace{-0.2cm} 
 &\algname{Byz-EF21-SGDM} & 63.42 $\pm$ 00.64 & 63.88 $\pm$ 00.64 & \textbf{71.89 $\pm$ 01.54} & 80.36 $\pm$ 01.06 & 80.98 $\pm$ 01.02  & $\mathcolorbox{orange!0}{63.42 \pm 00.64}$\\
 &\algname{Byz-DM21} & 62.77 $\pm$ 00.45 & 63.96 $\pm$ 01.18 & 66.04 $\pm$ 00.49 & 78.50 $\pm$ 02.53& $80.73 \pm 00.44$   & $\mathcolorbox{orange!0}{62.77 \pm 00.45}$\\
 &\algname{Byz-VR-DM21} & \textbf{66.70 $\pm$ 03.19} & \textbf{64.60 $\pm$ 00.62} & 70.08 $\pm$ 01.99 & \textbf{82.46 $\pm$ 00.59}& \textbf{82.66 $\pm$ 00.31}  & $\mathcolorbox{orange!0}{64.60 \pm 00.62}$\\
  \midrule
 & \algname{BR-DIANA} & 64.29 $\pm$ 02.73 &61.41 $\pm$ 00.97 & 52.09 $\pm$ 01.49 & 49.97 $\pm$ 03.31 & 68.35 $\pm$ 02.42 & \textbf{$\mathcolorbox{orange!0}{49.97 \pm 03.31}$}\\
 &\algname{Byz-VR-MARINA} & 42.68 $\pm$ 01.53 & 42.82 $\pm$ 01.05 & 70.71 $\pm$ 00.67 & 69.75 $\pm$ 01.93  & $70.57 \pm 01.10$ & $\mathcolorbox{orange!0}{42.68 \pm 01.53}$\\
 \hspace{-0.2cm}\algname{CWTM+NNM}\hspace{-0.15cm} 
 &\algname{Byz-EF21-SGDM} & 64.10 $\pm$ 02.96 & 64.56 $\pm$ 00.32 & \textbf{71.78 $\pm$ 01.75} & 79.89 $\pm$ 02.51 & $79.89 \pm 02.51$ & $\mathcolorbox{orange!0}{64.10 \pm 02.96}$\\
 &\algname{Byz-DM21} & 64.90 $\pm$ 00.29 & 63.86 $\pm$ 01.57 &65.89 $\pm$ 01.14 &80.71 $\pm$ 01.59 & 80.93 $\pm$ 01.68 & $\mathcolorbox{orange!0}{63.86 \pm 01.57}$\\
 &\algname{Byz-VR-DM21} & \textbf{68.11 $\pm$ 01.46} & \textbf{67.41 $\pm$ 00.45} & 71.04 $\pm$ 00.77 &  \textbf{81.87 $\pm$ 00.88} & $\textbf{83.08 $\pm$ 00.49}$ & $\mathcolorbox{orange!0}{67.41 \pm 00.45}$\\
\bottomrule
\end{tabular}}
\vspace{0.1cm}
\caption{Maximum testing accuracy ($\%$) across $T=5000$ learning steps on the CIFAR-10 dataset, under four Byzantine attacks strategies. There are $B=8$ Byzantine workers among $n=20$. 'N.A.' denotes the case with no Byzantine attackers. In each of the three horizontal blocks and under each attack, the best accuracy is highlighted in \textbf{bold}. Additionally, for every method, we report the worst-case accuracy across attacks. }
\label{table:results_cifar}
\end{table*}

\begin{table*}[ht!]
\centering
\resizebox{\textwidth}{!}{%
\def\arraystretch{1.2}
\begin{tabular}{c|lccccc|c} 
\toprule
\textbf{Aggregation}&\textbf{Method} & \textbf{ALIE} & \textbf{IPM} & \textbf{LF} & \textbf{SF} & \textbf{N.A.} & \textbf{Worst Case}\\ \midrule
 & \algname{BR-DIANA} & \textbf{60.52 $\pm$ 03.29} & $59.28 \pm 01.08$ & 48.47 $\pm$ 01.74 & 48.00 $\pm$ 04.43 & $62.25 \pm 01.63$  & $\mathcolorbox{orange!0}{48.00 \pm 04.43}$\\
& \algname{Byz-VR-MARINA} & 42.77 $\pm$ 01.34 & 26.76 $\pm$ 01.97 & 51.72 $\pm$ 21.35 &  63.06 $\pm$ 02.42 & 64.15 $\pm$ 02.18 & $\mathcolorbox{orange!0}{26.76 \pm 01.97}$\\
\hspace{-0.2cm}\algname{RFA+NNM}\hspace{-0.2cm}
& \algname{Byz-EF21-SGDM} & 20.23 $\pm$ 04.62 & 27.41 $\pm$ 04.01 & \textbf{66.05 $\pm$ 01.52} & 65.25 $\pm$ 01.54 & $72.62 \pm 00.38$  & $\mathcolorbox{orange!0}{20.23 \pm 04.62}$\\
& \algname{Byz-DM21} & 55.80 $\pm$ 00.61 & 55.64 $\pm$ 01.30 & 57.13 $\pm$ 00.72 & 68.13 $\pm$ 01.00 & $75.77 \pm 00.85$ & $\mathcolorbox{orange!0}{55.64 \pm 01.30}$\\
& \algname{Byz-VR-DM21} & 59.36 $\pm$ 00.92 & \textbf{61.36 $\pm$ 01.67} & 64.15 $\pm$ 01.77& \textbf{70.49 $\pm$ 02.92}  & \textbf{78.96 $\pm$ 00.74} & $\mathcolorbox{orange!0}{59.36 \pm 00.92}$\\
 \midrule
 &\algname{BR-DIANA} & \textbf{62.34 $\pm$ 03.85} & \textbf{58.71 $\pm$ 00.96} & 48.54 $\pm$ 03.82 & 48.81 $\pm$ 05.03 & $63.24 \pm 01.86$ & $\mathcolorbox{orange!0}{47.04 \pm 01.86}$\\
 &\algname{Byz-VR-MARINA} & 42.94 $\pm$ 00.66 & 25.37 $\pm$ 01.86 & 64.84 $\pm$ 04.01 & 50.97 $\pm$ 25.65 & $63.04 \pm 02.61$  & $\mathcolorbox{orange!0}{25.37 \pm 01.86}$\\
  \hspace{-0.2cm}\algname{CM+NNM}\hspace{-0.2cm} 
 &\algname{Byz-EF21-SGDM} & 20.36 $\pm$ 01.84 & 25.70 $\pm$ 00.39 & \textbf{65.31 $\pm$ 02.07} & 60.27 $\pm$ 00.83 & 71.57 $\pm$ 00.97  & $\mathcolorbox{orange!0}{20.36 \pm 01.84}$\\
 &\algname{Byz-DM21} & 55.23 $\pm$ 00.42 & 55.91 $\pm$ 00.44 & 56.08 $\pm$ 00.57 & 67.49 $\pm$ 02.84& $76.59 \pm 01.72$   & $\mathcolorbox{orange!0}{55.23 \pm 00.42}$\\
 &\algname{Byz-VR-DM21} & 55.44 $\pm$ 02.34 & 56.39 $\pm$ 01.93 & 59.05 $\pm$ 01.93 & \textbf{68.96 $\pm$ 01.14}& \textbf{79.57 $\pm$ 01.89}  & $\mathcolorbox{orange!0}{55.44 \pm 02.34}$\\
  \midrule
 & \algname{BR-DIANA} & \textbf{59.71 $\pm$ 05.10} &\textbf{59.58 $\pm$ 01.54 }& 46.39 $\pm$ 00.73 & 47.83 $\pm$ 07.10 & 61.00 $\pm$ 01.25 & \textbf{$\mathcolorbox{orange!0}{46.39 \pm 00.73}$}\\
 &\algname{Byz-VR-MARINA} & 43.28 $\pm$ 00.23 & 25.93 $\pm$ 02.38 & 57.55 $\pm$ 25.73 & 64.58 $\pm$ 12.93  & $64.74 \pm 12.54$ & $\mathcolorbox{orange!0}{25.93 \pm 02.38}$\\
 \hspace{-0.2cm}\algname{CWTM+NNM}\hspace{-0.15cm} 
 &\algname{Byz-EF21-SGDM} & 20.46 $\pm$ 01.78 & 25.91 $\pm$ 00.55 & \textbf{64.98 $\pm$ 00.21} & 61.98 $\pm$ 02.01 & $72.75 \pm 00.77$ & $\mathcolorbox{orange!0}{20.46 \pm 01.78}$\\
 &\algname{Byz-DM21} & 56.27 $\pm$ 00.27 & 56.34 $\pm$ 00.96 &57.71 $\pm$ 01.20 &67.55 $\pm$ 00.59 & 72.30 $\pm$ 05.23 & $\mathcolorbox{orange!0}{56.27 \pm 00.27}$\\
 &\algname{Byz-VR-DM21} & 54.31 $\pm$ 00.32 & 57.41 $\pm$ 00.86 & 59.06 $\pm$ 02.12 &  \textbf{72.78 $\pm$ 02.59} & $\textbf{79.23 $\pm$ 00.71}$ & $\mathcolorbox{orange!0}{54.31 \pm 00.32}$\\
\bottomrule
\end{tabular}}
\vspace{0.1cm}
\caption{Maximum testing accuracy ($\%$) across $T=5000$ learning steps on the CIFAR-10 dataset with a heterogeneity regimes of $a = 0.25$ (high), under four Byzantine attacks strategies. There are $B=8$ Byzantine workers among $n=20$. 'N.A.' denotes the case with no Byzantine attackers. In each of the three horizontal blocks and under each attack, the best accuracy is highlighted in \textbf{bold}. Additionally, For every method, we report the worst-case accuracy across attacks. }
\label{table:results_cifar_niid}
\end{table*}

\newpage

\begin{figure*}[ht]
\vskip 0.02in
\begin{center}
\includegraphics[width=\textwidth]{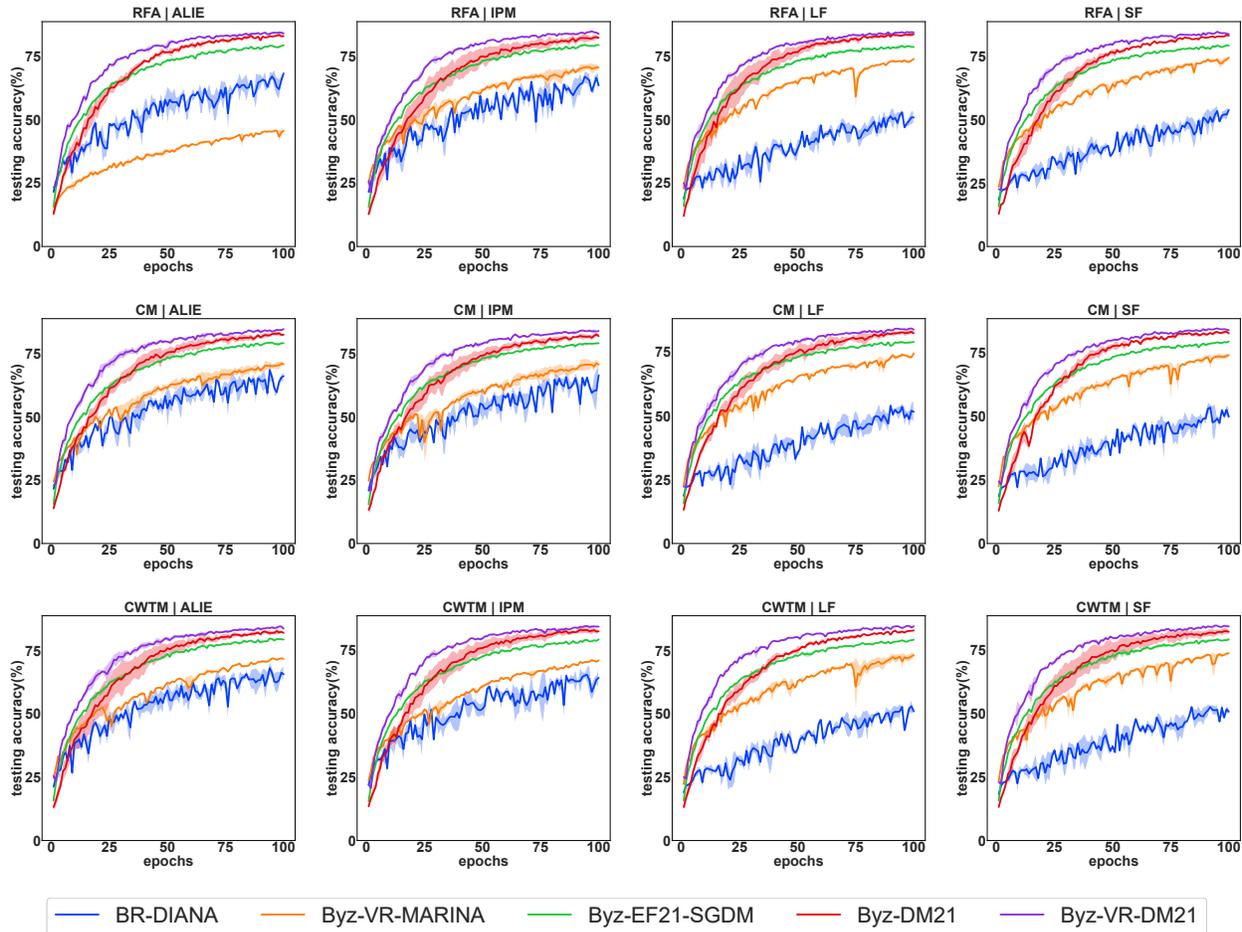}
\caption{The testing accuracy (\%) of 3 aggregation rules (\algname{RFA, CM, CWTM}) under 4 attacks (ALIE, IPM, LF, SF) on the CIFAR-10 dataset. The dataset is uniformly split among 20 workers, including 1 Byzantine worker. \algname{BR-DIANA} and \algname{Byz-VR-MARINA} use the $\text{Rand}_k$ compressor, while \algname{Byz-EF21-SGDM}, \algname{Byz-DM21}, and \algname{Byz-VR-DM21} use the $\text{Top}_k$ compressor, where $k=0.1d$.}
\label{cifar_0.05}
\end{center}
\vskip -0.1in
\end{figure*}
\subsection{Byzantine Ratio Experiments on CIFAR-10}

We evaluate validation accuracy under different Byzantine worker ratios.  The results include three configurations: 19 good workers with 1 Byzantine worker(Figure \ref{cifar_0.05}), 18 good workers with 2 Byzantine workers(Figure \ref{cifar_0.1}), and 16 good workers with 4 Byzantine workers(Figure \ref{cifar_0.2}). For all methods, we use a batch size of $b=64$ and select the step size from the following candidates: $\gamma \in \{0.5,0.05,0.005 \}$. We use $k=0.1d $ for both $\text{Top}_k$ and $\text{Rand}_k$ compressors. The training process is carried out over 100 epochs, iterating over 5000 iterations. To ensure reproducibility, all experiments were conducted using three different random seeds. We report the mean testing accuracy along with one standard error. 
\newpage
\begin{figure*}[ht]
\vskip -0.5in
\begin{center}
\includegraphics[width=\textwidth]{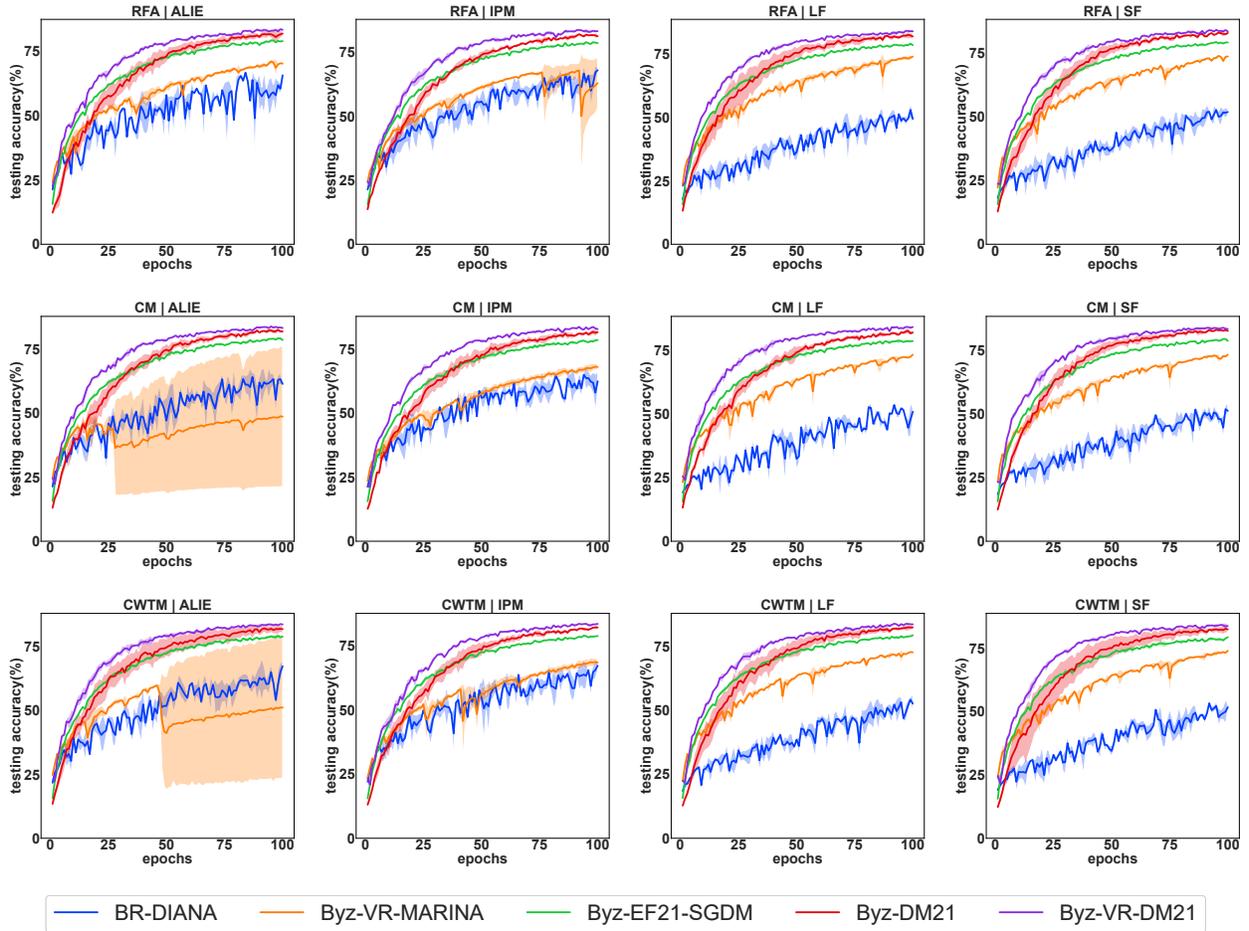}
\caption{The testing accuracy (\%) of 3 aggregation rules (\algname{RFA, CM, CWTM}) under 4 attacks (ALIE, IPM, LF, SF) on the CIFAR-10 dataset. The dataset is uniformly split among 20 workers, including 2 Byzantine workers. \algname{BR-DIANA} and \algname{Byz-VR-MARINA} use the $\text{Rand}_k$ compressor, while \algname{Byz-EF21-SGDM}, \algname{Byz-DM21}, and \algname{Byz-VR-DM21} use the $\text{Top}_k$ compressor, where $k=0.1d$.}
\label{cifar_0.1}
\end{center}
\vskip -0.1in
\end{figure*}
\newpage
\begin{figure*}[ht]
\vskip -0.02in
\begin{center}
\includegraphics[width=\textwidth]{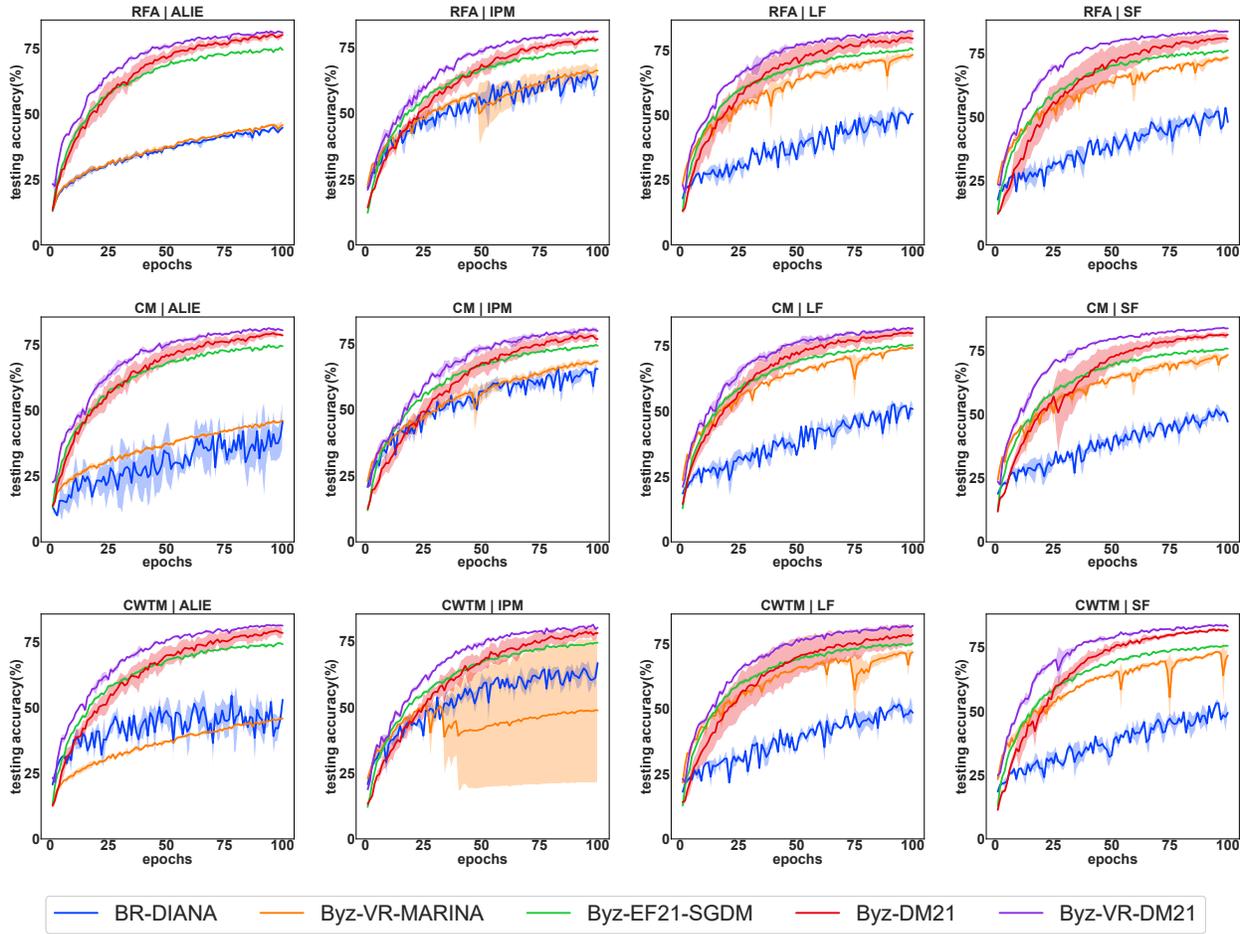}
\caption{The testing accuracy (\%) of 3 aggregation rules (\algname{RFA, CM, CWTM}) under 4 attacks (ALIE, IPM, LF, SF) on the CIFAR-10 dataset. The dataset is uniformly split among 20 workers, including 4 Byzantine workers. \algname{BR-DIANA} and \algname{Byz-VR-MARINA} use the $\text{Rand}_k$ compressor, while \algname{Byz-EF21-SGDM}, \algname{Byz-DM21}, and \algname{Byz-VR-DM21} use the $\text{Top}_k$ compressor, where $k=0.1d$.}
\label{cifar_0.2}
\end{center}
\vskip -0.1in
\end{figure*}

\newpage

\begin{figure*}[h!]
\vskip 0.02in
\begin{center}
\includegraphics[width=\textwidth]{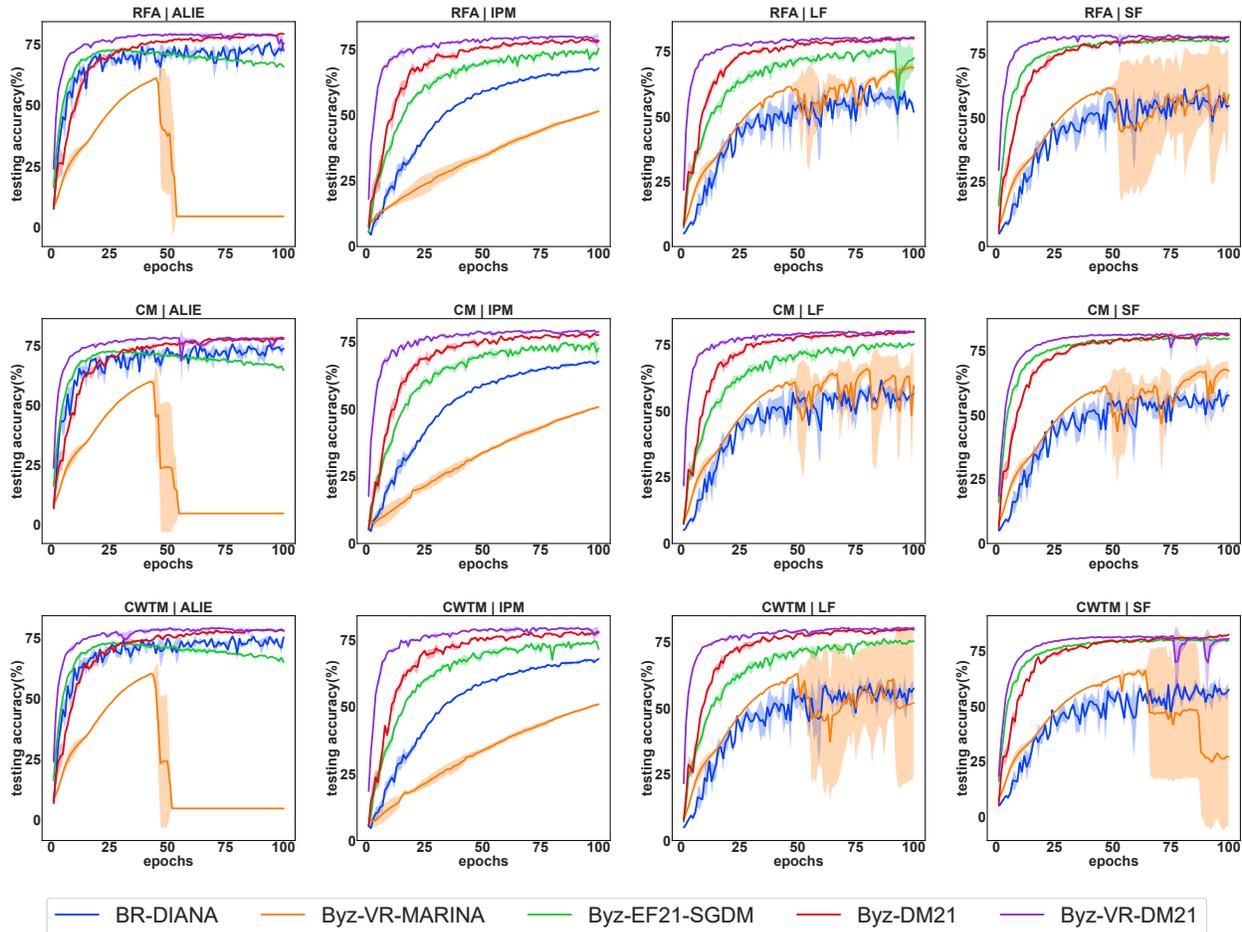}
\caption{The testing accuracy (\%) of 3 aggregation rules (\algname{RFA, CM, CWTM}) under 4 attacks (ALIE, IPM, LF, SF) on the FEMNIST dataset. \algname{BR-DIANA} and \algname{Byz-VR-MARINA} use the $\text{Rand}_k$ compressor, while \algname{Byz-EF21-SGDM}, \algname{Byz-DM21}, and \algname{Byz-VR-DM21} use the $\text{Top}_k$ compressor, where $k=0.1d$.}
\label{femnist}
\end{center}
\vskip -0.1in
\end{figure*}
\subsection{Empirical Results on FEMNIST}

We evaluate our algorithms on an image classiﬁcation task using the FEMNIST dataset \cite{caldas2018leaf} with a convolutional neural network (CNN) \cite{krizhevsky2009learning}. For all methods, we use a batch size of $b=32$ and select the step size from the following candidates :$\gamma \in \{0.5,0.05,0.005 \}$. We use $k=0.1d $ for both $\text{Top}_k$ and $\text{Rand}_k$ compressors. The training process is carried out over 100 epochs, iterating over 5000 iterations. To ensure reproducibility, all experiments were conducted using three different random seeds. We report the mean testing accuracy along with one standard error. 

Figure \ref{femnist} presents the testing accuracy of five different algorithms—\algname{BR-DIANA}, \algname{Byz-VR-MARINA}, \algname{Byz-EF21-SGDM}, \algname{Byz-DM21}, and \algname{Byz-VR-DM21}—under a 0.4 adversarial setting on the FEMNIST dataset. The results show that \algname{Byz-DM21} and \algname{Byz-VR-DM21} demonstrate a slight advantage over the other algorithms in all attack scenarios considered, achieving higher accuracy during convergence, even with a high proportion of Byzantine workers. This highlights the benefits of the double momentum and variance reduction approach.

\newpage
\begin{table*}[ht!]
\centering
\resizebox{\textwidth}{!}{%
\def\arraystretch{1.2}
\begin{tabular}{c|lccccc|c} 
\toprule
\textbf{Aggregation}&\textbf{Method} & \textbf{ALIE} & \textbf{IPM} & \textbf{LF} & \textbf{SF} & \textbf{N.A.} & \textbf{Worst Case}\\ \midrule
 & \algname{BR-DIANA} & 75.74 $\pm$ 00.78 & $67.94 \pm 00.60$ & 61.81 $\pm$ 04.07 & 61.06 $\pm$ 02.34 & $65.16 \pm 00.29$  & $\mathcolorbox{orange!0}{61.81 \pm 04.07}$\\
 & \algname{Byz-VR-MARINA} & 61.18 $\pm$ 00.00 & 51.45 $\pm$ 00.83 & 68.91 $\pm$ 02.00 &  62.77 $\pm$ 16.24 & $70.74 \pm 01.35$ & $\mathcolorbox{orange!0}{51.45 \pm 00.83}$\\
\hspace{-0.2cm}\algname{RFA+NNM}\hspace{-0.2cm}
& \algname{Byz-EF21-SGDM} & 72.94 $\pm$ 01.07 & 75.37 $\pm$ 00.49 & 76.07 $\pm$ 05.02 & 80.23 $\pm$ 00.40 & $80.32 \pm 00.94$  & $\mathcolorbox{orange!0}{72.94 \pm 01.07}$\\
& \algname{Byz-DM21} & \textbf{79.45 $\pm$ 00.34} & 79.01 $\pm$ 00.59 & 80.35 $\pm$ 00.14 & 81.49 $\pm$ 00.44 & $81.85 \pm 00.60$ & $\mathcolorbox{orange!0}{79.01 \pm 00.59}$\\
& \algname{Byz-VR-DM21} & 79.40 $\pm$ 07.33 & \textbf{80.05 $\pm$ 03.27} & \textbf{80.50 $\pm$ 01.12 }& \textbf{81.72 $\pm$ 00.15}  & \textbf{82.02  $\pm$ 00.41} & $\mathcolorbox{orange!0}{79.40 \pm 07.33}$\\
 \midrule
& \algname{BR-DIANA} & $75.71 \pm 01.71$ & 67.75 $\pm$ 00.67 & 61.61 $\pm$ 02.19 & 59.87 $\pm$ 02.81 & 73.41 $\pm$ 01.28 & $\mathcolorbox{orange!0}{59.87 \pm 02.81}$\\

& \algname{Byz-VR-MARINA} & 60.05 $\pm$ 00.00 & 50.85 $\pm$ 00.54 & 66.04 $\pm$ 11.14 & 67.79 $\pm$ 04.05 & 70.79 $\pm$ 04.79  & $\mathcolorbox{orange!0}{50.85 \pm 00.54}$\\
\hspace{-0.2cm}\algname{CM+NNM}\hspace{-0.2cm} 
&\algname{Byz-EF21-SGDM} & 72.95 $\pm$ 03.18 & 74.71 $\pm$ 01.83 & 75.66 $\pm$ 00.23 & 80.37 $\pm$ 00.34 & 80.45 $\pm$ 00.15  & $\mathcolorbox{orange!0}{72.95 \pm 03.18}$\\
&\algname{Byz-DM21} & 78.52 $\pm$ 00.36 & 78.07 $\pm$ 00.40 & 79.90 $\pm$ 00.52 & 81.60 $\pm$ 00.59& $81.65 \pm 00.42$   & $\mathcolorbox{orange!0}{78.07 \pm 00.40}$\\
&\algname{Byz-VR-DM21} & \textbf{78.53 $\pm$ 00.54} & \textbf{79.37 $\pm$ 00.36} & \textbf{80.29 $\pm$ 00.42} & \textbf{81.61 $\pm$ 00.51}& \textbf{81.98 $\pm$ 00.89}  & $\mathcolorbox{orange!0}{78.53 \pm 00.54}$\\
\midrule
 & \algname{BR-DIANA} & 69.48 $\pm$ 01.39 &67.94 $\pm$ 00.61 & 59.11 $\pm$ 03.31 & 59.96 $\pm$ 03.74 & $72.85 \pm 02.88$ & \textbf{$\mathcolorbox{orange!0}{59.11 \pm 03.31}$}\\
&\algname{Byz-VR-MARINA} & 60.18 $\pm$ 00.00 & 50.97 $\pm$ 00.79 & 62.90 $\pm$ 28.64 & 66.20 $\pm$ 31.62  & $69.30 \pm 06.43$ & $\mathcolorbox{orange!0}{50.97 \pm 00.79}$\\ 
\hspace{-0.2cm}\algname{CWTM+NNM}\hspace{-0.15cm}
&\algname{Byz-EF21-SGDM} & 72.94 $\pm$ 03.06 & 74.54 $\pm$ 02.51 & 76.12 $\pm$ 00.09 & 80.39 $\pm$ 00.45 & 80.50 $\pm$ 01.96 & $\mathcolorbox{orange!0}{72.94 \pm 03.06}$\\
&\algname{Byz-DM21} & 78.35 $\pm$ 01.13 & 78.00 $\pm$ 00.55 &79.89 $\pm$ 00.80 & 81.49 $\pm$ 00.45 & $\textbf{82.39 $\pm$ 00.13}$ & $\mathcolorbox{orange!0}{78.00 \pm 00.55}$\\
& \algname{Byz-VR-DM21} & \textbf{78.99 $\pm$ 00.35} & \textbf{79.61 $\pm$ 01.93} & \textbf{80.35 $\pm$ 00.34} &  \textbf{81.67 $\pm$ 00.66} & $81.86 \pm 00.89$ & $\mathcolorbox{orange!0}{78.99 \pm 00.35}$\\

 \bottomrule
\end{tabular}}
\vspace{0.1cm}
\caption{Maximum testing accuracy ($\%$) across $T=5000$ learning steps on the FEMNIST dataset, under four Byzantine attack strategies. There are $B=8$ Byzantine workers among $n=20$. 'N.A.' denotes the case with no Byzantine attackers. In each of the three horizontal blocks and under each attack, the best accuracy is highlighted in \textbf{bold}. Additionally, for every method, we report the worst-case accuracy across attacks. }
\label{table:results_femnist}
\end{table*}

In Table \ref{table:results_femnist}, we present the performance of five algorithms—\algname{BR-DIANA}, \algname{Byz-VR-MARINA}, \algname{Byz-EF21-SGDM}, \algname{Byz-DM21}, and \algname{Byz-VR-DM21}—on the FEMNIST dataset. For each algorithm and under every attack, we highlight in bold the algorithm that results in the highest accuracy for the considered scenario.
\newpage
\subsection{Additional Experiments}
\begin{figure*}[ht]
\vskip 0.02in
\begin{center}
\includegraphics[width=\textwidth]{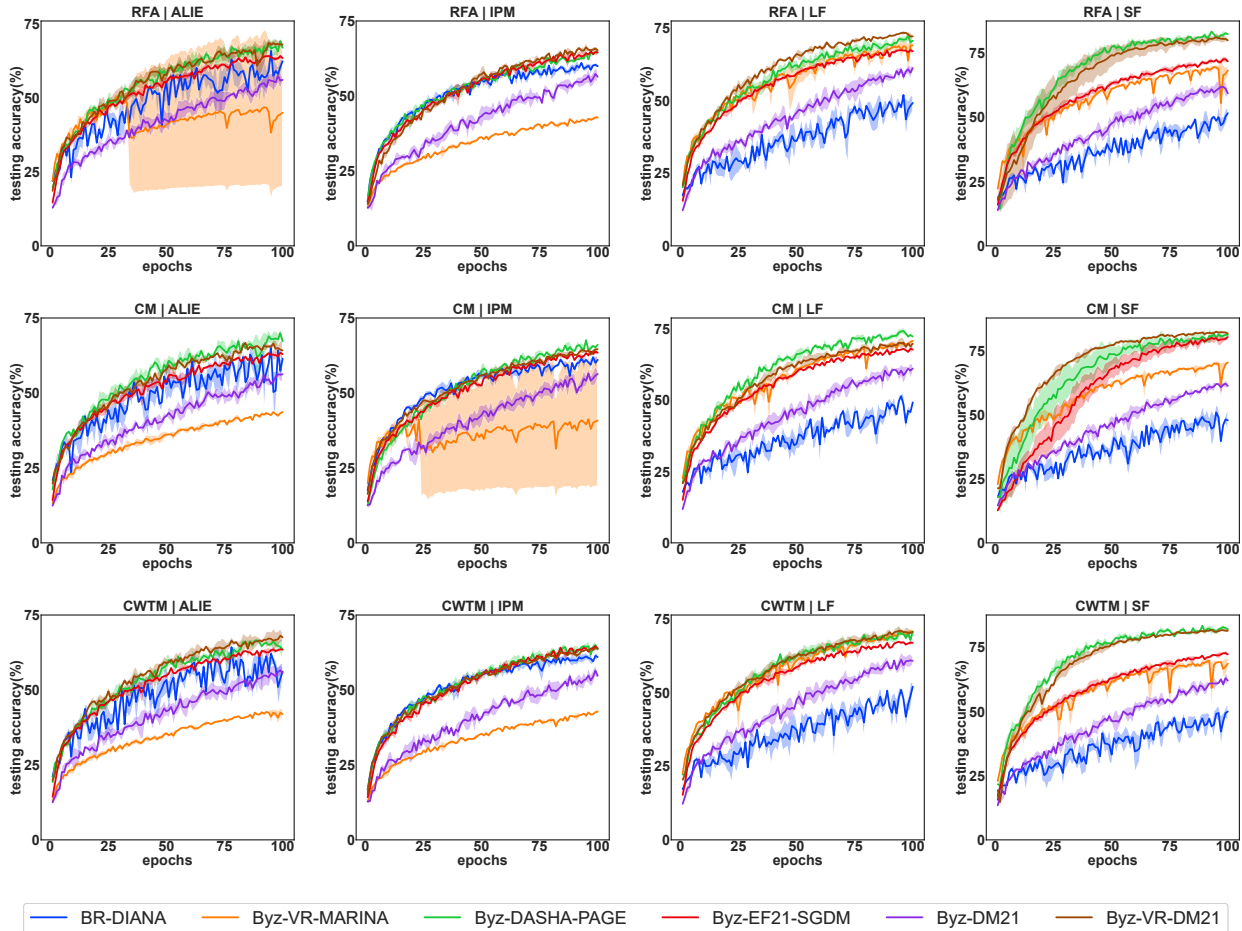}
\caption{The testing accuracy (\%) of 3 aggregation rules (\algname{RFA, CM, CWTM}) under 4 attacks (ALIE, IPM, LF, SF) on the CIFAR-10 dataset. \algname{BR-DIANA}, \algname{Byz-VR-MARINA}, and \algname{Byz-DASHA-PAGE} use the $\text{Rand}_k$ compressor, while \algname{Byz-EF21-SGDM}, \algname{Byz-DM21}, and \algname{Byz-VR-DM21} use the $\text{Top}_k$ compressor, where $k=0.1d$.}
\label{cifar_add}
\end{center}
\vskip -0.1in
\end{figure*}

Figure \ref{cifar_add} presents the testing accuracy of six different algorithms—\algname{BR-DIANA}, \algname{Byz-VR-MARINA}, \algname{Byz-DASHA-PAGE},\algname{Byz-EF21-SGDM}, \algname{Byz-DM21}, and \algname{Byz-VR-DM21}—under a 0.4 adversarial setting on the CIFAR-10 dataset. The results show that \algname{Byz-VR-DM21} consistently outperforms the other baselines across all considered attack scenarios, even when the proportion of Byzantine workers is high; the only exception is \algname{Byz-DASHA-PAGE}. Under our current CIFAR-10 setup and hyperparameters, \algname{Byz-DASHA-PAGE} is very competitive empirically and in some cases slightly better than \algname{Byz-DM21}, and close to \algname{Byz-VR-DM21} in terms of test accuracy. However, this comes at a non-negligible algorithmic cost: \algname{Byz-DASHA-PAGE} requires computing a full local gradient with probability 
 at each iteration, while our methods (\algname{Byz-DM21} and \algname{Byz-VR-DM21}) never rely on full-gradient evaluations and only use stochastic gradients. In federated or large-scale distributed settings, full gradients can be significantly more expensive than stochastic ones, especially when each client holds a large dataset. For this reason, we view \algname{Byz-VR-DM21} as offering a more favorable trade-off between robustness, convergence guarantees, and per-iteration computational cost, even when its final accuracy is comparable to that of \algname{Byz-DASHA-PAGE}.

\newpage

\section{Useful Facts}
For all $ a, b \in \mathbb{R}^d$ and $\alpha > 0, \rho \in  \left(0, 1\right]$ the following relations hold:
\begin{align}
    \lVert a+b\rVert^2\leq&(1+\rho)\lVert a\rVert^2+(1+\rho^{-1})\lVert b\rVert^2\label{eq_young},
    \\
    \lVert a+b+c\rVert^2\leq&3\lVert a\rVert^2+3\lVert b\rVert^2+3\lVert c\rVert^2,
    \\
     \lVert a+b+c+d\rVert^2\leq&4\lVert a\rVert^2+4\lVert b\rVert^2+4\lVert c\rVert^2+4\lVert d\rVert^2.
\end{align}

\textbf{Variance decomposition:} For any random vector $X\in \mathbb{R}^d$ and any non-random vector $c\in \mathbb{R}^d$, we have
\begin{align}\label{eq:vardecomp}
   \mathbb{E}\Big[\lVert X-c \rVert^2\Big] = \mathbb{E}\Big[ \lVert X - \mathbb{E}[ X] \rVert^2 \Big]+ \lVert \mathbb{E}[ X ]-c\rVert^2.
\end{align}

\begin{lemma}[\cite{richtarik2021ef21}]
   \label{lemma:step_lemma}
   Let $a,b>0$. If $0<\gamma\leq\frac{1}{\sqrt{a}+b}$, then $a\gamma^2+b\gamma\leq 1$. 
   Moreover, the bound is tight up to the factor of $2$ since $\frac{1}{\sqrt{a}+b} \leq \min\left\{\frac{1}{\sqrt{a}}, \frac{1}{b}\right\} \leq \frac{2}{\sqrt{a}+b}$.
\end{lemma}

\newpage

\section{Missing Proofs of Byz-DM21 for General Non-Convex Functions}\label{proof_DM21}
Let us state the following lemma used in the analysis of our methods.
\subsection{Supporting Lemmas}
We prepare the following lemmas to facilitate the proof of Theorem \ref{thm:rate}.

\begin{lemma}[Descent lemma from \cite{li2021page}]\label{lem:descent}
    Given an $L$-smooth function $f(x)$. For the update $x^{(t+1)}= x^{(t)}-\gamma g^{(t)}$, there holds
    \begin{equation}
        f(x^{(t+1)}) \leq f(x^{(t)}) - \frac{\gamma}{2} \lVert  \nabla f(x^{(t)}) \rVert^2 - (\frac{1}{2\gamma}- \frac{L}{2}) \lVert 
 x^{(t+1)} - x^{(t)}\rVert^2 + \frac{\gamma}{2}\lVert g^{(t)} - \nabla f(x^{(t)}) \rVert^2. 
    \end{equation}
\end{lemma}

\begin{lemma}[Robust aggregation error]\label{lem:robust_error} Suppose that Assumption \ref{assump:heterogeneity} holds. Then, for all $t \geq 0$ the iterates generated by \algname{Byz-DM21} in Algorithm \ref{alg:SGD2M} satisfy the following condition:
\begin{equation}
    \begin{split}
        & \lVert g^{(t)} - \overline{g}^{(t)} \rVert^2  \\
        &\leq \frac{\kappa}{G}\sum_{i\in \mathcal{G}} \lVert g_i^{(t)}- \overline{g}^{(t)} \rVert^2 \\
        & \leq  \frac{\kappa}{2G^2}\sum_{i,j\in \mathcal{G}} \lVert g_i^{(t)}- g_j^{(t)} \rVert^2 \\
        & \leq \frac{8\kappa(G-1)}{G^2} \sum_{i\in \mathcal{G}} \left( \lVert C_i^{(t)} \rVert^2 + \lVert P_i^{(t)} \rVert^2+ \lVert M_i^{(t)} \rVert^2   \right) + \frac{8\kappa(G-1)}{G} \zeta^2,
            \end{split}
    \end{equation}
\end{lemma}
where $\overline{g}^{(t)} = G^{-1} \sum_{i\in \mathcal{G}}g_i$, $C_i^{(t)}\eqdef g_i^{(t)} - u_i^{(t)}$, $P_i^{(t)}\eqdef u_i^{(t)} - v_i^{(t)}$, $M_i^{(t)}\eqdef v_i^{(t)} - \nabla f_i(x^{(t)})$.

\begin{proof}
Define $H_i^{(t)}\eqdef  \nabla f_i(x^{(t)})-\nabla f(x^{(t)})$. We consider
    \begin{equation*}
    \begin{split}
        & \sum_{i,j\in \mathcal{G}} \lVert g_i^{(t)}- g_j^{(t)} \rVert^2  \\
        & = \sum_{i,j\in \mathcal{G}, i\neq j} \lVert C_i^{(t)}+ P_i^{(t)} + M_i^{(t)} +  H_i^{(t)}- C_j^{(t)}- P_j^{(t)} - M_j^{(t)} -  H_j^{(t)} \rVert^2 \\
        & \leq 8  \sum_{i,j\in \mathcal{G}, i\neq j} 
 \left( \lVert C_i^{(t)} \rVert^2 + \lVert P_i^{(t)} \rVert^2+ \lVert M_i^{(t)} \rVert^2 + \lVert H_i^{(t)} \rVert^2 + \lVert C_j^{(t)} \rVert^2 + \lVert P_j^{(t)} \rVert^2+ \lVert M_j^{(t)} \rVert^2 + \lVert H_j^{(t)} \rVert^2 \right) \\
 & = 16(G-1)  \sum_{i\in \mathcal{G}} \left( \lVert C_i^{(t)} \rVert^2 + \lVert P_i^{(t)} \rVert^2+ \lVert M_i^{(t)} \rVert^2 + \lVert H_i^{(t)} \rVert^2  \right) \\
 & \leq 16(G-1)  \sum_{i\in \mathcal{G}} \left( \lVert C_i^{(t)} \rVert^2 + \lVert P_i^{(t)} \rVert^2+ \lVert M_i^{(t)} \rVert^2   \right) +  16(G-1)G \zeta^2.
    \end{split}
    \end{equation*}
    Since
    \begin{equation*}
      \sum_{i\in \mathcal{G}} \lVert g_i^{(t)}- \overline{g}^{(t)} \rVert^2 = \frac{1}{2G} \sum_{i,j\in \mathcal{G}} \lVert g_i^{(t)}- g_j^{(t)} \rVert^2,
    \end{equation*}
    and the aggregation mechanism is $(B,\kappa)$-robust\footnote{Note that the $B$ represents the number of Byzantine workers in this context, and $n-B$ has been replaced with $G$.}, we have
    \begin{equation*}
    \begin{split}
        & \lVert g^{(t)} - \overline{g}^{(t)} \rVert^2  \\
        &\leq \frac{\kappa}{G}\sum_{i\in \mathcal{G}} \lVert g_i^{(t)}- \overline{g}^{(t)} \rVert^2 \\
        & \leq  \frac{\kappa}{2G^2}\sum_{i,j\in \mathcal{G}} \lVert g_i^{(t)}- g_j^{(t)} \rVert^2 \\
        & \leq \frac{8\kappa(G-1)}{G^2} \sum_{i\in \mathcal{G}} \left( \lVert C_i^{(t)} \rVert^2 + \lVert P_i^{(t)} \rVert^2+ \lVert M_i^{(t)} \rVert^2   \right) + \frac{8\kappa(G-1)}{G} \zeta^2.
            \end{split}
    \end{equation*}
\end{proof}
\begin{lemma}[Accumulated compression error]\label{lem:compression_error} Let Assumption \ref{assump:smoothness} and \ref{assump:bound_variance} be satisfied, and suppose $\mathcal{C}$ is a contractive compressor. For every $i = 1,\ldots,G$, let the sequences $\{v_i^{(t)}\}_{t \geq 0}$, $\{u_i^{(t)}\}_{t \geq 0}$, and $\{g_i^{(t)}\}_{t \geq 0}$ be updated via
\begin{equation*}
\begin{split}
    g_i^{(t)} & = g_i^{(t-1)} + \mathcal{C}(u_i^{(t)}-g_i^{(t-1)}) \\
    u_i^{(t)} & = u_i^{(t-1)} + \eta (v_i^{(t)} - u_i^{(t-1)})\\
    v_i^{(t)} & = v_i^{(t-1)} + \eta (\nabla f_i(x^{(t)},\xi_i^{(t)})-v_i^{(t-1)}).
\end{split}
\end{equation*}

Then for all $t \geq 0$ the iterates generated by \algname{Byz-DM21} in Algorithm \ref{alg:SGD2M} satisfy

    \begin{equation}
    \begin{split}
       \sum_{t=0}^{T-1}\mathbb{E}\left[ \lVert C_i^{(t)}\rVert^2 \right]&=\sum_{t=0}^{T-1}\mathbb{E}\left[ \lVert g_i^{(t)} - u_i^{(t)}\rVert^2 \right] \\
       & \leq  \frac{12\eta^4}{\alpha^2} \sum_{t=0}^{T-1} \mathbb{E}\left[\lVert   \nabla f_i(x^{(t)})-v_i^{(t)}\rVert^2 \right]    + \frac{2   \eta^4 T \sigma^2}{\alpha}\\
        &  \quad   + \frac{12\eta^4L_i^2}{\alpha^2} \sum_{t=0}^{T-1} \mathbb{E}\left[\lVert x^{(t+1)}-x^{(t)}\rVert^2 \right]+ \frac{12\eta^2}{\alpha^2} \mathbb{E}\left[\lVert  u_i^{(t)}-v_i^{(t)}\rVert^2 \right].
    \end{split}
\end{equation}
\end{lemma}

\begin{proof}
By the update rules of $g_i^{(t)}$, $u_i^{(t)}$ and $v_i^{(t)}$, we derive
\begin{equation*}
\begin{split}
     &\mathbb{E}\left[ \lVert g_i^{(t)} - u_i^{(t)}\rVert^2 \right] \\
     &=    \mathbb{E}\left[ \lVert g_i^{(t-1)} - u_i^{(t)} + \mathcal{C}(u_i^{(t)}-g_i^{(t-1)}) \rVert^2 \right]  \\
    & =  \mathbb{E} \left[ \mathbb{E}_{\mathcal{C}} \left[ \lVert u_i^{(t)}- g_i^{(t-1)}  - \mathcal{C}(u_i^{(t)}-g_i^{(t-1)}) \rVert^2 \right] \right]       \\
    & \overset{(i)}{\leq} (1-\alpha)  \mathbb{E}\left[ \lVert u_i^{(t) }-g_i^{(t-1)} \rVert^2 \right] \\
    & = (1-\alpha) \mathbb{E}\left[\lVert u_i^{(t-1)} - g_i^{(t-1)} + \eta (v_i^{(t)} - u_i^{(t-1)})\rVert^2  \right]\\
    & = (1-\alpha) \mathbb{E}\Big[\lVert u_i^{(t-1)} - g_i^{(t-1)} + \eta \Big(v_i^{(t-1)} - u_i^{(t-1)}+\eta(\nabla f_i(x^{(t)}, \xi_i^{(t)})-v_i^{(t-1)})\Big) \rVert^2\Big]\\
    & = (1-\alpha) \mathbb{E}\Big[\lVert u_i^{(t-1)} - g_i^{(t-1)} + \eta \Big(v_i^{(t-1)} - u_i^{(t-1)}+\eta(\nabla f_i(x^{(t)})-v_i^{(t-1)})\\
    &\quad \quad \quad\quad+\eta(\nabla f_i(x^{(t)},\xi_i^{(t)})-\nabla f_i(x^{(t)})\Big) \rVert^2\Big]\\
    & = (1-\alpha) \mathbb{E}\Big[\mathbb{E}_{\xi_i^{(t)}}\Big[\lVert u_i^{(t-1)} - g_i^{(t-1)} + \eta(v_i^{(t-1)} - u_i^{(t-1)})+\eta^2(\nabla f_i(x^{(t)})-v_i^{(t-1)})\\
    &\quad \quad \quad\quad+\eta^2(\nabla f_i(x^{(t)},\xi_i^{(t)})-\nabla f_i(x^{(t)}) \rVert^2\Big]\Big]\\
    & = (1-\alpha) \mathbb{E}\Big[\lVert u_i^{(t-1)} - g_i^{(t-1)} + \eta(v_i^{(t-1)} - u_i^{(t-1)})+\eta^2(\nabla f_i(x^{(t)})-v_i^{(t-1)})\rVert^2\Big]\\
    &\quad \quad \quad\quad+(1-\alpha)\eta^4\mathbb{E}\Big[\lVert\nabla f_i(x^{(t)},\xi_i^{(t)})-\nabla f_i(x^{(t)}) \rVert^2\Big],
    \end{split}
    \end{equation*}
where $(i)$ refers to the contractive property, as defined in Definition \ref{def:contractive}, and then we use inequality~\eqref{eq_young} to obtain

    \begin{equation*}
        \begin{split}
    & \overset{(i)}{\leq}(1-\alpha)(1+\rho)  \mathbb{E}\left[\lVert  u_i^{(t-1)} -g_i^{(t-1)} \rVert^2 \right]  +  \eta^4 \sigma^2\\
    &\quad + (1-\alpha)(1+\rho^{-1}) \mathbb{E}\left[\lVert \eta (v_i^{(t-1)}-u_i^{(t-1)})+\eta^2   (\nabla f_i(x^{(t)})-v_i^{(t-1)})\rVert^2 \right]\\
    & = (1-\alpha)(1+\rho)  \mathbb{E}\left[\lVert  u_i^{(t-1)} -g_i^{(t-1)} \rVert^2 \right]  + (1-\alpha)(1+\rho^{-1}) \mathbb{E}\Big[\lVert \eta (v_i^{(t-1)}-u_i^{(t-1)}) \\
      &\quad +\eta^2   (\nabla f_i(x^{(t-1)})-v_i^{(t-1)})+ \eta^2 (\nabla f_i(x^{(t)}) - \nabla f_i(x^{(t-1)}))\rVert^2 \Big] +  \eta^4 \sigma^2    \\
    & \leq (1-\alpha)(1+\rho)  \mathbb{E}\left[\lVert  u_i^{(t-1)} -g_i^{(t-1)} \rVert^2 \right] + \eta^4 \sigma^2 + 3(1-\alpha)(1+\rho^{-1}) \eta^2\mathbb{E}\left[\lVert v_i^{(t-1)}-u_i^{(t-1)}\rVert^2 \right]     \\
      &\quad  + 3(1-\alpha)(1+\rho^{-1}) \eta^4 \mathbb{E}\left[\lVert \nabla f_i(x^{(t-1)})-v_i^{(t-1)} \rVert^2 \right]\\
      &\quad+ 3(1-\alpha)(1+\rho^{-1}) \eta^4 \mathbb{E}\left[\lVert \nabla f_i(x^{(t)}) - \nabla f_i(x^{(t-1)}) \rVert^2 \right]\\
    & \overset{(ii)}{\leq} (1-\alpha)(1+\rho)  \mathbb{E}\left[\lVert  u_i^{(t-1)} -g_i^{(t-1)} \rVert^2 \right] + \eta^4 \sigma^2 + 3(1-\alpha)(1+\rho^{-1}) \eta^2\mathbb{E}\left[\lVert v_i^{(t-1)}-u_i^{(t-1)}\rVert^2 \right]     \\
      &\quad + 3(1-\alpha)(1+\rho^{-1}) \eta^4 \mathbb{E}\left[\lVert \nabla f_i(x^{(t-1)})-v_i^{(t-1)} \rVert^2 \right]\\
      &\quad+ 3(1-\alpha)(1+\rho^{-1}) \eta^4 L_i^2\mathbb{E}\left[\lVert x^{(t)}-x^{(t-1)} \rVert^2 \right]\,,\\
    \end{split}
    \end{equation*}
where $(i)$ refers to the bounded variance assumption, as stated in Definition \ref{assump:bound_variance}, and $(ii)$ leverages the smoothness property of $f_i(\cdot)$. By setting $\rho = \alpha/2$, we obtain the following result.
\begin{align*}
(1-\alpha)(1+\frac{\alpha}{2}) &= 1-\frac{\alpha}{2}- \frac{\alpha^2}{2}\leq 1-\frac{\alpha}{2} ,
\end{align*}
and
\begin{align*}
(1-\alpha)(1+\rho^{-1}) &= \frac{2}{\alpha}-\alpha -1 \leq \frac{2}{\alpha}.
\end{align*}
Therefore we attain
\begin{equation*}
\begin{split}
     &\mathbb{E}\left[ \lVert g_i^{(t)} - u_i^{(t)}\rVert^2 \right] \\
     &=    \mathbb{E}\left[ \lVert g_i^{(t-1)} - u_i^{(t)} + \mathcal{C}(u_i^{(t)}-g_i^{(t-1)}) \rVert^2 \right]  \\
     & \leq \left(1-\frac{\alpha}{2}\right)  \mathbb{E}\left[\lVert  u_i^{(t-1)} -g_i^{(t-1)} \rVert^2 \right] +  \frac{6\eta^4}{\alpha}\mathbb{E}\left[\lVert   \nabla f_i(x^{(t-1)})-v_i^{(t-1)}\rVert^2 \right]   +  \eta^4 \sigma^2\\
            & \quad+ \frac{6\eta^4L_i^2}{\alpha} \mathbb{E}\left[\lVert x^{(t)}-x^{(t-1)}\rVert^2 \right]     + \frac{6\eta^2}{\alpha} \mathbb{E}\left[\lVert  u_i^{(t-1)}-v_i^{(t-1)}\rVert^2 \right].
\end{split}
\end{equation*}
Summing up the above inequality from $t=0$ to $t=T-1$ leads to
\begin{equation*}
    \begin{split}
        \sum_{t=0}^{T-1}\mathbb{E}\left[ \lVert g_i^{(t)} - u_i^{(t)}\rVert^2 \right] & \leq  \frac{12\eta^4}{\alpha^2} \sum_{t=0}^{T-1} \mathbb{E}\left[\lVert   \nabla f_i(x^{(t)})-v_i^{(t)}\rVert^2 \right]    + \frac{2   \eta^4 T \sigma^2}{\alpha} \\
        &  \quad   + \frac{12\eta^4L_i^2}{\alpha^2} \sum_{t=0}^{T-1} \mathbb{E}\left[\lVert x^{(t+1)}-x^{(t)}\rVert^2 \right]+ \frac{12\eta^2}{\alpha^2} \mathbb{E}\left[\lVert  u_i^{(t)}-v_i^{(t)}\rVert^2 \right].
    \end{split}
\end{equation*}
\end{proof}


\begin{lemma}[Accumulated second momentum deviation]\label{lem:second_momentum} Let Assumption \ref{assump:smoothness} and \ref{assump:bound_variance} be satisfied, and suppose $0 < \eta \leq 1$. For every $i = 1,\ldots,G$, let the sequences $\{v_i^{(t)}\}_{t \geq 0}$ and $\{u_i^{(t)}\}_{t \geq 0}$ be updated via

\begin{equation*}
    \begin{split}
        &v_i^{(t)}=(1-\eta) v_i^{(t-1)} +  \eta \nabla f_i(x^{(t)},\xi_i^{(t)}) \\
        &u_i^{(t)}  = u_i^{(t-1)} + \eta (v_i^{(t)} - u_i^{(t-1)}).
    \end{split}
\end{equation*}
Then for all $t \geq 0$ the iterates generated by \algname{Byz-DM21} in Algorithm \ref{alg:SGD2M} satisfy 
\begin{equation}
\begin{split}
     \sum_{t=0}^{T-1}\mathbb{E}\left[ \lVert P_i^{(t)}\rVert^2 \right]&=\sum_{t=0}^{T-1}\mathbb{E}\left[ \lVert u_i^{(t)} - v_i^{(t)}\rVert^2 \right]\\
     &\leq 6\sum_{t=0}^{T-1}\mathbb{E}\Big[\lVert v_i^{(t)}-\nabla f_i(x^{(t)}\rVert^2\Big]+  6L_i^2\sum_{t=0}^{T-1}\mathbb{E}\Big[\lVert x^{(t)}-x^{(t-1)}\rVert^2\Big]+\eta T\sigma^2.
\end{split}
\end{equation}
\end{lemma}
\begin{proof}
By the update rule of $u_i^{(t)}$ and $v_i^{(t)}$ , we have
\begin{equation*}
    \begin{split}
        &\mathbb{E}\left[ \lVert u_i^{(t)} - v_i^{(t)}\rVert^2 \right] \\
        & =\mathbb{E}\left[ \lVert u_i^{(t-1)} - v_i^{(t)} + \eta (v_i^{(t)} - u_i^{(t-1)})\rVert^2 \right]\\
        & = (1-\eta)^2 \mathbb{E}\left[ \lVert u_i^{(t-1)} - v_i^{(t)}\rVert^2 \right]\\
         & = (1-\eta)^2 \mathbb{E}\Big[ \lVert (1-\eta) v_i^{(t-1)} +  \eta \nabla f_i(x^{(t)},\xi_i^{(t)}) - u_i^{(t-1)}\rVert^2 \Big]\\
         & = (1-\eta)^2 \mathbb{E}\Big[ \lVert  (v_i^{(t-1)}-u_i^{(t-1)}) +  \eta (\nabla f_i(x^{(t)},\xi_i^{(t)})-v_i^{(t-1)})\rVert^2 \Big]\\
         & =(1-\eta)^2 \mathbb{E}\Big[ \lVert  (u_i^{(t-1)}-v_i^{(t-1)}) +  \eta (v_i^{(t-1)}-\nabla f_i(x^{(t)}))+\eta (\nabla f_i(x^{(t)})-\nabla f_i(x^{(t)},\xi_i^{(t)}))\rVert^2 \Big]\\
         & = (1-\eta)^2 \mathbb{E}\Big[ \mathbb{E}_{\xi_i^{(t)}}\Big[\lVert  u_i^{(t-1)}-v_i^{(t-1)} +  \eta (v_i^{(t-1)}-\nabla f_i(x^{(t)}))+\eta (\nabla f_i(x^{(t)})-\nabla f_i(x^{(t)},\xi_i^{(t)}))\rVert^2 \Big]\Big]\\
         & =(1-\eta)^2 \Big(\mathbb{E}\Big[ \lVert  u_i^{(t-1)}-v_i^{(t-1)} +  \eta (v_i^{(t-1)}-\nabla f_i(x^{(t)}))\rVert^2 \Big]+\eta^2\mathbb{E}\Big[ \lVert \nabla f_i(x^{(t)},\xi_i^{(t)})-\nabla f_i(x^{(t)})\rVert^2 \Big]\Big)\\
         &\overset{(i)}{\leq} (1-\eta)^2 \mathbb{E}\Big[ \lVert  u_i^{(t-1)}-v_i^{(t-1)} +  \eta (v_i^{(t-1)}-\nabla f_i(x^{(t)}))\rVert^2 \Big]+\eta^2\sigma^2\\
         &\overset{(ii)}{\leq} (1-\eta)^2 (1+\rho)\mathbb{E}\Big[ \lVert  u_i^{(t-1)}-v_i^{(t-1)}\rVert^2 \Big] +  \eta^2 (1+\rho^{-1})\mathbb{E}\Big[ \lVert v_i^{(t-1)}-\nabla f_i(x^{(t)})\rVert^2 \Big]+\eta^2\sigma^2,\\
\end{split}
\end{equation*}
where $(i)$ utilizes Assumption \ref{assump:bound_variance}, $(ii)$ holds by inequality~\eqref{eq_young}. Setting $\rho=\eta/2$, and because of $\eta \in \left(0,1\right]$, we obtain
\begin{equation*}
(1-\eta)^2(1+\frac{\eta}{2}) = 1-\frac{3\eta}{2}+ \frac{\eta^3}{2}\leq 1-\eta,
\end{equation*}
and 
\begin{equation*}
\eta^2(1+\frac{2}{\eta}) = \eta^2+2\eta \leq 3\eta\,.
\end{equation*}
There holds
\begin{equation*}
    \begin{split}
        &\mathbb{E}\left[ \lVert u_i^{(t)} - v_i^{(t)}\rVert^2 \right] \\
        &\leq (1-\eta)\mathbb{E}\Big[ \lVert  u_i^{(t-1)}-v_i^{(t-1)}\rVert^2 \Big] +  3\eta\mathbb{E}\Big[ \lVert v_i^{(t-1)}-\nabla f_i(x^{(t)})\rVert^2 \Big]+\eta^2\sigma^2\\
        &\overset{(i)}{\leq} (1-\eta)\mathbb{E}\Big[ \lVert  u_i^{(t-1)}-v_i^{(t-1)}\rVert^2 \Big] +  6\eta\mathbb{E}\Big[ \lVert v_i^{(t-1)}-\nabla f_i(x^{(t-1)})\rVert^2 \Big]\\
        &\quad+6\eta\mathbb{E}\Big[ \lVert \nabla f_i(x^{(t)})-\nabla f_i(x^{(t-1)})\rVert^2 \Big]+\eta^2\sigma^2\\
        &\overset{(ii)}{\leq} (1-\eta)\mathbb{E}\Big[ \lVert  u_i^{(t-1)}-v_i^{(t-1)}\rVert^2 \Big] +  6\eta\mathbb{E}\Big[ \lVert v_i^{(t-1)}-\nabla f_i(x^{(t-1)})\rVert^2 \Big]\\
        &\quad+6\eta L_i^2\mathbb{E}\Big[ \lVert x^{(t)}-x^{(t-1)}\rVert^2 \Big]+\eta^2\sigma^2,
    \end{split}
\end{equation*}
where $(i)$ uses inequality~\eqref{eq_young}, and $(ii)$ uses smoothness of $f_i(\cdot)$. Summing up the above inequality from $t = 0$ to $t = T-1$ yields
\begin{equation*}
    \begin{split}
        \sum_{t=0}^{T-1}\mathbb{E}\left[ \lVert P_i^{(t)}\rVert^2 \right] & \leq 6\sum_{t=0}^{T-1}\mathbb{E}\Big[\lVert v_i^{(t)}-\nabla f_i(x^{(t)}\rVert^2\Big]+6L_i^2\sum_{t=0}^{T-1}\mathbb{E}\Big[\lVert x^{(t)}-x^{(t-1)}\rVert^2\Big]\\
        &\quad+ \frac{1}{\eta}\sum_{t=0}^{T-1}\mathbb{E}\left[ \lVert P_i^{(0)}\rVert^2 \right]+\eta T\sigma^2\\
        &\leq 6\sum_{t=0}^{T-1}\mathbb{E}\Big[\lVert v_i^{(t)}-\nabla f_i(x^{(t)}\rVert^2\Big]+  6L_i^2\sum_{t=0}^{T-1}\mathbb{E}\Big[\lVert x^{(t)}-x^{(t-1)}\rVert^2\Big]+\eta T\sigma^2,
    \end{split}
\end{equation*}
where $ P_i^{(0)}=u_i^{(0)}-v_i^{(0)}$, because of $u_i^{(0)}=v_i^{(0)}$, $P_i^{(0)}=0$.
\end{proof}

\begin{lemma}[Accumulated momentum deviation]\label{lem:momentum} Let Assumption \ref{assump:smoothness} and \ref{assump:bound_variance} be satisfied, and suppose $0 < \eta \leq 1$. For every $i = 1,\ldots,G$, let the sequences $\{v_i^{(t)}\}_{t \geq 0}$ be updated via

\begin{equation*}
    \begin{split}
        &v_i^{(t)}=(1-\eta) v_i^{(t-1)} +  \eta \nabla f_i(x^{(t)},\xi_i^{(t)}) .
    \end{split}
\end{equation*}
Then for all $t \geq 0$ the iterates generated by \algname{Byz-DM21} in Algorithm \ref{alg:SGD2M} satisfy 
\begin{equation}
\begin{split}
     \frac{1}{G}\sum_{t=0}^{T-1} \sum_{i\in\mathcal{G}} \mathbb{E}\Big[\lVert M_i^{(t)} \rVert^2\Big]&= \frac{1}{G}\sum_{t=0}^{T-1} \sum_{i\in\mathcal{G}} \mathbb{E}\Big[\lVert v_i^{(t)}-\nabla f_i(x^{(t)}) \rVert^2\Big]  \\
     &\leq \frac{ \widetilde L^2}{\eta^2}\sum_{t=0}^{T-1}\mathbb{E} \left[  \lVert x^{(t+1)}-x^{(t)}  \rVert^2 \right]+ \eta T \sigma^2 + \frac{1}{\eta G} \sum_{i\in\mathcal{G}} \mathbb{E} \left[  \lVert v_i^{(0)}-\nabla f_i(x^{(0)})  \rVert^2 \right],     
\end{split}
\end{equation}
and
\begin{equation}
\begin{split}
    \sum_{t=0}^{T-1} \mathbb{E}\Big[\lVert \widetilde M_i^{(t)}\rVert^2\Big]& = \sum_{t=0}^{T-1}  \mathbb{E}\Big[\lVert \overline{v}^{(t)} - \nabla f(x^{(t)}) \rVert^2\Big] \\
    &\leq \frac{ L^2}{\eta^2}\sum_{t=0}^{T-1}\mathbb{E} \left[  \lVert x^{(t+1)}-x^{(t)}  \rVert^2 \right]+ \frac{\eta T \sigma^2}{G} + \frac{1}{\eta } \mathbb{E} \left[  \lVert \overline{v}^{(0)}-\nabla f_{\mathcal{G}}(x^{(0)})  \rVert^2 \right].      
\end{split}
\end{equation} 
\end{lemma}

\begin{proof}
By the update rule of $v_i^{(t)}$ , and consider

\begin{equation*}
\begin{split}
    & \lVert v_i^{(t)} - \nabla f_i(x^{(t)}) \rVert^2 \\
    & = \lVert  (1-\eta)v_i^{(t-1)} + \eta \nabla f_i(x^{(t)}_i, \xi_i^{(t)}) -\nabla f_i(x^{(t)})\rVert^2\\
    & = \lVert  (1-\eta)(v_i^{(t-1)}-\nabla f_i(x^{(t)})) + \eta (\nabla f_i(x^{(t)}_i, \xi_i^{(t)}) - \nabla f_i(x^{(t)}))\rVert^2.\\
\end{split}
\end{equation*}
Taking the expectation on both sides and using the law of total expectation, we obtain
\begin{equation*}
\begin{split}
    & \mathbb{E}\Big[\lVert v_i^{(t)} - \nabla f_i(x^{(t)}) \rVert^2\Big] \\
    & =  \mathbb{E} \Big[ \mathbb{E}_{\xi_i^{(t)}} \Big[\lVert  (1-\eta)(v_i^{(t-1)}-\nabla f_i(x^{(t)})) + \eta (\nabla f_i(x^{(t)}_i, \xi_i^{(t)}) - \nabla f_i(x^{(t)}))  \rVert^2 \Big] \Big],
\end{split}
\end{equation*}
there holds
\begin{equation*}
    \begin{split}
    & \mathbb{E}\Big[\lVert v_i^{(t)} - \nabla f_i(x^{(t)}) \rVert^2\Big] \\
    & = (1-\eta)^2\mathbb{E}\Big[\lVert v_i^{(t-1)}-\nabla f_i(x^{(t)})\rVert^2\Big] + \eta^2\mathbb{E}\Big[\lVert \nabla f_i(x^{(t)}, \xi_i^{(t)}) - \nabla f_i(x^{(t)})\rVert^2\Big]\\
    & \overset{(i)}{\leq}  (1-\eta)^2(1+a)\mathbb{E} \left[  \lVert v_i^{(t-1)}-\nabla f_i(x^{(t-1)})  \rVert^2 \right]  +\eta^2\sigma^2 \\
    &\quad  + (1-\eta)^2(1+a^{-1})\mathbb{E} \left[  \lVert \nabla f_i(x^{(t)})-\nabla f_i(x^{(t-1)})  \rVert^2 \right].
\end{split}
\end{equation*}
where $(i)$ uses Assumption \ref{assump:bound_variance} and inequality~\eqref{eq_young}. for any $a>0$,
we take $a=\eta(1-\eta)^{-1}$ and use the $L$-smoothness of $f_i(\cdot)$ to obtain
\begin{equation*}
    \begin{split}
    & \mathbb{E}\Big[\lVert M_i^{(t)} \rVert^2\Big] \\
    & = \mathbb{E}\Big[\lVert v_i^{(t)} - \nabla f_i(x^{(t)}) \rVert^2\Big]\\
    & \leq  (1-\eta)\mathbb{E} \left[  \lVert M_i^{(t-1)}  \rVert^2 \right]   + \frac{ L_i^2}{\eta} \mathbb{E} \left[  \lVert x^{(t)}-x^{(t-1)}  \rVert^2 \right] + \eta^2 \sigma^2.
\end{split}
\end{equation*}
Summing up the above inequality over all $i\in \mathcal{G}$ and from $t=0$ to $t=T-1$ yields
\begin{equation*}
\begin{split}
     & \frac{1}{G} \sum_{t=0}^{T-1} \sum_{i\in\mathcal{G}} \mathbb{E}\Big[\lVert M_i^{(t)} \rVert^2\Big]  \\
     &\leq \frac{ \widetilde L^2}{\eta^2}\sum_{t=0}^{T-1}\mathbb{E} \left[  \lVert x^{(t+1)}-x^{(t)}  \rVert^2 \right]+ \eta T \sigma^2 + \frac{1}{\eta G} \sum_{i\in\mathcal{G}} \mathbb{E} \left[  \lVert v_i^{(0)}-\nabla f_i(x^{(0)})  \rVert^2 \right].      
\end{split}
\end{equation*}

Using the same arguments, we obtain
\begin{equation*}
    \begin{split}
    & \mathbb{E}\Big[\lVert \overline{v}^{(t)} - \nabla f_{\mathcal{G}}(x^{(t)}) \rVert^2\Big] \\
    & \quad\leq  (1-\eta)\mathbb{E} \Big[  \lVert \overline{v}^{(t-1)}-\nabla f_{\mathcal{G}}(x^{(t-1)})  \rVert^2 \Big]   + \frac{  L^2}{\eta} \mathbb{E} \Big[  \lVert x^{(t-1)}-x^{(t)}  \rVert^2 \Big] + \frac{\eta^2 \sigma^2}{G},
\end{split}
\end{equation*}
and
\begin{equation*}
\begin{split}
    & \sum_{t=0}^{T-1}\mathbb{E}\Big[\lVert \overline{v}^{(t)} - \nabla f_{\mathcal{G}}(x^{(t)}) \rVert^2\Big] \\
    &\leq \frac{ L^2}{\eta^2}\sum_{t=0}^{T-1}\mathbb{E} \left[  \lVert x^{(t+1)}-x^{(t)}  \rVert^2 \right]+ \frac{\eta T \sigma^2}{G} + \frac{1}{\eta } \mathbb{E} \left[  \lVert \overline{v}^{(0)}-\nabla f_{\mathcal{G}}(x^{(0)})  \rVert^2 \right].      
\end{split}
\end{equation*}
\end{proof}

\subsection{Proof of Theorem \ref{thm:rate}}
\begin{proof} By Lemma \ref{lem:descent}, there holds, for any $\gamma \leq 1/(2L)$,
\begin{equation}\label{descent}
    f(x^{(t+1)}) \leq f(x^{(t)}) - \frac{\gamma}{2} \lVert  \nabla f(x^{(t)}) \rVert^2 - \frac{1}{4\gamma} \lVert 
 x^{(t+1)} - x^{(t)}\rVert^2 + \frac{\gamma}{2}\lVert g^{(t)} - \nabla f(x^{(t)}) \rVert^2.
\end{equation}
    Summing the above from $t = 0$ to $t = T - 1$ and taking expectation, we define $\delta_t \eqdef \mathbb{E}\Big[ \nabla f(x^{(t)})-f(x^*)\Big]$, then we have 
    \begin{equation*}
    \begin{split}
        \frac{1}{T} \sum_{t=0}^{T-1} \mathbb{E}\Big[\lVert \nabla f(x^{(t)}) \rVert^2\Big] \leq \frac{2\delta}{\gamma T}-\frac{1}{2\gamma^2 T}\sum_{t=0}^{T-1}\mathbb{E}\left[\lVert x^{(t+1)}-x^{(t)}\rVert^2\right]+\frac{1}{T}\sum_{t=0}^{T-1}\mathbb{E}\left[\lVert g^{(t)}-\nabla f (x^{(t)})\rVert^2\right].
    \end{split}
\end{equation*}
In order to control the error between $g^{(t)}$ and $\nabla f(x^{(t)})$, we decompose it into four terms
\begin{equation*}
    \begin{split}
        \lVert g^{(t)} - \nabla f(x^{(t)})\rVert^2 &\leq 4\lVert g^{(t)}- \overline{g}^{(t)} \rVert^2 + 4\lVert \overline g^{(t)}- \overline{u}^{(t)} \rVert^2 + 4\lVert \overline u^{(t)}- \overline{v}^{(t)} \rVert^2 + 4\lVert \overline v^{(t)}- \nabla f(x^{(t)}) \rVert^2 \\
        & \leq 4\lVert g^{(t)}- \overline{g}^{(t)} \rVert^2 + \frac{4}{G}\sum_{i\in \mathcal{G}} \lVert g_i^{(t)} - u_i^{(t)}\rVert^2\\
        &\quad+ \frac{4}{G}\sum_{i\in \mathcal{G}} \lVert u_i^{(t)} - v_i^{(t)}\rVert^2 +4 \lVert \overline{v}^{(t)} - \nabla f(x^{(t)})\rVert^2,
    \end{split}
\end{equation*}
where $\overline{g} = G^{-1} \sum_{i\in\mathcal{G}} g_i$, $\overline{u} = G^{-1} \sum_{i\in\mathcal{G}} u_i$ and $\overline{v} = G^{-1} \sum_{i\in\mathcal{G}} v_i$.

Next, we apply the technical lemmas from the previous section to derive a bound on the deviation between $g^{(t)}$ and $\nabla f(x^{(t)})$. First, we invoke Lemma \ref{lem:robust_error} to obtain the following result
\begin{equation}\label{bound_deviation}
    \begin{split}
    &\lVert g^{(t)} - \nabla f(x^{(t)}) \rVert^2\\
    &\leq \frac{32\kappa(G-1)}{G^2} \sum_{i\in \mathcal{G}} \left( \lVert C_i^{(t)} \rVert^2 + \lVert P_i^{(t)} \rVert^2+ \lVert M_i^{(t)} \rVert^2   \right) + \frac{32\kappa(G-1)}{G} \zeta^2\\
    &\quad + \frac{4}{G}\sum_{i\in \mathcal{G}} \lVert C_i^{(t)}\rVert^2+ \frac{4}{G}\sum_{i\in \mathcal{G}} \lVert P_i^{(t)}\rVert^2 +4 \lVert \widetilde M_i^{(t)}\rVert^2\\
    &\leq \frac{4(8\kappa+1)}{G} \sum_{i\in \mathcal{G}} \lVert C_i^{(t)}  \rVert^2 + \frac{4(8\kappa+1)}{G} \sum_{i\in \mathcal{G}} \lVert P_i^{(t)}  \rVert^2 + \frac{32\kappa}{G} \sum_{i\in \mathcal{G}} \lVert M_i^{(t)}  \rVert^2\\
    &\quad+ 4 \lVert \widetilde M_i^{(t)}  \rVert^2 +32\kappa \zeta^2,
    \end{split}
\end{equation}
where $C_i^{(t)} = g_i^{(t)}-u_i^{(t)}$, $P_i^{(t)} = u_i^{(t)}-v_i^{(t)}$, $ \widetilde M_i^{(t)} = \overline{v}_i^{(t)}- \nabla f_i(x^{(t)})$. By summing up the above inequality from $t = 0$ to $t= T-1$, we obtain
\begin{equation*}
    \begin{split}
        &\sum_{t=0}^{T-1} \lVert g^{(t)} - \nabla f(x^{(t)}) \rVert^2\\
        & \leq \frac{4(8\kappa+1)}{G} \sum_{t=0}^{T-1}\sum_{i\in \mathcal{G}} \lVert C_i^{(t)}  \rVert^2 + \frac{4(8\kappa+1)}{G} \sum_{t=0}^{T-1}\sum_{i\in \mathcal{G}} \lVert P_i^{(t)}  \rVert^2 + \frac{32\kappa}{G} \sum_{t=0}^{T-1}\sum_{i\in \mathcal{G}} \lVert M_i^{(t)}  \rVert^2\\
    &\quad+ 4 \sum_{t=0}^{T-1} \lVert \widetilde M_i^{(t)}  \rVert^2 +32\kappa T\zeta^2.
    \end{split}
\end{equation*}
Next, by taking the expectation and using Lemma \ref{lem:compression_error}, and define $R^{(t)}\eqdef \mathbb{E}\Big[\lVert x^{(t+1)}-x^{(t)}\rVert^2\Big]$, we have
\begin{equation*}
    \begin{split}
        &\sum_{t=0}^{T-1} \mathbb{E}\Big[\lVert g^{(t)}-\nabla f(x^{(t)})\rVert^2\Big]\\
        &\leq \frac{4(8\kappa+1)}{G} \sum_{i\in \mathcal{G}} \Big(\frac{12\eta^4}{\alpha^2} \sum_{t=0}^{T-1} \mathbb{E}\left[\lVert    M_i^{(t)}\rVert^2 \right]    + \frac{2   \eta^4 T \sigma^2}{\alpha}   + \frac{12\eta^4L_i^2}{\alpha^2} \sum_{t=0}^{T-1} \mathbb{E}\left[\lVert R^{(t)}\rVert^2 \right]+ \frac{12\eta^2}{\alpha^2} \mathbb{E}\left[\lVert  P_i^{(t)}\rVert^2 \right]\Big)\\
        &  \quad+   \frac{4(8\kappa+1)}{G} \sum_{t=0}^{T-1}\sum_{i\in \mathcal{G}} \mathbb{E}\Big[\lVert P_i^{(t)}  \rVert^2 \Big] + \frac{32\kappa}{G} \sum_{t=0}^{T-1}\sum_{i\in \mathcal{G}} \mathbb{E}\Big[\lVert M_i^{(t)}  \rVert^2 \Big]+ 4 \sum_{t=0}^{T-1}\Big[\lVert \widetilde M_i^{(t)}  \rVert^2 \Big]+32\kappa T\zeta^2 \\
        & \leq \Big(\frac{48\eta^4(8\kappa+1)}{\alpha^2 G}+\frac{32\kappa}{G}\Big) \sum_{t=0}^{T-1}\sum_{i\in \mathcal{G}} \mathbb{E} \Big[\lVert M_i^{(t)}\rVert^2 \Big] + \frac{8(8\kappa+1)\eta^4 T \sigma^2}{\alpha}+\frac{48\eta^4\widetilde L^2(8\kappa+1)}{\alpha^2}\sum_{t=0}^{T-1} \mathbb{E} \Big[\lVert R^{(t)}\rVert^2 \Big]\\
        &\quad + (\frac{\alpha^2+12\eta^2}{\alpha^2})(\frac{4(8\kappa+1)}{G}) \sum_{t=0}^{T-1}\sum_{i\in \mathcal{G}}  \mathbb{E}\Big[ \lVert P_i^{(t)}  \rVert^2 \Big]+ 4 \sum_{t=0}^{T-1} \mathbb{E}\Big[\lVert \widetilde M_i^{(t)}  \rVert^2 \Big]+32\kappa T\zeta^2. 
    \end{split}
\end{equation*}
Then, by using Lemma \ref{lem:second_momentum}, we get
\begin{equation*}
    \begin{split}
        &\sum_{t=0}^{T-1} \mathbb{E}\Big[\lVert g^{(t)}-\nabla f(x^{(t)})\rVert^2\Big]\\
        &\leq \Big(\frac{48\eta^4(8\kappa+1)}{\alpha^2 G}+ \frac{32\kappa}{G} \Big)\sum_{t=0}^{T-1}\sum_{i\in \mathcal{G}} \mathbb{E}\left[\lVert M_i^{(t)}  \rVert^2\right]+ 4 \sum_{t=0}^{T-1} \mathbb{E}\left[\lVert \widetilde M_i^{(t)}  \rVert^2\right]+32\kappa T\zeta^2\\
        &\quad +\frac{48\eta^4\widetilde L^2(8\kappa+1)}{\alpha^2}\sum_{t=0}^{T-1} \mathbb{E} \Big[\lVert R^{(t)}\rVert^2 \Big]+\frac{8(8\kappa+1)\eta^4 T \sigma^2}{\alpha}\\
        &\quad+(\frac{\alpha^2+12\eta^2}{\alpha^2})(\frac{4(8\kappa+1)}{G})\sum_{i\in \mathcal{G}}\Big(6\sum_{t=0}^{T-1}\mathbb{E}\Big[\lVert M_i^{(t)}\rVert^2\Big]+  6L_i^2\sum_{t=0}^{T-1}\mathbb{E}\Big[\lVert R^{(t)}\rVert^2\Big]+\eta T\sigma^2\Big)\\
        &\leq \Big(\frac{(48\eta^4+288\eta^2+24\alpha^2)(8\kappa+1)+32\kappa\alpha^2}{\alpha^2 G}\Big) \sum_{t=0}^{T-1}\sum_{i\in \mathcal{G}} \mathbb{E} \Big[\lVert M_i^{(t)}\rVert^2 \Big] + \frac{8(8\kappa+1)\eta^4 T \sigma^2}{\alpha}\\
        &\quad+4 \sum_{t=0}^{T-1} \mathbb{E}\left[\lVert \widetilde M_i^{(t)}  \rVert^2\right]+(\frac{4(8\kappa+1)(\alpha^2+12\eta^2)}{\alpha^2})\eta T\sigma^2+32\kappa T\zeta^2 \\
        &\quad +\Big(\frac{24\widetilde L^2(8\kappa+1)(12\eta^4+\alpha^2+12\eta^2)}{\alpha^2}\Big)\sum_{t=0}^{T-1} \mathbb{E} \Big[\lVert R^{(t)}\rVert^2 \Big].
    \end{split}
\end{equation*}

Furthermore, by using Lemma \ref{lem:momentum}, we get
\begin{equation*}
    \begin{split}
        &\sum_{t=0}^{T-1} \mathbb{E}\Big[\lVert g^{(t)}-\nabla f(x^{(t)})\rVert^2\Big]\\
        &\leq \Big(\frac{(48\eta^4+288\eta^2+24\alpha^2)(8\kappa+1)+32\kappa\alpha^2}{\alpha^2 } \Big)\Big(\frac{ \widetilde L^2}{\eta^2}\sum_{t=0}^{T-1}\mathbb{E} \left[  \lVert R^{(t)}  \rVert^2 \right]+ \eta T \sigma^2 + \frac{1}{\eta G} \sum_{i\in\mathcal{G}} \mathbb{E} \left[  \lVert M_i^{(0)}\rVert^2 \right]\Big)\\
        &\quad+ \frac{8(8\kappa+1)\eta^4 T \sigma^2}{\alpha}+\frac{4 L^2}{\eta^2}\sum_{t=0}^{T-1}\mathbb{E} \left[  \lVert R^{(t)}  \rVert^2 \right]+ \frac{4\eta T \sigma^2}{G} + \frac{4}{\eta } \mathbb{E} \left[  \lVert \widetilde M_i^{(0)}  \rVert^2 \right]+32\kappa T\zeta^2\\
        &\quad +(\frac{4(8\kappa+1)(\alpha^2+12\eta^2)}{\alpha^2})\eta T\sigma^2+\Big(\frac{24\widetilde L^2(8\kappa+1)(12\eta^4+\alpha^2+12\eta^2)}{\alpha^2}\Big)\sum_{t=0}^{T-1} \mathbb{E} \Big[\lVert R^{(t)}\rVert^2 \Big]\\
        &\leq \Bigg( \frac{\widetilde L^2(48\eta^4+288\eta^2+24\alpha^2)(8\kappa+1)+32\kappa\widetilde L^2\alpha^2+4\alpha^2 L^2}{\alpha^2\eta^2} \\
        &\quad +\frac{24\widetilde L^2(8\kappa+1)(12\eta^4+\alpha^2+12\eta^2)}{\alpha^2}\Bigg)\sum_{t=0}^{T-1}\mathbb{E} \left[  \lVert R^{(t)}  \rVert^2 \right]+32\kappa T\zeta^2+\frac{4}{\eta}\mathbb{E}\Big[\lVert \widetilde M^{(0)}\rVert^2\Big]\\
        &\quad+\frac{(48\eta^4+288\eta^2+24\alpha^2)(8\kappa+1)+32\kappa\alpha^2}{\eta\alpha^2G}\mathbb{E} \sum_{i\in \mathcal{G}}\left[  \lVert M_i^{(0)}\rVert^2 \right]\\
        &\quad+\Big(\frac{(48\eta^4+336\eta^2+28\alpha^2)(8\kappa+1)}{\alpha^2 }+ 32\kappa +\frac{4}{G}+\frac{8(8\kappa+1)\eta^3}{\alpha}\Big)\eta T\sigma^2.
    \end{split}
\end{equation*}
Subtracting $f(x^*)$  from both sides of inequality \eqref{descent}, taking expectation and defining $\delta_t \eqdef  \nabla f(x^{(t)})-f(x^*)$, we derive
\begin{equation*}
    \begin{split}
        \mathbb{E} \left[ \lVert \nabla f(\hat{x}^{(T)})\rVert^2\right] &\leq \frac{2\delta}{\gamma T}-\frac{A}{ T}\sum_{t=0}^{T-1}\mathbb{E}\left[\lVert x^{(t+1)}-x^{(t)}\rVert^2\right]+32\kappa \zeta^2+\frac{4}{\eta T}\mathbb{E}\Big[\lVert \widetilde M^{(0)}\rVert^2\Big]\\
        &\quad+\frac{(48\eta^4+288\eta^2+24\alpha^2)(8\kappa+1)+32\kappa\alpha^2}{\eta\alpha^2GT}\mathbb{E} \sum_{i\in \mathcal{G}}\left[  \lVert M_i^{(0)}\rVert^2 \right]\\
        &\quad+\Big(\frac{48(\eta^4+7\eta^2)(8\kappa+1)}{\alpha^2 }+ 4(64\kappa+7) +\frac{4}{G}+\frac{8(8\kappa+1)\eta^3}{\alpha}\Big)\eta \sigma^2,
    \end{split}
\end{equation*}
where $\hat{x}^{(T)}$ is sampled uniformly at random from $T$ iterates and
\begin{equation*}
    \begin{split}
        A &= \frac{1}{\gamma^2} \Bigg(\frac{1}{2}-\frac{ \gamma^2\widetilde L^2(48\eta^4+288\eta^2+24\alpha^2)(8\kappa+1)}{\alpha^2\eta^2} -\frac{4\gamma^2(8\kappa\widetilde L^2+ L^2)}{\eta^2}  \\
        &\quad \quad \quad-\frac{24\gamma^2\widetilde L^2(8\kappa+1)(12\eta^4+\alpha^2+12\eta^2)}{\alpha^2}\Bigg)\\
        &= \frac{1}{\gamma^2} \Bigg(\frac{1}{2}-\frac{ 48\gamma^2\widetilde L^2(\eta^2+6)(8\kappa+1)}{\alpha^2}  -\frac{24\gamma^2\widetilde L^2(8\kappa+1)}{\eta^2}-\frac{4\gamma^2(8\kappa\widetilde L^2+ L^2)}{\eta^2}\\
        &\quad \quad\quad-\frac{24\gamma^2\widetilde L^2(8\kappa+1)(12\eta^4+\alpha^2+12\eta^2)}{\alpha^2}\Bigg)\\
        &= \frac{1}{\gamma^2} \Bigg(\frac{1}{2}-\frac{ 48\gamma^2\widetilde L^2(\eta^2+6)(8\kappa+1)}{\alpha^2}  -\frac{4\gamma^2((56\kappa+6)\widetilde L^2+ L^2)}{\eta^2}  \\
        &\quad\quad\quad-\frac{24\gamma^2\widetilde L^2(8\kappa+1)(12\eta^4+\alpha^2+12\eta^2)}{\alpha^2}\Bigg)\\
        &= \frac{1}{\gamma^2} \Big(\frac{1}{2}-\frac{ 24\gamma^2\widetilde L^2(8\kappa+1)(12\eta^4+\alpha^2+14\eta^2+12)}{\alpha^2}  -\frac{4\gamma^2((56\kappa+6)\widetilde L^2+ L^2)}{\eta^2} \Big)\\
        &\overset{(i)}{\geq}\frac{1}{\gamma^2}\Big( \frac{1}{2}-\frac{ 936\gamma^2\widetilde L^2(8\kappa+1)}{\alpha^2}  -\frac{4\gamma^2((56\kappa+6)\widetilde L^2+ L^2)}{\eta^2} \Big)\\
        &\overset{(ii)}{\geq} 0\,,
    \end{split}
\end{equation*}
where $(i)$ and (ii) are due to $\eta\leq 1$ and the assumption on step-size.

Finally, by using the choice of the momentum parameter, we derive
\begin{equation*}
\begin{split}
    &\eta \leq \min \Bigg\{  \Big( \frac{8\alpha^2 \sqrt{(56\kappa+6)\widetilde L^2+L^2} \delta_0}{48(8\kappa+1) \sigma^2 T}\Big)^{\nicefrac{1}{6}}, \Big( \frac{8\alpha^2 \sqrt{(56\kappa+6)\widetilde L^2+L^2} \delta_0}{336(8\kappa+1) \sigma^2 T}\Big)^{\nicefrac{1}{4}},\\
    &\Big( \frac{8\sqrt{(56\kappa+6)\widetilde L^2+L^2}\delta_0}{(64\kappa+7) \sigma^2 T}\Big)^{\nicefrac{1}{2}},
     \Big( \frac{8\sqrt{(56\kappa+6)\widetilde L^2+L^2}\delta_0G}{ 4\sigma^2 T}\Big)^{\nicefrac{1}{2}},\Big( \frac{8\alpha \sqrt{(56\kappa+6)\widetilde L^2+L^2}\delta_0}{8(8\kappa+1) \sigma^2 T}\Big)^{\nicefrac{1}{5}}   \Bigg\},
\end{split}
\end{equation*}
ensures that 
$\frac{48\eta^5 (8\kappa+1)\sigma^2}{\alpha^2}\leq \frac{8\sqrt{(56\kappa+6)\widetilde L^2+L^2} }{\eta T}, \frac{336\eta^3 (8\kappa+1)\sigma^2}{\alpha^2}\leq \frac{8\sqrt{(56\kappa+6)\widetilde L^2+L^2} }{\eta T},
4\eta (64\kappa+7)\sigma^2\leq \frac{8\sqrt{(56\kappa+6)\widetilde L^2+L^2} }{\eta T},\frac{4\eta \sigma^2}{G}\leq \frac{8\sqrt{(56\kappa+6)\widetilde L^2+L^2} }{\eta T},$ and $\frac{8\eta^4 (8\kappa+1)\sigma^2}{\alpha}\leq \frac{8\sqrt{(56\kappa+6)\widetilde L^2+L^2} }{\eta T}$ , we denote $\widehat{L}=\sqrt{(\kappa+1)\widetilde L^2+L^2}$, then we obtain
\begin{equation*}
    \begin{split}
        \mathbb{E}\Big[\lVert \nabla f(\hat{x}^{(T)})\rVert^2\Big]\leq &\Big( \frac{(48(8\kappa+1))^{\nicefrac{1}{5}}\sigma^{\nicefrac{2}{5}}\widehat{L}\delta_0}{ \alpha^{\nicefrac{2}{5}} T}\Big)^{\nicefrac{5}{6}}+\Big( \frac{(336(8\kappa+1))^{1/3}\sigma^{\nicefrac{2}{3}}\widehat{L}\delta_0}{ \alpha^{\nicefrac{2}{3}} T}\Big)^{\nicefrac{3}{4}}+32\kappa \zeta^2+\frac{\Phi_0}{\gamma T}\\
        &+\Big( \frac{4(64\kappa+7)\sigma^2 \widehat{L}\delta_0}{  T}\Big)^{\nicefrac{1}{2}}
        +\Big( \frac{16\sigma^2 \widehat{L}\delta_0}{ G T}\Big)^{\nicefrac{1}{2}}
        +\Big( \frac{(8(8\kappa+1))^{1/4}\sigma^{\nicefrac{1}{2}} \widehat{L} \delta_0}{ \alpha^{\nicefrac{1}{4}}T}\Big)^{\nicefrac{4}{5}}.
    \end{split}
\end{equation*}
\\
This concludes the proof.
\end{proof}

\section{Missing Proofs of Byz-VR-DM21 for General Non-Convex Functions}\label{proof_VR_DM21}
Let us state the following lemma that is used in the analysis of our methods.
\subsection{Supporting Lemmas}

\begin{lemma}[Robust aggregation error]\label{lem:robust_error_vr}
Using Lemma \ref{lem:robust_error} and suppose that Assumption \ref{assump:heterogeneity} holds. Then, for all $t \geq 0$ the iterates generated by \algname{Byz-VR-DM21} in Algorithm \ref{alg:SGD2M} satisfy the following condition: 
\begin{equation*}
    \begin{split}
        & \lVert g^{(t)} - \overline{g}^{(t)} \rVert^2  \\
        &\leq \frac{\kappa}{G}\sum_{i\in \mathcal{G}} \lVert g_i^{(t)}- \overline{g}^{(t)} \rVert^2 \\
        & \leq  \frac{\kappa}{2G^2}\sum_{i,j\in \mathcal{G}} \lVert g_i^{(t)}- g_j^{(t)} \rVert^2 \\
        & \leq \frac{8\kappa(G-1)}{G^2} \sum_{i\in \mathcal{G}} \left( \lVert C_i^{(t)} \rVert^2 + \lVert P_i^{(t)} \rVert^2+ \lVert M_i^{(t)} \rVert^2   \right) + \frac{8\kappa(G-1)}{G} \zeta^2,
            \end{split}
    \end{equation*}
\end{lemma}
where $\overline{g}^{(t)} = G^{-1} \sum_{i\in \mathcal{G}}g_i$, $C_i^{(t)}= g_i^{(t)} - u_i^{(t)}$, $P_i^{(t)}= u_i^{(t)} - v_i^{(t)}$, $M_i^{(t)}= v_i^{(t)} - \nabla f_i(x^{(t)})$.

\begin{lemma}[Accumulated compression error]\label{lem:compression_error_vr} Let Assumption \ref{assump:smoothness} and \ref{assump:bound_variance} be satisfied, and suppose $\mathcal{C}$ is a contractive compressor. For every $i = 1,\ldots,G$, let the sequences $\{v_i^{(t)}\}_{t \geq 0}$, $\{u_i^{(t)}\}_{t \geq 0}$, and $\{g_i^{(t)}\}_{t \geq 0}$ be updated via
\begin{equation*}
\begin{split}
    g_i^{(t)} & = g_i^{(t-1)} + \mathcal{C}(u_i^{(t)}-g_i^{(t-1)}) \\
    u_i^{(t)} & = u_i^{(t-1)} + \eta (v_i^{(t)} - u_i^{(t-1)})\\
    v_i^{(t)} &= (1 - \eta) v_i^{(t-1)} + \eta \nabla f_i(x^{(t)}, \xi_i^{(t)}) \\
    &\quad+  (1 - \eta) (\nabla f_i(x^{(t)}, \xi_i^{(t)})-\nabla f_i(x^{(t-1)}, \xi_i^{(t)})).
\end{split}
\end{equation*}
Then for all $t \geq 0$ the iterates generated by \algname{Byz-VR-DM21} in Algorithm \ref{alg:SGD2M} satisfy 
    \begin{equation}
    \begin{split}
        \sum_{t=0}^{T-1}\mathbb{E}\left[ \lVert C_i^{(t)}\rVert^2 \right]&=\sum_{t=0}^{T-1}\mathbb{E}\left[ \lVert g_i^{(t)} - u_i^{(t)}\rVert^2 \right] \\
       & \leq  \frac{12\eta^4}{\alpha^2} \sum_{t=0}^{T-1} \mathbb{E}\left[\lVert   \nabla f_i(x^{(t)})-v_i^{(t)}\rVert^2 \right]    + \frac{4   \eta^4 T \sigma^2}{\alpha} \\
        &  \quad   + \frac{4\eta^2(\ell_i^2+\frac{3}{\alpha}L_i^2)}{\alpha}  \sum_{t=0}^{T-1} \mathbb{E}\left[\lVert x^{(t+1)}-x^{(t)}\rVert^2 \right]\\
        &\quad+ \frac{12\eta^2}{\alpha^2} \sum_{t=0}^{T-1}\mathbb{E}\left[\lVert  u_i^{(t)}-v_i^{(t)}\rVert^2 \right].
    \end{split}
\end{equation}
\end{lemma}

\begin{proof}
We define 
\begin{equation}
    \begin{split}
        &\mathcal{S}_i^{(t)} \eqdef \nabla f_i(x^{(t)},\xi_i^{(t)})-\nabla f_i(x^{(t)}),\\
        & \mathcal{S}^{(t)} \eqdef \frac{1}{G}\sum_{i\in \mathcal{G}} \mathcal{S}_i^{(t)},\\
        &\mathcal{M}_i^{(t)}\eqdef\nabla f_i(x^{(t)}, \xi_i^{(t)})-\nabla f_i(x^{(t)})+\nabla f_i(x^{(t-1)})-\nabla f_i(x^{(t-1)}, \xi_i^{(t)}), \\
        &\mathcal{M}^{(t)} \eqdef \frac{1}{G}\sum_{i\in \mathcal{G}} \mathcal{M}_i^{(t)}.
    \end{split}
\end{equation} 
Then by Assumptions \ref{assump:bound_variance}, we have
\begin{equation}
\begin{split}
    &\mathbb{E}\Big[\mathcal{S}_i^{(t)}\Big] = \mathbb{E}\Big[\mathcal{M}_i^{(t)}\Big] = \mathbb{E}\Big[\mathcal{S}^{(t)}\Big] = \mathbb{E}\Big[\mathcal{M}^{(t)}\Big] = 0 , \\
    & \mathbb{E}\Big[\lVert\mathcal{S}_i^{(t)}\rVert^2\Big] \leq \sigma^2 \quad, \quad \mathbb{E}\Big[\lVert\mathcal{S}^{(t)}\rVert^2\Big]\leq\frac{\sigma^2}{G}.
\end{split}
\end{equation}

Furthermore, we can derive
\begin{equation}\label{vr_smooth}
    \begin{split}
        \mathbb{E}\Big[\lVert\mathcal{M}^{(t)}\rVert^2\Big]\quad &= \quad \mathbb{E}\Big[\lVert\frac{1}{G}\sum_{i\in\mathcal{G}}\mathcal{M}_i^{(t)}\rVert^2\Big]\\
        & =\quad \frac{1}{G^2} \mathbb{E}\Big[\lVert\sum_{i\in\mathcal{G}}\mathcal{M}_i^{(t)}\rVert^2\Big]\\
        & =\quad \frac{1}{G^2}\sum_{i\in\mathcal{G}} \mathbb{E}\Big[\lVert\mathcal{M}_i^{(t)}\rVert^2\Big] + \frac{1}{G^2}\sum_{i,j\in\mathcal{G},i\neq j}\mathbb{E}\Big[ \langle\mathcal{M}_i^{(t)},\mathcal{M}_j^{(t)}\rangle\Big]\\
        &\overset{(i)}{=}\quad \frac{1}{G^2}\sum_{i\in\mathcal{G}} \mathbb{E}\Big[\lVert\mathcal{M}_i^{(t)}\rVert^2\Big] + \frac{1}{G^2}\sum_{i,j\in\mathcal{G},i\neq j}\Big\langle\mathbb{E}[\mathcal{M}_i^{(t)}],\mathbb{E}[\mathcal{M}_j^{(t)}]\Big\rangle\\
        &=\quad \frac{1}{G^2}\sum_{i\in\mathcal{G}} \mathbb{E}\Big[\lVert\mathcal{M}_i^{(t)}\rVert^2\Big]\\
        &\leq\quad\frac{1}{G^2}\sum_{i\in\mathcal{G}} \mathbb{E}\Big[\lVert \nabla f_i(x^{(t)},\xi_i^{(t)})- \nabla f_i(x^{(t-1)},\xi_i^{(t)})\rVert^2\Big]\\
        &\leq\quad\frac{1}{G^2}\sum_{i\in\mathcal{G}}\ell_i^2 \mathbb{E}\Big[\lVert x^{(t)}- x^{(t-1)}\rVert^2\Big] = \frac{\widetilde\ell^2}{G}\mathbb{E}\Big[\lVert x^{(t)}- x^{(t-1)}\rVert^2\Big].
    \end{split}
\end{equation}
Step $(i)$ holds due to the conditional independence of $\mathcal{M}_i^{(t)}$ and $\mathcal{M}_j^{(t)}$, while the final inequality follows from the smoothness of the stochastic functions, as stated in Assumption \ref{assump:smoothness}. Consequently, we obtain the following result:
\begin{equation}
    \begin{split}
        \mathbb{E}\Big[ \lVert\mathcal{M}_i^{(t)}\rVert^2\Big]\leq \ell_i^2\mathbb{E}\Big[\lVert x^{(t)}-x^{(t-1)}\rVert^2\Big],\quad\quad\mathbb{E}\Big[ \lVert\mathcal{M}^{(t)}\rVert^2\Big]\leq \frac{\widetilde\ell^2}{G}\mathbb{E}\Big[\lVert x^{(t)}-x^{(t-1)}\rVert^2\Big],
    \end{split}
\end{equation}
where the first inequality is obtained by using a similar derivation.

Then by the update rules of $g_i^{(t)}$, $u_i^{(t)}$ and $v_i^{(t)}$, we derive
\begin{equation*}
\begin{split}
     &\mathbb{E}\left[ \lVert g_i^{(t)} - u_i^{(t)}\rVert^2 \right] \\
     &=    \mathbb{E}\left[ \lVert g_i^{(t-1)} - u_i^{(t)} + \mathcal{C}(u_i^{(t)}-g_i^{(t-1)}) \rVert^2 \right]  \\
    & =  \mathbb{E} \left[ \mathbb{E}_{\mathcal{C}} \left[ \lVert u_i^{(t)}- g_i^{(t-1)}  - \mathcal{C}(u_i^{(t)}-g_i^{(t-1)}) \rVert^2 \right] \right]       \\
    & \overset{(i)}{\leq} (1-\alpha)  \mathbb{E}\left[ \lVert u_i^{(t) }-g_i^{(t-1)} \rVert^2 \right] \\
    & = (1-\alpha) \mathbb{E}\left[\lVert u_i^{(t-1)} - g_i^{(t-1)} + \eta (v_i^{(t)} - u_i^{(t-1)})\rVert^2  \right]\\
    & = (1-\alpha) \mathbb{E}\left[\lVert u_i^{(t-1)} - g_i^{(t-1)} + \eta \Big((1-\eta)(v_i^{(t-1)}-\nabla f_i(x^{(t-1)},\xi^{(t)}))+\nabla f_i(x^{(t)},\xi^{(t)}) - u_i^{(t-1)}\Big)\rVert^2  \right]\\
    & =(1-\alpha) \mathbb{E}\Bigg[\lVert  \eta\Big(v_i^{(t-1)}-u_i^{(t-1)} +\eta(\nabla f_i(x^{(t)},\xi^{(t)})-v_i^{(t-1)})+(1-\eta)(\nabla f_i(x^{(t)})-\nabla f_i(x^{(t-1)})) \Big)\\
    &\quad+ \eta(1-\eta)(\nabla f_i(x^{(t)}, \xi_i^{(t)})-\nabla f_i(x^{(t)})+\nabla f_i(x^{(t-1)})-\nabla f_i(x^{(t-1)}, \xi_i^{(t)})) +u_i^{(t-1)} - g_i^{(t-1)}\rVert^2  \Bigg]\\
    & =(1-\alpha) \mathbb{E}\Bigg[\lVert \eta\Big(v_i^{(t-1)}-u_i^{(t-1)} +\eta(\nabla f_i(x^{(t)},\xi^{(t)})-\nabla f_i(x^{(t)}))+\eta(\nabla f_i(x^{(t-1)})-v_i^{(t-1)}) \Big)\\
    &\quad +u_i^{(t-1)} - g_i^{(t-1)} +\eta(\nabla f_i(x^{(t)})-\nabla f_i(x^{(t-1)}))\\
    &\quad + \eta(1-\eta)(\nabla f_i(x^{(t)}, \xi_i^{(t)})-\nabla f_i(x^{(t)})+\nabla f_i(x^{(t-1)})-\nabla f_i(x^{(t-1)}, \xi_i^{(t)}))  \rVert^2  \Bigg]\\
    & = (1-\alpha) \mathbb{E}\Bigg[\mathbb{E}_{\xi_i^{(t)}}\Big[\lVert u_i^{(t-1)} - g_i^{(t-1)} + \eta (v_i^{(t-1)} - u_i^{(t-1)})+\eta^2\mathcal{S}_i^{(t)} +\eta^2(\nabla f_i(x^{(t-1)})-v_i^{(t-1)})\\
    &\quad+\eta(\nabla f_i(x^{(t)})-\nabla f_i(x^{(t-1)}))
    +\eta(1-\eta)\mathcal{M}_i^{(t)}\rVert^2  \Big]\Bigg]\\
    &= (1-\alpha) \mathbb{E}\Bigg[\lVert u_i^{(t-1)} - g_i^{(t-1)} + \eta (v_i^{(t-1)} - u_i^{(t-1)}) +\eta^2(\nabla f_i(x^{(t-1)})-v_i^{(t-1)})\\
    &\quad+\eta(\nabla f_i(x^{(t)})-\nabla f_i(x^{(t-1)}))\rVert^2 \Bigg]+(1-\alpha) \mathbb{E}\Big[\lVert \eta^2\mathcal{S}_i^{(t)}+\eta(1-\eta)\mathcal{M}_i^{(t)}\rVert^2 \Big]\\
    & \leq  (1-\alpha) \mathbb{E}\Bigg[\lVert u_i^{(t-1)} - g_i^{(t-1)} + \eta (v_i^{(t-1)} - u_i^{(t-1)}) +\eta^2(\nabla f_i(x^{(t-1)})-v_i^{(t-1)})\\
    &\quad+\eta(\nabla f_i(x^{(t)})-\nabla f_i(x^{(t-1)}))\rVert^2 \Bigg]\\
    &\quad+2(1-\alpha)\eta^4 \mathbb{E}\Big[\lVert \mathcal{S}_i^{(t)}\rVert^2\Big]+ 2(1-\alpha)\eta^2(1-\eta)^2\mathbb{E}\Big[\lVert \mathcal{M}_i^{(t)} \rVert^2\Big]\\
     & \overset{(ii)}{\leq} (1-\alpha)(1+\rho)  \mathbb{E}\left[\lVert  u_i^{(t-1)} -g_i^{(t-1)} \rVert^2 \right] +  2\eta^4 \sigma^2 +  2\eta^2\ell_i^2  \mathbb{E}\Big[\lVert x^{(t)}-x^{(t-1)}\rVert^2 \Big]\\
      & \quad+ (1-\alpha)(1+\rho^{-1}) \mathbb{E}\bigg[\lVert \eta (v_i^{(t-1)}-u_i^{(t-1)})+\eta^2  (\nabla f_i(x^{(t-1)})-v_i^{(t-1)})\\
      &\quad+\eta(\nabla f_i(x^{(t)})-\nabla f_i(x^{(t-1)}))\rVert^2\bigg]   
      \end{split}
    \end{equation*}
    \begin{equation*}
        \begin{split}
        \mathbb{E}\left[ \lVert g_i^{(t)} - u_i^{(t)}\rVert^2 \right]
      &\overset{(iii)}{\leq} (1-\alpha)(1+\rho)  \mathbb{E}\left[\lVert  u_i^{(t-1)} -g_i^{(t-1)} \rVert^2 \right] +  2\eta^4 \sigma^2 +  2\eta^2\ell_i^2  \mathbb{E}\Big[\lVert x^{(t)}-x^{(t-1)}\rVert^2 \Big]\\
      &\quad+3\eta^2(1-\alpha)(1+\rho^{-1}) \mathbb{E}\left[\lVert  v_i^{(t-1)}-u_i^{(t-1)}\rVert^2 \right]\\
      & \quad +3\eta^4(1-\alpha)(1+\rho^{-1})\mathbb{E}\left[\lVert   \nabla f_i(x^{(t-1)})-v_i^{(t-1)}\rVert^2 \right]\\
      &\quad+3\eta^2(1-\alpha)(1+\rho^{-1})\mathbb{E}\left[\lVert \nabla f_i(x^{(t)})-\nabla f_i(x^{(t-1)})\rVert^2\right],\\
    \end{split}
    \end{equation*}
    where $(i)$ refers to the contractive property, as defined in Definition \ref{def:contractive}, $(ii)$ hold by inequality~\eqref{eq_young} for any $\rho > 0$, and $(iii)$ leverages the smoothness property of $f_i(\cdot)$. Setting $\rho=\alpha/2$, we obtain
\begin{equation*}
(1-\alpha)(1+\frac{\alpha}{2}) = 1-\frac{\alpha}{2}- \frac{\alpha^2}{2}\leq 1-\frac{\alpha}{2},
\end{equation*}
and 
\begin{equation*}
(1-\alpha)(1+\frac{2}{\alpha}) = \frac{2}{\alpha}-\alpha -1 \leq \frac{2}{\alpha}.
\end{equation*}
There holds
\begin{equation*}
\begin{split}
     &\mathbb{E}\left[ \lVert g_i^{(t)} - u_i^{(t)}\rVert^2 \right] \\
     &=    \mathbb{E}\left[ \lVert g_i^{(t-1)} - u_i^{(t)} + \mathcal{C}(u_i^{(t)}-g_i^{(t-1)}) \rVert^2 \right]  \\
     & \leq \left(1-\frac{\alpha}{2}\right)  \mathbb{E}\left[\lVert  u_i^{(t-1)} -g_i^{(t-1)} \rVert^2 \right] +  \frac{6\eta^4}{\alpha}\mathbb{E}\left[\lVert   \nabla f_i(x^{(t-1)})-v_i^{(t-1)}\rVert^2 \right]   +2  \eta^4 \sigma^2\\
            & \quad+ \eta^2(2\ell_i^2+\frac{6}{\alpha}L_i^2) \mathbb{E}\left[\lVert x^{(t)}-x^{(t-1)}\rVert^2 \right]     + \frac{6\eta^2}{\alpha} \mathbb{E}\left[\lVert  u_i^{(t-1)}-v_i^{(t-1)}\rVert^2 \right].
\end{split}
\end{equation*}
Summing up the above inequality from $t=0$ to $t=T-1$ leads to
\begin{equation*}
    \begin{split}
        \sum_{t=0}^{T-1}\mathbb{E}\left[ \lVert g_i^{(t)} - u_i^{(t)}\rVert^2 \right] & \leq  \frac{12\eta^4}{\alpha^2} \sum_{t=0}^{T-1} \mathbb{E}\left[\lVert   \nabla f_i(x^{(t)})-v_i^{(t)}\rVert^2 \right]    + \frac{4   \eta^4 T \sigma^2}{\alpha} \\
        &  \quad   + \frac{4\eta^2(\ell_i^2+\frac{3}{\alpha}L_i^2)}{\alpha}  \sum_{t=0}^{T-1} \mathbb{E}\left[\lVert x^{(t+1)}-x^{(t)}\rVert^2 \right]\\
        &\quad+ \frac{12\eta^2}{\alpha^2} \sum_{t=0}^{T-1}\mathbb{E}\left[\lVert  u_i^{(t)}-v_i^{(t)}\rVert^2 \right].
    \end{split}
\end{equation*}
\end{proof}
\begin{lemma}[Accumulated second momentum deviation]\label{lem:second_momentum_vr}Let Assumption \ref{assump:smoothness} and \ref{assump:bound_variance} be satisfied, and suppose $0 < \eta \leq 1$. For every $i = 1,\ldots,G$, let the sequences $\{v_i^{(t)}\}_{t \geq 0}$ and $\{u_i^{(t)}\}_{t \geq 0}$ be updated via
\begin{equation*}
\begin{split}
    u_i^{(t)} & = u_i^{(t-1)} + \eta (v_i^{(t)} - u_i^{(t-1)}),\\
    v_i^{(t)} &= (1 - \eta) v_i^{(t-1)} + \eta \nabla f_i(x^{(t)}, \xi_i^{(t)}) \\
    &\quad+  (1 - \eta) (\nabla f_i(x^{(t)}, \xi_i^{(t)})-\nabla f_i(x^{(t-1)}, \xi_i^{(t)})).
\end{split}
\end{equation*}
Then for all $t \geq 0$ the iterates generated by \algname{Byz-VR-DM21} in Algorithm \ref{alg:SGD2M} satisfy 
\begin{equation}
\begin{split}
     \sum_{t=0}^{T-1}\mathbb{E}\left[ \lVert P_i^{(t)}\rVert^2 \right]&=\sum_{t=0}^{T-1}\mathbb{E}\left[ \lVert u_i^{(t)} - v_i^{(t)}\rVert^2 \right]\\
     &\leq 6\sum_{t=0}^{T-1}\mathbb{E}\Big[\lVert M_i^{(t)}\rVert^2\Big]+\frac{2(\ell_i^2+ \frac{2}{\eta} L_i^2)}{\eta}\sum_{t=0}^{T-1}\mathbb{E}\Big[\lVert x^{(t+1)}-x^{(t)}\rVert^2\Big]+2\eta T\sigma^2.
\end{split}
\end{equation}
\end{lemma}
\begin{proof}

By the update rule of $u_i^{(t)}$ and $v_i^{(t)}$ , we have
\begin{equation*}
    \begin{split}
        &\mathbb{E}\left[ \lVert u_i^{(t)} - v_i^{(t)}\rVert^2 \right] \\
        & =\mathbb{E}\left[ \lVert u_i^{(t-1)} - v_i^{(t)} + \eta (v_i^{(t)} - u_i^{(t-1)})\rVert^2 \right]\\
        & = (1-\eta)^2 \mathbb{E}\left[ \lVert u_i^{(t-1)} - v_i^{(t)}\rVert^2 \right]\\
         & = (1-\eta)^2 \mathbb{E}\Big[ \lVert (1-\eta) v_i^{(t-1)} +  \eta \nabla f_i(x^{(t)},\xi_i^{(t)}) - u_i^{(t-1)}+ (1-\eta)(\nabla f_i(x^{(t)},\xi_i^{(t)})- \nabla f_i(x^{(t-1)},\xi_i^{(t)}))\rVert^2 \Big]\\
         & = (1-\eta)^2 \mathbb{E}\Big[ \lVert  (v_i^{(t-1)}-u_i^{(t-1)}) +  \eta (\nabla f_i(x^{(t)},\xi_i^{(t)})-v_i^{(t-1)}) +(1-\eta)(\nabla f_i(x^{(t)})-\nabla f_i(x^{(t-1)}))\\
         & \quad+ (1-\eta)(\nabla f_i(x^{(t)},\xi_i^{(t)})-\nabla f_i(x^{(t)})+\nabla f_i(x^{(t-1)})- \nabla f_i(x^{(t-1)},\xi_i^{(t)}))\rVert^2 \Big]\\
         & = (1-\eta)^2 \mathbb{E}\Big[ \lVert (v_i^{(t-1)}-u_i^{(t-1)})+\eta(\nabla f_i(x^{(t)},\xi_i^{(t)})-\nabla f_i(x^{(t)}))+\eta(\nabla f_i(x^{(t-1)})-v_i^{(t-1)})\\
         &\quad +\nabla f_i(x^{(t)})-\nabla f_i(x^{(t-1)})+ (1-\eta) \mathcal{M}_i^{(t)}\rVert^2 \Big]\\
         & = (1-\eta)^2 \mathbb{E}\Big[\mathbb{E}_{\xi_i^{(t)}} \Big[\lVert (v_i^{(t-1)}-u_i^{(t-1)})+\eta(\nabla f_i(x^{(t)},\xi_i^{(t)})-\nabla f_i(x^{(t)}))+\eta(\nabla f_i(x^{(t-1)})-v_i^{(t-1)})\\
         &\quad +\nabla f_i(x^{(t)})-\nabla f_i(x^{(t-1)})+ (1-\eta) \mathcal{M}_i^{(t)}\rVert^2 \Big]\Big]\\
         & = (1-\eta)^2 \mathbb{E}\Big[ \lVert (v_i^{(t-1)}-u_i^{(t-1)})+\eta(\nabla f_i(x^{(t-1)})-v_i^{(t-1)})+(\nabla f_i(x^{(t)})-\nabla f_i(x^{(t-1)}) )\rVert^2 \Big]\\
         &\quad+(1-\eta)^2 \mathbb{E}\Big[ \lVert \eta \mathcal{S}_i^t +(1-\eta)\mathcal{M}_i^t\rVert^2 \Big]\\
          & \overset{(i)}{\leq} (1-\eta)^2(1+\rho)\mathbb{E}\Big[\lVert u_i^{(t-1)}-v_i^{(t-1)} \rVert^2\Big] + (1-\eta)^2(1+\rho^{-1})\mathbb{E}\Bigg[\lVert \eta(v_i^{(t-1)}-\nabla f_i(x^{(t-1)})) \\
          &\quad+\nabla f_i(x^{(t)})-\nabla f_i(x^{(t-1)}) \rVert^2\Bigg]+2(1-\eta)^2\eta^2\mathbb{E}\Big[ \lVert\mathcal{S}_i^t \rVert^2 \Big]+2(1-\eta)^4\mathbb{E}\Big[ \lVert \mathcal{M}_i^t \rVert^2 \Big]\\
          & \leq (1-\eta)^2(1+\rho)\mathbb{E}\Big[\lVert u_i^{(t-1)}-v_i^{(t-1)} \rVert^2\Big] + 2\eta^2(1+\rho^{-1})\mathbb{E}\Big[\lVert v_i^{(t-1)}-\nabla f_i(x^{(t-1)}) \rVert^2\Big]+2\eta^2 \sigma^2\\
          &\quad+2(1-\eta)^2(1+\rho^{-1})\mathbb{E}\Big[ \lVert \mathcal\nabla f_i(x^{(t)})-\nabla f_i(x^{(t-1)}) \rVert^2 \Big]+2\ell_i^2\mathbb{E}\Big[ \lVert x^{(t)}-x^{(t-1)} \rVert^2 \Big],
          \end{split}
          \end{equation*}
        where $(i)$ uses inequality \eqref{vr_smooth}, smoothness of $f_i(\cdot)$ and inequality~\eqref{eq_young} for any $\rho > 0$. Setting $\rho = \eta/2$, and because of $\eta \in \left(0,1\right]$, we obtain
        \begin{equation*}
        (1-\eta)^2(1+\frac{\eta}{2}) = 1-\frac{3\eta}{2}+ \frac{\eta^3}{2}\leq 1-\eta, \quad(1-\eta)^2(1+\frac{2}{\eta})\leq 1+\frac{2}{\eta}-\eta-2\leq \frac{2}{\eta} 
        \end{equation*}
        and 
        \begin{equation*}
        \eta^2(1+\frac{2}{\eta}) = \eta^2+2\eta \leq 3\eta.
        \end{equation*}
        There holds 
        \begin{equation*}
        \begin{split}
        &\mathbb{E}\left[ \lVert u_i^{(t)} - v_i^{(t)}\rVert^2 \right] \\
          & \leq (1-\eta)^2(1+\frac{\eta}{2})\mathbb{E}\Big[\lVert u_i^{(t-1)}-v_i^{(t-1)} \rVert^2\Big] + 2\eta^2(1+\frac{2}{\eta})\mathbb{E}\Big[\lVert v_i^{(t-1)}-\nabla f_i(x^{(t-1)}) \rVert^2\Big]\\
          &\quad + 2(1-\eta)^2(1+\frac{2}{\eta})\mathbb{E}\Big[ \lVert \mathcal\nabla f_i(x^{(t)})-\nabla f_i(x^{(t-1)}) \rVert^2 \Big]+2\eta^2\sigma^2+ 2\ell_i^2 \mathbb{E}\Big[\lVert x^{(t)}-x^{(t-1)}\rVert^2\Big]\\
          & \leq (1-\eta)\mathbb{E}\Big[\lVert u_i^{(t-1)}-v_i^{(t-1)} \rVert^2\Big]+6\eta\mathbb{E}\Big[\lVert v_i^{(t-1)}-\nabla f_i(x^{(t-1)}) \rVert^2\Big]\\
          &\quad+ 2(\ell_i^2+\frac{2}{\eta}L_i^2) \mathbb{E}\Big[\lVert x^{(t)}-x^{(t-1)}\rVert^2\Big]+2\eta^2\sigma^2.\\
\end{split}
\end{equation*}
Summing up the above inequality from $t = 0$ to $t = T-1$ yields
\begin{equation*}
    \begin{split}
        \sum_{t=0}^{T-1}\mathbb{E}\left[ \lVert P_i^{(t)}\rVert^2 \right] & \leq 6\sum_{t=0}^{T-1}\mathbb{E}\Big[\lVert v_i^{(t)}-\nabla f_i(x^{(t)}\rVert^2\Big]+\frac{2(\ell_i^2+ \frac{2}{\eta}L_i^2)}{\eta}\sum_{t=0}^{T-1}\mathbb{E}\Big[\lVert x^{(t+1)}-x^{(t)}\rVert^2\Big]\\
        &\quad+ \frac{1}{\eta }\sum_{t=0}^{T-1}\mathbb{E}\left[ \lVert P_i^{(0)}\rVert^2 \right]+2\eta T\sigma^2\\
        &\leq 6\sum_{t=0}^{T-1}\mathbb{E}\Big[\lVert M_i^{(t)}\rVert^2\Big]+\frac{2(\ell_i^2+ \frac{2}{\eta}) L_i^2}{\eta}\sum_{t=0}^{T-1}\mathbb{E}\Big[\lVert x^{(t+1)}-x^{(t)}\rVert^2\Big]+2\eta T\sigma^2,
    \end{split}
\end{equation*}

where $ P_i^{(0)}=u_i^{(0)}-v_i^{(0)}$, because of $u_i^{(0)}=v_i^{(0)}$, $P_i^{(0)}=0$.
\end{proof}

\begin{lemma}[Accumulated momentum deviation]\label{lem:momentum_vr} Let Assumption \ref{assump:smoothness} and \ref{assump:bound_variance} be satisfied, and suppose $0 < \eta \leq 1$. For every $i = 1,\ldots,G$, let the sequences $\{v_i^{(t)}\}_{t \geq 0}$ be updated via

\begin{equation*}
        v_i^{(t)}=(1-\eta) v_i^{(t-1)} +  \eta \nabla f_i(x^{(t)},\xi_i^{(t)}) + (1-\eta)(\nabla f_i(x^{(t)},\xi_i^{(t)})- \nabla f_i(x^{(t-1)},\xi_i^{(t)})).
\end{equation*}
Then for all $t \geq 0$ the iterates generated by \algname{Byz-VR-DM21} in Algorithm \ref{alg:SGD2M} satisfy 
\begin{equation}
\begin{split}
      \frac{1}{G} \sum_{t=0}^{T-1} \sum_{i\in\mathcal{G}} \mathbb{E}\Big[\lVert M_i^{(t)} \rVert^2\Big] &=\frac{1}{G} \sum_{t=0}^{T-1} \sum_{i\in\mathcal{G}} \mathbb{E}\Big[\lVert v_i^{(t)} - \nabla f_i(x^{(t)})  \rVert^2\Big] \\
     &\leq \frac{2\widetilde\ell^2}{\eta}\sum_{t=0}^{T-1}\mathbb{E} \left[  \lVert x^{(t+1)}-x^{(t)}  \rVert^2 \right]+ 2\eta T \sigma^2 + \frac{1}{\eta G} \sum_{i\in\mathcal{G}} \mathbb{E} \left[  \lVert v_i^{(0)}-\nabla f_i(x^{(0)})  \rVert^2 \right],        
\end{split}
\end{equation}
and
\begin{equation}
\begin{split}
     \sum_{t=0}^{T-1}\mathbb{E}\Big[\lVert \widetilde M_i^{(t)} \rVert^2\Big]&= \sum_{t=0}^{T-1}\mathbb{E}\Big[\lVert \overline{v}^{(t)} - \nabla f_{\mathcal{G}}(x^{(t)}) \rVert^2\Big] \\
    &\leq \frac{2\widetilde\ell^2 }{\eta}\sum_{t=0}^{T-1}\mathbb{E} \left[  \lVert x^{(t+1)}-x^{(t)}  \rVert^2 \right]+ \frac{2\eta T \sigma^2}{G} + \frac{1}{\eta } \mathbb{E} \left[  \lVert \overline{v}^{(0)}-\nabla f_{\mathcal{G}}(x^{(0)})  \rVert^2 \right].      
\end{split}
\end{equation} 
\end{lemma}

\begin{proof}
By the update rule of $v_i^{(t)}$ , and consider
\begin{equation*}
\begin{split}
    & \lVert v_i^{(t)} - \nabla f_i(x^{(t)}) \rVert^2 \\
    & = \lVert  (1-\eta)(v_i^{(t-1)}-\nabla f_i(x^{(t)})) + \eta (\nabla f_i(x^{(t)}_i, \xi_i^{(t)}) - \nabla f_i(x^{(t)})) \\
    &\quad + (1-\eta)(\nabla f_i(x^{(t)},\xi_i^{(t)})- \nabla f_i(x^{(t-1)},\xi_i^{(t)}))\rVert^2\\
    &= \lVert (1-\eta)(v_i^{(t-1)}-\nabla f_i(x^{(t)})) + \eta (\nabla f_i(x^{(t-1)}_i, \xi_i^{(t)}) - \nabla f_i(x^{(t)}))\\
    & \quad + \nabla f_i(x^{(t)},\xi_i^{(t)})- \nabla f_i(x^{(t-1)},\xi_i^{(t)})\rVert^2\\
    & = \lVert (1-\eta)(v_i^{(t-1)}-\nabla f_i(x^{(t-1)}))+\eta(\nabla f_i(x^{(t)},\xi^{(t)})-\nabla f_i(x^{(t)}))\\
    &\quad + (1-\eta)\Big(\nabla f_i(x^{(t-1)})-\nabla f_i(x^{(t-1)},\xi^{(t)})+\nabla f_i(x^{(t)},\xi^{(t)})-\nabla f_i(x^{(t)})\Big)\rVert^2  \\
    &=\lVert (1-\eta)(v_i^{(t-1)}-\nabla f(x^{(t-1)})+\eta \mathcal{S}_i^{(t)}+(1-\eta)\mathcal{M}_i^{(t)}\rVert^2.
\end{split}
\end{equation*}
Taking the expectation on both sides and using the law of total expectation, we obtain
\begin{equation*}
\begin{split}
    & \mathbb{E}\Big[\lVert v_i^{(t)} - \nabla f_i(x^{(t)}) \rVert^2\Big] \\
    & =  \mathbb{E} \Big[ \mathbb{E}_{\xi_i^{(t)}} \Big[\lVert  (1-\eta)(v_i^{(t-1)}-\nabla f(x^{(t-1)})+\eta \mathcal{S}_i^{(t)}+(1-\eta)\mathcal{M}_i^{(t)}  \rVert^2 \Big] \Big].
\end{split}
\end{equation*}
Then, there holds
\begin{equation*}
    \begin{split}
    & \mathbb{E}\Big[\lVert v_i^{(t)} - \nabla f_i(x^{(t)}) \rVert^2\Big] \\
    & = (1-\eta)^2\mathbb{E}\Big[\lVert v_i^{(t-1)} - \nabla f_i(x^{(t-1)}) \rVert^2\Big]+\mathbb{E}\Big[\lVert \eta \mathcal{S}_i^{(t)}+(1-\eta)\mathcal{M}_i^{(t)}\rVert^2\Big]\\
    & \leq  (1-\eta)\mathbb{E} \left[  \lVert v_i^{(t-1)}-\nabla f_i(x^{(t-1)})  \rVert^2 \right]   + 2\eta^2 \mathbb{E}   \left[\lVert  \mathcal{S}_i^{(t)} \rVert^2 \right] +2\mathbb{E} \left[ \lVert \mathcal{M}_i^{(t)}\rVert^2\right] \\
    &\leq (1-\eta)\mathbb{E} \left[  \lVert v_i^{(t-1)}-\nabla f_i(x^{(t-1)})  \rVert^2 \right] +2\eta^2\sigma^2+2\ell_i^2\mathbb{E}\left[  \lVert x^{(t)}-x^{(t-1)}\rVert^2 \right].
\end{split}
\end{equation*}

Summing up the above inequality over all $i\in \mathcal{G}$ and from $t=0$ to $t=T-1$ yields
\begin{equation*}
\begin{split}
     & \frac{1}{G} \sum_{t=0}^{T-1} \sum_{i\in\mathcal{G}} \mathbb{E}\Big[\lVert M_i^{(t)} \rVert^2\Big]  \\
     &\leq \frac{2\widetilde\ell^2}{\eta}\sum_{t=0}^{T-1}\mathbb{E} \left[  \lVert x^{(t)}-x^{(t-1)}  \rVert^2 \right]+ 2\eta T \sigma^2 + \frac{1}{\eta G} \sum_{i\in\mathcal{G}} \mathbb{E} \left[  \lVert v_i^{(0)}-\nabla f_i(x^{(0)})  \rVert^2 \right].     
\end{split}
\end{equation*}

Using the same arguments, we obtain
\begin{equation*}
    \begin{split}
    & \mathbb{E}\Big[\lVert \overline{v}^{(t)} - \nabla f_{\mathcal{G}}(x^{(t)}) \rVert^2\Big] \\
    & \quad\leq  (1-\eta)\mathbb{E} \left[  \lVert \overline{v}^{(t-1)}-\nabla f_{\mathcal{G}}(x^{(t-1)})  \rVert^2 \right]   + \widetilde\ell^2 \mathbb{E} \left[  \lVert x^{(t)}-x^{(t-1)}  \rVert^2 \right] + \frac{2\eta^2 \sigma^2}{G}
\end{split}
\end{equation*}
and
\begin{equation*}
\begin{split}
    & \sum_{t=0}^{T-1}\mathbb{E}\Big[\lVert \overline{v}^{(t)} - \nabla f_{\mathcal{G}}(x^{(t)}) \rVert^2\Big] \\
    &\leq \frac{2\widetilde\ell^2 }{\eta }\sum_{t=0}^{T-1}\mathbb{E} \left[  \lVert x^{(t)}-x^{(t-1)}  \rVert^2 \right]+ \frac{2\eta T \sigma^2}{G} + \frac{1}{\eta } \mathbb{E} \left[  \lVert \overline{v}^{(0)}-\nabla f_{\mathcal{G}}(x^{(0)})  \rVert^2 \right].      
\end{split}
\end{equation*}
\end{proof}

\subsection{Proof of Theorem \ref{thm:rate_vr}}
\begin{proof} By Lemma \ref{lem:descent}, there holds, for any $\gamma \leq 1/(2L)$,
\begin{equation}\label{descent_vr}
    f(x^{(t+1)}) \leq f(x^{(t)}) - \frac{\gamma}{2} \lVert  \nabla f(x^{(t)}) \rVert^2 - \frac{1}{4\gamma} \lVert 
 x^{(t+1)} - x^{(t)}\rVert^2 + \frac{\gamma}{2}\lVert g^{(t)} - \nabla f(x^{(t)}) \rVert^2.
\end{equation}
    Summing the above from $t = 0$ to $t = T - 1$ and taking expectation, we have 
    \begin{equation*}
    \begin{split}
        \frac{1}{T} \sum_{t=0}^{T-1} \mathbb{E}\Big[\lVert \nabla f(x^{(t)}) \rVert^2\Big] \leq \frac{2\delta}{\gamma T}-\frac{1}{2\gamma^2 T}\sum_{t=0}^{T-1}\mathbb{E}\left[\lVert x^{(t+1)}-x^{(t)}\rVert^2\right]+\frac{1}{T}\sum_{t=0}^{T-1}\mathbb{E}\left[\lVert g^{(t)}-\nabla f (x^{(t)})\rVert^2\right],
    \end{split}
\end{equation*}
where $\delta_t = \mathbb{E}\Big[ \nabla f(x^{(t)})-f(x^*)\Big]$. We note that
\begin{equation*}
    \begin{split}
        \lVert g^{(t)} - \nabla f(x^{(t)})\rVert^2 &\leq 4\lVert g^{(t)}- \overline{g}^{(t)} \rVert^2 + 4\lVert \overline g^{(t)}- \overline{u}^{(t)} \rVert^2 + 4\lVert \overline u^{(t)}- \overline{v}^{(t)} \rVert^2 + 4\lVert \overline v^{(t)}- \nabla f(x^{(t)}) \rVert^2 \\
        & \leq 4\lVert g^{(t)}- \overline{g}^{(t)} \rVert^2 + \frac{4}{G}\sum_{i\in \mathcal{G}} \lVert g_i^{(t)} - u_i^{(t)}\rVert^2\\
        &\quad+ \frac{4}{G}\sum_{i\in \mathcal{G}} \lVert u_i^{(t)} - v_i^{(t)}\rVert^2 +4 \lVert \overline{v}^{(t)} - \nabla f(x^{(t)})\rVert^2,
    \end{split}
\end{equation*}
where $\overline{g} = G^{-1} \sum_{i\in\mathcal{G}} g_i$, $\overline{u} = G^{-1} \sum_{i\in\mathcal{G}} u_i$ and $\overline{v} = G^{-1} \sum_{i\in\mathcal{G}} v_i$.

Next, we apply the technical lemmas from the previous section to derive a bound on the deviation between $g^{(t)}$ and $\nabla f(x^{(t)})$. First, we invoke Lemma \ref{lem:robust_error} to obtain the following result
\begin{equation*}
    \begin{split}
    &\lVert g^{(t)} - \nabla f(x^{(t)}) \rVert^2\\
    &\leq \frac{32\kappa(G-1)}{G^2} \sum_{i\in \mathcal{G}} \left( \lVert C_i^{(t)} \rVert^2 + \lVert P_i^{(t)} \rVert^2+ \lVert M_i^{(t)} \rVert^2   \right) + \frac{32\kappa(G-1)}{G} \zeta^2\\
    &\quad + \frac{4}{G}\sum_{i\in \mathcal{G}} \lVert C_i^{(t)}\rVert^2+ \frac{4}{G}\sum_{i\in \mathcal{G}} \lVert P_i^{(t)}\rVert^2 +4 \lVert \widetilde M_i^{(t)}\rVert^2\\
    &\leq \frac{4(8\kappa+1)}{G} \sum_{i\in \mathcal{G}} \lVert C_i^{(t)}  \rVert^2 + \frac{4(8\kappa+1)}{G} \sum_{i\in \mathcal{G}} \lVert P_i^{(t)}  \rVert^2 + \frac{32\kappa}{G} \sum_{i\in \mathcal{G}} \lVert M_i^{(t)}  \rVert^2\\
    &\quad+ 4  \lVert \widetilde M_i^{(t)}  \rVert^2 +32\kappa \zeta^2,
    \end{split}
\end{equation*}
where $C_i^{(t)} = g_i^{(t)}-u_i^{(t)}$, $P_i^{(t)} = u_i^{(t)}-v_i^{(t)}$, $ \widetilde M_i^{(t)} = \overline{v}_i^{(t)}- \nabla f_i(x^{(t)})$. By summing up the above inequality from $t = 0$ to $t= T-1$, we obtain
\begin{equation*}
    \begin{split}
        &\sum_{t=0}^{T-1} \lVert g^{(t)} - \nabla f(x^{(t)}) \rVert^2\\
        & \leq \frac{4(8\kappa+1)}{G} \sum_{t=0}^{T-1}\sum_{i\in \mathcal{G}} \lVert C_i^{(t)}  \rVert^2 + \frac{4(8\kappa+1)}{G} \sum_{t=0}^{T-1}\sum_{i\in \mathcal{G}} \lVert P_i^{(t)}  \rVert^2 + \frac{32\kappa}{G} \sum_{t=0}^{T-1}\sum_{i\in \mathcal{G}} \lVert M_i^{(t)}  \rVert^2\\
    &\quad+ 4 \sum_{t=0}^{T-1} \lVert \widetilde M_i^{(t)}  \rVert^2 +32\kappa T\zeta^2.
    \end{split}
\end{equation*}
Next, by taking expectation and using \ref{lem:compression_error_vr} and define $R^{(t)}\eqdef \mathbb{E}\Big[\lVert x^{(t+1)}-x^{(t)}\rVert^2\Big]$, we have
\begin{equation*}
    \begin{split}
        &\sum_{t=0}^{T-1} \mathbb{E}\Big[\lVert g^{(t)}-\nabla f(x^{(t)})\rVert^2\Big]\\
        &\leq \frac{4(8\kappa+1)}{G} \sum_{i\in \mathcal{G}} \Big(\frac{12\eta^4}{\alpha^2} \sum_{t=0}^{T-1} \mathbb{E}\left[\lVert    M_i^{(t)}\rVert^2 \right]    +\frac{4\eta^2(\ell_i^2+\frac{3}{\alpha}L_i^2)}{\alpha}  \sum_{t=0}^{T-1} \mathbb{E}\left[\lVert R^{(t)}\rVert^2 \right]\\
        &\quad+ \frac{4\eta^4 T \sigma^2}{\alpha}+ \frac{12\eta^2}{\alpha^2} \sum_{t=0}^{T-1}\mathbb{E}\left[\lVert  P_i^{(t)}\rVert^2 \right]\Big)+   \frac{4(8\kappa+1)}{G} \sum_{t=0}^{T-1}\sum_{i\in \mathcal{G}} \mathbb{E}\Big[\lVert P_i^{(t)}  \rVert^2 \Big]\\
        &\quad+ \frac{32\kappa}{G} \sum_{t=0}^{T-1}\sum_{i\in \mathcal{G}} \mathbb{E}\Big[\lVert M_i^{(t)}  \rVert^2\Big]+ 4 \sum_{t=0}^{T-1} \Big[\lVert \widetilde{M}_i^{(t)}  \rVert^2\Big] +32\kappa T\zeta^2 \\
        & \leq (\frac{48\eta^4(8\kappa+1)}{\alpha^2 G}+\frac{32\kappa}{G}) \sum_{t=0}^{T-1}\sum_{i\in \mathcal{G}} \mathbb{E} \Big[\lVert M_i^{(t)}\rVert^2 \Big] + \frac{16(8\kappa+1)\eta^4 T \sigma^2}{\alpha}\\
        &\quad +\frac{16\eta^2(8\kappa+1)(\widetilde\ell^2+\frac{3}{\alpha}\widetilde L^2)}{\alpha}\sum_{t=0}^{T-1} \mathbb{E} \Big[\lVert R^{(t)}\rVert^2 \Big]+ (\frac{\alpha^2+12\eta^2}{\alpha^2})(\frac{4(8\kappa+1)}{G}) \sum_{t=0}^{T-1}\sum_{i\in \mathcal{G}} \mathbb{E}\Big[\lVert P_i^{(t)}  \rVert^2\Big]\\
        &\quad+ 4 \sum_{t=0}^{T-1} \mathbb{E}\Big[\lVert \widetilde{M}_i^{(t)}  \rVert^2 \Big]+32\kappa T\zeta^2. 
    \end{split}
\end{equation*}
Then, by using \ref{lem:second_momentum_vr}, we get
\begin{equation*}
    \begin{split}
        &\sum_{t=0}^{T-1} \mathbb{E}\Big[\lVert g^{(t)}-\nabla f(x^{(t)})\rVert^2\Big]\\
        &\leq \Big(\frac{48\eta^4(8\kappa+1)}{\alpha^2 G}+\frac{32\kappa}{G}\Big)\sum_{t=0}^{T-1}\sum_{i\in \mathcal{G}} \mathbb{E}\left[\lVert M_i^{(t)}  \rVert^2\right]+ 4 \sum_{t=0}^{T-1} \mathbb{E}\left[\lVert \widetilde{M}_i^{(t)}  \rVert^2\right]+32\kappa T\zeta^2\\
        &\quad +\frac{16\eta^2(8\kappa+1)(\widetilde\ell^2+\frac{3}{\alpha}\widetilde L^2)}{\alpha} \mathbb{E} \Big[\lVert R^{(t)}\rVert^2 \Big]+\frac{16(8\kappa+1)\eta^4 T \sigma^2}{\alpha}\\
        &\quad+(\frac{\alpha^2+12\eta^2}{\alpha^2})(\frac{4(8\kappa+1)}{G})\sum_{i\in \mathcal{G}}\Big(6\sum_{t=0}^{T-1}\mathbb{E}\Big[\lVert M_i^{(t)}\rVert^2\Big]+\frac{2(\ell_i^2+ \frac{2}{\eta} L_i^2)}{\eta}\sum_{t=0}^{T-1}\mathbb{E}\Big[\lVert R^{(t)}\rVert^2\Big]+2\eta T\sigma^2\Big)\\
        &\leq \Big(\frac{48\eta^4(8\kappa+1)+24(8\kappa+1)(\alpha^2+12\eta^2)+32\kappa\alpha^2}{\alpha^2 G}\Big) \sum_{t=0}^{T-1}\sum_{i\in \mathcal{G}} \mathbb{E} \Big[\lVert M_i^{(t)}\rVert^2 \Big] + \frac{16(8\kappa+1)\eta^4 T \sigma^2}{\alpha}\\
        &\quad+4 \sum_{t=0}^{T-1} \mathbb{E}\left[\lVert \widetilde{M}_i^{(t)}  \rVert^2\right]+(\frac{8(\alpha^2+12\eta^2)(8\kappa+1)}{\alpha^2})\eta T\sigma^2+32\kappa T\zeta^2 \\
        &\quad +\Big(\frac{16\eta^2(8\kappa+1)(\widetilde\ell^2+\frac{3}{\alpha}\widetilde L^2)}{\alpha}+\frac{8(8\kappa+1)(\alpha^2+12\eta^2)(\widetilde\ell^2+\frac{2}{\eta}\widetilde L^2)}{\alpha^2\eta}\Big)\sum_{t=0}^{T-1} \mathbb{E} \Big[\lVert R^{(t)}\rVert^2 \Big].
    \end{split}
\end{equation*}

Furthermore, by using \ref{lem:momentum_vr}, we get
\begin{equation*}
    \begin{split}
        &\sum_{t=0}^{T-1} \mathbb{E}\Big[\lVert g^{(t)}-\nabla f(x^{(t)})\rVert^2\Big]\\
        &\leq \Big(\frac{48\eta^4(8\kappa+1)+24(8\kappa+1)(\alpha^2+12\eta^2)+32\kappa\alpha^2}{\alpha^2 } \Big)\Big(\frac{2\widetilde\ell^2}{\eta}\sum_{t=0}^{T-1}\mathbb{E} \left[  \lVert R^{(t)}  \rVert^2 \right]+2 \eta T \sigma^2 \\
        &\quad+ \frac{1}{\eta G} \sum_{i\in\mathcal{G}} \mathbb{E} \left[  \lVert M_i^{(0)}  \rVert^2 \right]\Big)+ \frac{16(8\kappa+1)\eta^4 T \sigma^2}{\alpha}+\frac{8\widetilde\ell^2 }{\eta }\sum_{t=0}^{T-1}\mathbb{E} \left[  \lVert R^{(t)}  \rVert^2 \right]\\
        &\quad+ \frac{8\eta T \sigma^2}{G} + \frac{4}{\eta } \mathbb{E} \left[  \lVert \widetilde M^{(0)}  \rVert^2 \right]+(\frac{8(\alpha^2+12\eta^2)(8\kappa+1)}{\alpha^2})\eta T\sigma^2+32\kappa T\zeta^2\\
        &\quad +\Big(\frac{16\eta^2(8\kappa+1)(\widetilde\ell^2+\frac{3}{\alpha}\widetilde L^2)}{\alpha}+\frac{8(8\kappa+1)(\alpha^2+12\eta^2)(\widetilde\ell^2+\frac{2}{\eta}\widetilde L^2)}{\alpha^2\eta}\Big)\sum_{t=0}^{T-1} \mathbb{E} \Big[\lVert R^{(t)}\rVert^2 \Big]\\
        &\leq \Big( \frac{\widetilde \ell^2(96\eta^4(8\kappa+1)+48(8\kappa+1)(\alpha^2+12\eta^2)+64\kappa\alpha^2)}{\alpha^2\eta} +\frac{8(8\kappa+1)(\alpha^2+12\eta^2)(\widetilde\ell^2+\frac{2}{\eta}\widetilde L^2)}{\alpha^2\eta}\\
        &\quad+\frac{8\widetilde\ell^2}{\eta }+\frac{16\eta^2(8\kappa+1)(\widetilde\ell^2+\frac{3}{\alpha}\widetilde L^2)}{\alpha}\Big)\sum_{t=0}^{T-1}\mathbb{E} \left[  \lVert R^{(t)}  \rVert^2 \right]+32\kappa T\zeta^2+\frac{4}{\eta}\mathbb{E}\Big[\lVert \widetilde{M}^{(0)}\rVert^2\Big]\\
        &\quad+\frac{48\eta^4(8\kappa+1)+24(8\kappa+1)(\alpha^2+12\eta^2)+32\kappa\alpha^2}{\alpha^2\eta G} \sum_{i\in \mathcal{G}}\mathbb{E}\left[  \lVert M_i^{(0)}\rVert^2 \right]\\
        &\quad+\Big(\frac{96\eta^4(8\kappa+1)+56(8\kappa+1)(\alpha^2+12\eta^2)}{\alpha^2 }+ 32\kappa +\frac{8}{G}+\frac{16\eta^3(8\kappa+1)}{\alpha}\Big)\eta T\sigma^2.
    \end{split}
\end{equation*}
Subtracting $f(x^*)$ from both sides of inequality \eqref{descent_vr}, taking expectation and defining $\delta_t \eqdef \mathbb{E}\Big[ \nabla f(x^{(t)})-f(x^*)\Big]$, we derive
\begin{equation*}
    \begin{split}
        &\mathbb{E} \left[ \lVert \nabla f(\hat{x}^{(T)})\rVert^2\right] \\
        &\leq \frac{2\delta}{\gamma T}-\frac{A}{ T}\sum_{t=0}^{T-1}\mathbb{E}\left[\lVert x^{(t+1)}-x^{(t)}\rVert^2\right]+32\kappa \zeta^2+\frac{4}{\eta T}\mathbb{E}\Big[\lVert \widetilde{M}^{(0)}\rVert^2\Big]\\
        &\quad+\frac{48\eta^4(8\kappa+1)+24(8\kappa+1)(\alpha^2+12\eta^2)+32\kappa\alpha^2}{\alpha^2 T\eta G} \sum_{i\in \mathcal{G}}\mathbb{E}\left[  \lVert M_i^{(0)}\rVert^2 \right]\\ 
        &\quad+\Big(\frac{96(8\kappa+1)(\eta^4+7\eta^2)}{\alpha^2 }+ 8(60\kappa+7)+\frac{8}{G}+\frac{16\eta^3(8\kappa+1)}{\alpha}\Big)\eta \sigma^2,
    \end{split}
\end{equation*}
where $\hat{x}^{(T)}$ is sampled uniformly at random from $T$ iterates and
\begin{equation*}
    \begin{split}
        A &= \frac{1}{\gamma^2}\Big( \frac{1}{2}-\frac{16\gamma^2\widetilde \ell^2(6\eta^4(8\kappa+1)+3(8\kappa+1)(\alpha^2+12\eta^2)+4\kappa\alpha^2)}{\alpha^2\eta} -\frac{8\gamma^2\widetilde\ell^2}{\eta }\\
        &\quad\quad\quad-\frac{8\gamma^2(8\kappa+1)(\alpha^2+12\eta^2)(\widetilde\ell^2+\frac{2}{\eta}\widetilde L^2)}{\alpha^2\eta}-\frac{16\gamma^2\eta^2(8\kappa+1)(\widetilde\ell^2+\frac{3}{\alpha}\widetilde L^2)}{\alpha}\Big)\\
        &= \frac{1}{\gamma^2}\Big( \frac{1}{2}-\frac{16\gamma^2\widetilde \ell^2(6\eta^4(8\kappa+1)+3(8\kappa+1)(\alpha^2+12\eta^2))}{\alpha^2\eta}-\frac{64\gamma^2\widetilde \ell^2 \kappa}{\eta} -\frac{8\gamma^2\widetilde\ell^2}{\eta }\\
        &\quad\quad\quad-\frac{8\gamma^2\widetilde\ell^2(8\kappa+1)(\alpha^2+12\eta^2)}{\alpha^2\eta}-\frac{16\gamma^2\widetilde L^2(8\kappa+1)(\alpha^2+12\eta^2)}{\alpha^2\eta^2}\\
        &\quad\quad\quad-\frac{16\gamma^2\widetilde\ell^2\eta^2(8\kappa+1)}{\alpha}-\frac{48\gamma^2\widetilde L^2\eta^2(8\kappa+1)}{\alpha^2}\Big)\\
        &= \frac{1}{\gamma^2}\Big( \frac{1}{2}-\frac{96\gamma^2\widetilde \ell^2\eta^3(8\kappa+1)}{\alpha^2}-\frac{48\gamma^2\widetilde \ell^2(8\kappa+1)}{\eta}-\frac{576\eta\gamma^2\widetilde \ell^2(8\kappa+1)}{\alpha^2}-\frac{64\gamma^2\widetilde \ell^2 \kappa}{\eta} -\frac{8\gamma^2\widetilde\ell^2}{\eta }\\
        &\quad\quad\quad-\frac{8\gamma^2\widetilde\ell^2(8\kappa+1)}{\eta}-\frac{96\eta\gamma^2\widetilde\ell^2(8\kappa+1)}{\alpha^2}-\frac{16\gamma^2\widetilde L^2(8\kappa+1)}{\eta^2}-\frac{192\gamma^2\widetilde L^2(8\kappa+1)}{\alpha^2}\\
        &\quad\quad\quad-\frac{16\gamma^2\widetilde\ell^2\eta^2(8\kappa+1)}{\alpha}-\frac{48\gamma^2\widetilde L^2\eta^2(8\kappa+1)}{\alpha^2}\Big)\\
        &= \frac{1}{\gamma^2}\Big( \frac{1}{2}-\frac{8\gamma^2(8\kappa+1)( 12\eta^3\widetilde \ell^2+84\eta\widetilde \ell^2 +24\widetilde L^2+6\eta^2\widetilde L^2+2\alpha\eta^2\widetilde \ell^2)}{\alpha^2}-\frac{64\gamma^2\widetilde \ell^2(8\kappa+1)}{\eta}\\
        &\quad\quad\quad-\frac{16\gamma^2\widetilde L^2(8\kappa+1)}{\eta^2}\Big)\\ 
        &\overset{(i)} {\geq}0.
    \end{split}
\end{equation*}
where $(i)$ are due to $\eta\leq 1$, $\alpha\leq 1$,  and the assumption on step-size.


Finally, by using the choice of the momentum parameter
\begin{equation*}
\begin{split}
    \eta \leq \min \Bigg\{ &  \Big( \frac{40\alpha^2\delta_0\sqrt{(8\kappa+1)\widetilde\ell^2}}{96(8\kappa+1) \sigma^2 T}\Big)^{\nicefrac{2}{11}}, \Big( \frac{40\alpha^2\delta_0\sqrt{(8\kappa+1)\widetilde\ell^2}}{672(8\kappa+1) \sigma^2 T}\Big)^{\nicefrac{2}{7}}, \Big( \frac{40\delta_0\sqrt{(8\kappa+1)\widetilde\ell^2}}{8(60\kappa+7) \sigma^2 T}\Big)^{\nicefrac{2}{3}},\\
    & \Big( \frac{40\delta_0G\sqrt{(8\kappa+1)\widetilde\ell^2}}{ 8\sigma^2 T}\Big)^{\nicefrac{2}{3}},\Big( \frac{40\alpha\delta_0\sqrt{(8\kappa+1)\widetilde\ell^2}}{16(8\kappa+1) \sigma^2 T}\Big)^{\nicefrac{2}{9}}   \Bigg\},
\end{split}
\end{equation*} ensures that $\frac{96\eta^5 (\kappa+1)\sigma^2}{\alpha^2}\leq \frac{40\sqrt{(8 \kappa+1)\widetilde \ell^2} }{\sqrt{\eta} T}, \frac{672\eta^3 (\kappa+1)\sigma^2}{\alpha^2}\leq \frac{40\sqrt{(8\kappa+1)\widetilde \ell^2} }{\sqrt{\eta} T}, 8\eta (60\kappa+7)\sigma^2\leq \frac{40\sqrt{(8\kappa+1)\widetilde \ell^2} }{\sqrt{\eta} T},\frac{8\eta \sigma^2}{G}\leq \frac{40\sqrt{(8\kappa+1)\widetilde \ell^2} }{\sqrt{\eta} T},$ and $\frac{16\eta^4 (8\kappa+1)\sigma^2}{\alpha}\leq \frac{40\sqrt{(8\kappa+1)\widetilde \ell^2} }{\sqrt{\eta} T}$ ,
we obtain
\begin{equation*}
    \begin{split}
        &\mathbb{E}\Big[\lVert \nabla f(\hat{x}^{(T)})\rVert^2\Big]\\
        &\leq\Big( \frac{40(96(8\kappa+1))^{\nicefrac{1}{10}}\sigma^{\nicefrac{2}{10}}\delta_0\sqrt{(8\kappa+1)\widetilde\ell^2}}{ \alpha^{\nicefrac{2}{10}} T}\Big)^{\nicefrac{10}{11}}+\Big( \frac{40(672(8\kappa+1))^{\nicefrac{1}{6}}\sigma^{\nicefrac{2}{6}}\delta_0\sqrt{(8\kappa+1)\widetilde\ell^2}}{ \alpha^{\nicefrac{2}{6}} T}\Big)^{\nicefrac{6}{7}}\\
        &+\Big( \frac{40(8(60\kappa+7))^{\nicefrac{1}{2}}\sigma\delta_0\sqrt{(8\kappa+1)\widetilde\ell^2}}{  T}\Big)^{\nicefrac{2}{3}}+\Big( \frac{40(\sigma\delta_0\sqrt{8(8\kappa+1)\widetilde\ell^2})}{ G^{\nicefrac{1}{2}} T}\Big)^{\nicefrac{2}{3}}\\
        &+\Big( \frac{40(16(8\kappa+1))^{\nicefrac{1}{8}}\sigma^{\nicefrac{1}{4}}\delta_0\sqrt{(8\kappa+1)\widetilde\ell^2}}{ \alpha^{\nicefrac{1}{8}}T}\Big)^{\nicefrac{8}{9}}+32\kappa \zeta^2+\frac{\Phi_0}{\gamma T}.
    \end{split}
\end{equation*}
This concludes the proof.
\end{proof}

\newpage
\section{Missing Proofs for Polyak-{\L}ojasiewicz Functions}\label{appendix:PL} 
In this section, we prove our convergence rates under the (Polyak-{\L}ojasiewicz) P{\L}-condition.


\subsection{Byz-DM21}

\begin{lemma}[Descent lemma]\label{lemma:DM21_pl}
    Suppose that Assumptions \ref{assump:smoothness}, \ref{assump:heterogeneity}, and \ref{as:pl} hold. Then for all $s>0$ we have
    \begin{equation}
    \begin{split}
        \mathbb{E}\Big[f(x^{t + 1}) - f(x^*)\Big]  
                        &\leq (1-\gamma \mu) \mathbb{E}\Big[f(x^{(t)})-f(x^{*})\Big]- (\frac{1}{2\gamma}-\frac{L}{2})\mathbb{E}\Big[\lVert x^{(t+1)}-x^{(t)}\rVert^2\Big]\\
                        &\quad+\frac{\gamma}{2}\mathbb{E}\Big[\lVert g^{(t)} - \nabla f(x^{(t)}) \rVert^2 \Big].
    \end{split}
    \end{equation}
\end{lemma}
\begin{proof}
    The result follows from combining Lemma \ref{lem:descent} with Assumption \ref{as:pl}.
\end{proof}

\begin{theorem}\label{thm:general_pl_DM21}
	Let Assumptions \ref{assump:smoothness}, \ref{assump:heterogeneity}, \ref{assump:bound_variance}, and \ref{as:pl} hold. Let us take $\eta \in (0,1]$ and 
    \begin{equation*}
        0 < \gamma \leq \min\left\{\frac{1}{L + \sqrt{A}}, \frac{\eta}{2\mu}, \frac{\alpha}{4\mu} \right\}
    \end{equation*}
    where $A = \Big( \frac{104L^2(8\kappa+1)}{\eta^2}+\frac{48L^2(8\kappa+1)(2\eta^4+28\eta^2+\alpha^2+48)}{\alpha^2}\Big)$, Then for all $T \geq 0$ the iterates of \algname{Byz-DM21} satisfy
    \begin{equation*}
    \begin{split}
        \mathbb{E}\Big[f(x^{T})-f(x^*)\Big]\leq& (1-\gamma \mu)^T \mathbb{E}\Big[\Psi^{0}\Big]+\frac{4\eta\sigma^2(112\kappa G+13G+1)}{\mu G}\\
        &+\frac{96\eta^3\sigma^2\Big( 8\kappa+13+\eta^2(8\kappa+1)\Big)}{\mu\alpha^2}+\frac{16\kappa\zeta^2}{\mu},
    \end{split}
    \end{equation*}
    where $\mathbb{E}\Big[\Psi^{0}\Big] = \mathbb{E}\Big[ f(x^{(0)})-f(x^{*})\Big]+\frac{8\gamma(8\kappa+1)}{\alpha}\mathbb{E}\Big[\lVert g^{(0)}-u^{(0)}\rVert^2\Big]+\frac{4\gamma(8\kappa+1)(\alpha^2+24\eta^2)}{\eta\alpha^2}\mathbb{E}\Big[\lVert  u^{(0)}-v^{(0)} \rVert^2 \Big]+\frac{16\gamma((8\kappa+1)(3\alpha^2+72\eta^2+6\eta^4)+2\kappa\alpha^2)}{\alpha^2 \eta}\mathbb{E}\Big[ \lVert v^{(0)} -\nabla f(x^{0})\rVert^2\Big]+\frac{4\gamma}{\eta}  \mathbb{E}\Big[\lVert  \overline{v}^{(0)}-\nabla f(x^{0})\rVert^2\Big]$.
    \begin{proof}
    Let
    \begin{align}\label{system_1}
    \begin{cases}
     D, E, F, H\geq 0\\
    D = \frac{8\gamma(8\kappa+1)}{\alpha}\\
    E = \frac{4\gamma(8\kappa+1)(\alpha^2+24\eta^2)}{\eta\alpha^2}\\
    F = \frac{16\gamma((8\kappa+1)(3\alpha^2+72\eta^2+6\eta^4)+2\kappa\alpha^2 )}{\alpha^2 \eta}\\
    H = \frac{4\gamma}{\eta}
    \end{cases}
    \end{align}

        using inequality \eqref{bound_deviation} and Assumption \ref{as:pl}, we obtain
        \begin{equation}\label{first_pl}
            \begin{split}
                 &\mathbb{E}\Big[ f(x^{(t+1)})-f(x^{*})\Big]\\
                 &\leq(1-\gamma \mu) \mathbb{E}\Big[f(x^{(t)})-f(x^{*})\Big]- (\frac{1}{2\gamma}-\frac{L}{2})\mathbb{E}\Big[\lVert x^{(t+1)}-x^{(t)}\rVert^2\Big]+\frac{\gamma}{2}\Psi^{(t)}\\
                 &\leq (1-\gamma \mu) \mathbb{E}\Big[f(x^{(t)})-f(x^{*})\Big]- (\frac{1}{2\gamma}-\frac{L}{2})\mathbb{E}\Big[\lVert x^{(t+1)}-x^{(t)}\rVert^2\Big]+2\gamma\mathbb{E}\Big[\lVert \widetilde M^{(t)}\rVert^2\Big]+16\gamma\kappa\zeta^2\\
                 &\quad +\frac{2\gamma(8\kappa+1)}{G}\sum_{i\in \mathcal{G}}\mathbb{E}\Big[\lVert C_i^{(t)}\rVert^2\Big]+\frac{2\gamma(8\kappa+1)}{G}\sum_{i\in \mathcal{G}}\mathbb{E}\Big[\lVert P_i^{(t)}\rVert^2\Big]+\frac{16\gamma\kappa}{G}\sum_{i\in \mathcal{G}}\mathbb{E}\Big[\lVert M^{(t)}\rVert^2\Big]\\
                 &\leq (1-\gamma \mu) \mathbb{E}\Big[f(x^{(t)})-f(x^{*})\Big]- (\frac{1}{2\gamma}-\frac{L}{2})\mathbb{E}\Big[\lVert x^{(t+1)}-x^{(t)}\rVert^2\Big]+2\gamma\mathbb{E}\Big[\lVert \widetilde M^{(t)}\rVert^2\Big]+16\gamma\kappa\zeta^2\\
                 &\quad +2\gamma(8\kappa+1)\mathbb{E}\Big[\lVert C^{(t)}\rVert^2\Big]+2\gamma(8\kappa+1)\mathbb{E}\Big[\lVert P^{(t)}\rVert^2\Big]+16\gamma\kappa\mathbb{E}\Big[\lVert M^{(t)}\rVert^2\Big].
            \end{split}
        \end{equation}
        Let $\delta^{(t+1)}\eqdef f(x^{(t+1)})-f(x^{*})$, $R^{(t)}\eqdef x^{(t+1)}-x^{(t)}$. By adding $D \cdot\mathbb{E}\Big[\lVert C^{(t+1)}\rVert^2\Big]$ and using Lemma \ref{lem:compression_error}, we have
        \begin{equation*}
            \begin{split}
                &\mathbb{E}\Big[ \delta^{(t+1)}\Big]+D\cdot \mathbb{E}\Big[\lVert C^{(t+1)}\rVert^2\Big]\\
 &\leq (1-\gamma \mu) \mathbb{E}\Big[\delta^{(t)}\Big]- (\frac{1}{2\gamma}-\frac{L}{2})\mathbb{E}\Big[\lVert R^{(t)}\rVert^2\Big]+2\gamma\mathbb{E}\Big[\lVert \widetilde M^{(t)}\rVert^2\Big]+16\gamma\kappa\zeta^2\\
 &\quad +2\gamma(8\kappa+1)\mathbb{E}\Big[\lVert C^{(t)}\rVert^2\Big]+2\gamma(8\kappa+1)\mathbb{E}\Big[\lVert P^{(t)}\rVert^2\Big]+16\gamma\kappa\mathbb{E}\Big[\lVert M^{(t)}\rVert^2\Big]\\
 &\quad + D\cdot \Bigg( (1-\frac{\alpha}{2}) \mathbb{E}\left[\lVert  C^{(t)} \rVert^2 \right] +  \frac{6\eta^4}{\alpha}\mathbb{E}\left[\lVert   M^{(t)}\rVert^2 \right]   +  \eta^4 \sigma^2+ \frac{6\eta^4 L^2}{\alpha} \mathbb{E}\left[\lVert R^{(t)}\rVert^2 \right] \\
 &\quad+ \frac{6\eta^2}{\alpha} \mathbb{E}\left[\lVert  P^{(t)}\rVert^2 \right]\Bigg)\\
 &\leq (1-\gamma \mu) \mathbb{E}\Big[\delta^{(t)}\Big]- (\frac{1}{2\gamma}-\frac{L}{2})\mathbb{E}\Big[\lVert R^{(t)}\rVert^2\Big]+2\gamma\mathbb{E}\Big[\lVert \widetilde M^{(t)}\rVert^2\Big]+16\gamma\kappa\zeta^2\\
 &\quad +2\gamma(8\kappa+1)\mathbb{E}\Big[\lVert C^{(t)}\rVert^2\Big]+2\gamma(8\kappa+1)\mathbb{E}\Big[\lVert P^{(t)}\rVert^2\Big]+16\gamma\kappa\mathbb{E}\Big[\lVert M^{(t)}\rVert^2\Big]\\
 &\quad + D\cdot (1-\frac{\alpha}{2}) \mathbb{E}\left[\lVert  C^{(t)} \rVert^2 \right] +D\cdot  \frac{6\eta^4}{\alpha}\mathbb{E}\left[\lVert   M^{(t)}\rVert^2 \right] +D\cdot \eta^4 \sigma^2+D\cdot\frac{6\eta^4 L^2}{\alpha} \mathbb{E}\left[\lVert R^{(t)}\rVert^2 \right] \\
 &\quad+ D\cdot\frac{6\eta^2}{\alpha} \mathbb{E}\left[\lVert  P^{(t)}\rVert^2 \right].
\end{split}
\end{equation*}
Using inequality \eqref{system_1} and substituting $D = \frac{8\gamma(8\kappa+1)}{\alpha}$, we get

\begin{equation*}
    \begin{split}
    &\mathbb{E}\Big[ \delta^{(t+1)}\Big]+D\cdot\mathbb{E}\Big[\lVert C^{(t+1)}\rVert^2\Big]\\
 &\leq (1-\gamma \mu) \mathbb{E}\Big[\delta^{(t)}\Big]- (\frac{1}{2\gamma}-\frac{L}{2}-\frac{48\gamma\eta^4 L^2(8\kappa+1)}{\alpha^2})\mathbb{E}\Big[\lVert R^{(t)}\rVert^2\Big]+2\gamma\mathbb{E}\Big[\lVert \widetilde M^{(t)}\rVert^2\Big]+16\gamma\kappa\zeta^2\\
 &\quad + D\cdot(1-\frac{\alpha}{4})\mathbb{E}\left[\lVert  C^{(t)} \rVert^2 \right]+\frac{2\gamma(8\kappa+1)(\alpha^2+24\eta^2)}{\alpha^2}\mathbb{E}\left[\lVert  P^{(t)} \rVert^2 \right]\\
 &\quad + \frac{16\gamma(3\eta^4(8\kappa+1)+\kappa\alpha^2)}{\alpha^2}\mathbb{E}\left[\lVert  M^{(t)} \rVert^2 \right]+ \frac{8\gamma\eta^4\sigma^2(8\kappa+1)}{\alpha}.
            \end{split}
        \end{equation*}
        Next, by adding $E\cdot\mathbb{E}\Big[\lVert  P^{(t+1)} \rVert^2 \Big]$ and using Lemma \ref{lem:second_momentum}, we have
        \begin{equation*}
    \begin{split}
         &\mathbb{E}\Big[ \delta^{(t+1)}\Big]+D\cdot\mathbb{E}\Big[\lVert C^{(t+1)}\rVert^2\Big]+E\cdot\mathbb{E}\left[\lVert  P^{(t+1)} \rVert^2 \right]\\
         &\leq  (1-\gamma \mu) \mathbb{E}\Big[\delta^{(t)}\Big]- (\frac{1}{2\gamma}-\frac{L}{2}-\frac{48\gamma\eta^4 L^2(8\kappa+1)}{\alpha^2})\mathbb{E}\Big[\lVert R^{(t)}\rVert^2\Big]+2\gamma\mathbb{E}\Big[\lVert \widetilde M^{(t)}\rVert^2\Big]+16\gamma\kappa\zeta^2\\
         &\quad+ \frac{8\gamma\eta^4\sigma^2(8\kappa+1)}{\alpha}+\frac{16\gamma(3\eta^4(8\kappa+1)+\kappa\alpha^2)}{\alpha^2}\mathbb{E}\left[\lVert  M^{(t)} \rVert^2 \right]\\
 &\quad + D\cdot(1-\frac{\alpha}{4})\mathbb{E}\left[\lVert  C^{(t)} \rVert^2 \right]+\frac{2\gamma(8\kappa+1)(\alpha^2+24\eta^2)}{\alpha^2}\mathbb{E}\left[\lVert  P^{(t)} \rVert^2 \right] \\
 &\quad +E\cdot\Big((1-\eta)\mathbb{E}\Big[ \lVert  P^{(t)}\rVert^2 \Big] +  6\eta\mathbb{E}\Big[ \lVert M^{(t)}\rVert^2 \Big]+6\eta L^2\mathbb{E}\Big[ \lVert R^{(t)}\rVert^2 \Big]+\eta^2\sigma^2\Big).
 \end{split}
 \end{equation*}
 Using inequality \eqref{system_1} and substituting $E = \frac{4\gamma(8\kappa+1)(\alpha^2+24\eta^2)}{\eta\alpha^2}$, we attain
 \begin{equation*}
     \begin{split}
     &\mathbb{E}\Big[ \delta^{(t+1)}\Big]+D\cdot\mathbb{E}\Big[\lVert C^{(t+1)}\rVert^2\Big]+E\cdot\mathbb{E}\left[\lVert  P^{(t+1)} \rVert^2 \right]\\
 &\leq  (1-\gamma \mu) \mathbb{E}\Big[\delta^{(t)}\Big]- \Big(\frac{1}{2\gamma}-\frac{L}{2}-\frac{48\gamma\eta^4 L^2(8\kappa+1)}{\alpha^2}-\frac{24\gamma L^2(8\kappa+1)(\alpha^2+24\eta^2)}{\alpha^2}\Big)\mathbb{E}\Big[\lVert R^{(t)}\rVert^2\Big]\\
 &\quad+2\gamma\mathbb{E}\Big[\lVert \widetilde M^{(t)}\rVert^2\Big]+\frac{4\eta\gamma\sigma^2(8\kappa+1)(\alpha^2+24\eta^2)}{\alpha^2}+16\gamma\kappa\zeta^2\\
 &\quad + D\cdot(1-\frac{\alpha}{4})\mathbb{E}\left[\lVert  C^{(t)} \rVert^2 \right]+E\cdot(1-\frac{\eta}{2})\mathbb{E}\left[\lVert  P^{(t)} \rVert^2 \right]\\
 &\quad+\frac{8\gamma\Big((8\kappa+1)(3\alpha^2+72\eta^2)+6\eta^4(8\kappa+1)+2\kappa\alpha^2 \Big)}{\alpha^2}\mathbb{E}\Big[ \lVert M^{(t)}\rVert^2\Big] + \frac{8\gamma\eta^4\sigma^2(8\kappa+1)}{\alpha}.
    \end{split}
\end{equation*}
Then, by adding $F \cdot \mathbb{E}\Big[ \lVert M^{(t+1)}\rVert^2\Big]$, using Lemma \ref{lem:momentum}, and substituting $ F = \frac{16\gamma((8\kappa+1)(3\alpha^2+72\eta^2+6\eta^4)+2\kappa\alpha^2 )}{\alpha^2 \eta}$, we obtain
\begin{equation*}
    \begin{split}
        &\mathbb{E}\Big[ \delta^{(t+1)}\Big]+D\cdot\mathbb{E}\Big[\lVert C^{(t+1)}\rVert^2\Big]+E\cdot\mathbb{E}\left[\lVert  P^{(t+1)} \rVert^2 \right] +F\cdot\mathbb{E}\Big[ \lVert M^{(t+1)}\rVert^2\Big]\\
        &\leq  (1-\gamma \mu) \mathbb{E}\Big[\delta^{(t)}\Big]- \Big(\frac{1}{2\gamma}-\frac{L}{2}-\frac{48\gamma\eta^4 L^2(8\kappa+1)}{\alpha^2}-\frac{24\gamma L^2(8\kappa+1)(\alpha^2+24\eta^2)}{\alpha^2}\Big)\mathbb{E}\Big[\lVert R^{(t)}\rVert^2\Big]\\
 &\quad + D\cdot(1-\frac{\alpha}{4})\mathbb{E}\left[\lVert  C^{(t)} \rVert^2 \right]+E\cdot(1-\frac{\eta}{2})\mathbb{E}\left[\lVert  P^{(t)} \rVert^2 \right]+\frac{4\eta\gamma\sigma^2(8\kappa+1)(\alpha^2+24\eta^2)}{\alpha^2}+16\gamma\kappa\zeta^2\\
 &\quad+\frac{8\gamma\Big((8\kappa+1)(3\alpha^2+72\eta^2+6\eta^4)+2\kappa\alpha^2 \Big)}{\alpha^2}\mathbb{E}\Big[ \lVert M^{(t)}\rVert^2\Big] + \frac{8\gamma\eta^4\sigma^2(8\kappa+1)}{\alpha} +2\gamma\mathbb{E}\Big[\lVert \widetilde M^{(t)}\rVert^2\Big]\\
 &\quad+F\cdot\Big(\mathbb{E} \left[  \lVert M^{(t)}  \rVert^2 \right]   + \frac{ L^2}{\eta} \mathbb{E} \left[  \lVert R^{(t)}  \rVert^2 \right] + \eta^2 \sigma^2 \Big)\\
 &\leq (1-\gamma \mu) \mathbb{E}\Big[\delta^{(t)}\Big]+ D\cdot(1-\frac{\alpha}{4})\mathbb{E}\left[\lVert  C^{(t)} \rVert^2 \right]+E\cdot(1-\frac{\eta}{2})\mathbb{E}\left[\lVert  P^{(t)} \rVert^2 \right]\\
 &\quad+\frac{4\eta\gamma\sigma^2(8\kappa+1)(\alpha^2+24\eta^2+2\alpha\eta^3)}{\alpha^2}+ F\cdot(1-\frac{\eta}{2})\mathbb{E}\Big[ \lVert M^{(t)}\rVert^2\Big]\\
 &\quad- \Bigg(\frac{1}{2\gamma}-\frac{L}{2}-\frac{24\gamma L^2(8\kappa+1)(2\eta^4+\alpha^2+24\eta^2)}{\alpha^2}\\
 &\quad \quad\quad-\frac{16\gamma L^2\Big((8\kappa+1)(3\alpha^2+72\eta^2+6\eta^4)+2\kappa\alpha^2 \Big)}{\alpha^2 \eta^2}\Bigg)\mathbb{E}\Big[\lVert R^{(t)}\rVert^2\Big]\\
 &\quad+2\gamma\mathbb{E}\Big[\lVert \widetilde M^{(t)}\rVert^2\Big]+16\gamma\kappa\zeta^2+\frac{16\gamma\eta\sigma^2\Big((8\kappa+1)(3\alpha^2+72\eta^2+6\eta^4)+2\kappa\alpha^2 \Big)}{\alpha^2 }.
    \end{split}
\end{equation*}
Furthermore, by adding $H \cdot \mathbb{E}\Big[\lVert \widetilde M^{(t+1)}\rVert^2\Big]$, using Lemma \ref{lem:momentum}, and substituting $ H = \frac{4\gamma}{\eta}$, we arrive at
\begin{equation*}
    \begin{split}
        &\mathbb{E}\Big[ \delta^{(t+1)}\Big]+D\cdot\mathbb{E}\Big[\lVert C^{(t+1)}\rVert^2\Big]+E\cdot\mathbb{E}\left[\lVert  P^{(t+1)} \rVert^2 \right] +F\cdot\mathbb{E}\Big[ \lVert M^{(t+1)}\rVert^2\Big]+H\cdot \mathbb{E}\Big[\lVert \widetilde M^{(t+1)}\rVert^2\Big]\\
        &\leq (1-\gamma \mu) \mathbb{E}\Big[\delta^{(t)}\Big] + D\cdot(1-\frac{\alpha}{4})\mathbb{E}\left[\lVert  C^{(t)} \rVert^2 \right]+E\cdot(1-\frac{\eta}{2})\mathbb{E}\left[\lVert  P^{(t)} \rVert^2 \right]\\
        &\quad+\frac{4\eta\gamma\sigma^2(8\kappa+1)(\alpha^2+24\eta^2+2\alpha\eta^3)}{\alpha^2}+16\gamma\kappa\zeta^2\\
        &\quad-\Bigg(\frac{1}{2\gamma}-\frac{L}{2}-\frac{24\gamma L^2(8\kappa+1)(2\eta^4+\alpha^2+24\eta^2)}{\alpha^2}\\
        &\quad\quad\quad-\frac{16\gamma L^2\Big((8\kappa+1)(3\alpha^2+72\eta^2+6\eta^4)+2\kappa\alpha^2 \Big)}{\alpha^2 \eta^2}\Bigg)\mathbb{E}\Big[\lVert R^{(t)}\rVert^2\Big]\\
 &\quad + F\cdot(1-\frac{\eta}{2})\mathbb{E}\Big[ \lVert M^{(t)}\rVert^2\Big]+\frac{16\gamma\eta\sigma^2\Big((8\kappa+1)(3\alpha^2+72\eta^2+6\eta^4)+2\kappa\alpha^2 \Big)}{\alpha^2 }\\
 &\quad +2\gamma\mathbb{E}\Big[\lVert \widetilde M^{(t)}\rVert^2\Big]+H\cdot\Big( (1-\eta)\mathbb{E} \left[  \lVert \widetilde M^{(t)}  \rVert^2 \right]   + \frac{  L^2}{\eta} \mathbb{E} \left[  \lVert R^{(t)}  \rVert^2 \right] + \frac{\eta^2 \sigma^2}{G}\Big)\\
 &\leq (1-\gamma \mu) \mathbb{E}\Big[\delta^{(t)}\Big]+ D\cdot(1-\frac{\alpha}{4})\mathbb{E}\left[\lVert  C^{(t)} \rVert^2 \right]+E\cdot(1-\frac{\eta}{2})\mathbb{E}\left[\lVert  P^{(t)} \rVert^2 \right]\\
 &\quad+\frac{4\eta\gamma\sigma^2(8\kappa+1)(\alpha^2+24\eta^2+2\alpha\eta^3)}{\alpha^2}+16\gamma\kappa\zeta^2\\
 &\quad + F\cdot(1-\frac{\eta}{2})\mathbb{E}\Big[ \lVert M^{(t)}\rVert^2\Big]+\frac{16\gamma\eta\sigma^2\Big((8\kappa+1)(3\alpha^2+72\eta^2+6\eta^4)+2\kappa\alpha^2 \Big)}{\alpha^2 }\\
 &\quad +H\cdot(1-\frac{\eta}{2})\mathbb{E} \left[  \lVert \widetilde M^{(t)}  \rVert^2 \right]   + \frac{ 4\gamma L^2}{\eta^2} \mathbb{E} \left[  \lVert R^{(t)}  \rVert^2 \right] + \frac{4\gamma\eta \sigma^2}{G}.
    \end{split}
\end{equation*}
Finally, we bound $R^{(t)}$,
\begin{equation*}
    \begin{split}
        &\Bigg(\frac{1}{2\gamma}-\frac{L}{2}-\frac{4\gamma L^2}{\eta^2}-\frac{24\gamma L^2(8\kappa+1)(2\eta^4+\alpha^2+24\eta^2)}{\alpha^2}\\
        &\quad-\frac{16\gamma L^2\Big((8\kappa+1)(3\alpha^2+72\eta^2+6\eta^4)+2\kappa\alpha^2 \Big)}{\alpha^2 \eta^2}\Bigg)\mathbb{E}\Big[\lVert R^{(t)}\rVert^2\Big] \geq 0.
    \end{split}
\end{equation*}
Let 
\begin{equation*}
\begin{split}
A \eqdef& \Big(\frac{8 L^2}{\eta^2}+\frac{48 L^2(8\kappa+1)(2\eta^4+\alpha^2+24\eta^2)}{\alpha^2}+\frac{32 L^2\Big((8\kappa+1)(3\alpha^2+72\eta^2+6\eta^4)+2\kappa\alpha^2 \Big)}{\alpha^2 \eta^2}\Big)\\
&=\Big( \frac{104L^2(8\kappa+1)}{\eta^2}+\frac{48L^2(8\kappa+1)(2\eta^4+28\eta^2+\alpha^2+48)}{\alpha^2}\Big).
\end{split}
\end{equation*}
Taking $0 < \gamma \leq \frac{1}{L+\sqrt{A}}$ and applying Lemma \ref{lemma:step_lemma} gives
\begin{equation*}
    \frac{1}{2\gamma}-\frac{L}{2}-\frac{\gamma A}{2}\geq 0.
\end{equation*}

Let $\mathbb{E}\Big[\Psi^{(t+1)}\Big] \eqdef\mathbb{E}\Big[ \delta^{(t+1)}\Big]+D\cdot\mathbb{E}\Big[\lVert C^{(t+1)}\rVert^2\Big]+E\cdot\mathbb{E}\Big[\lVert  P^{(t+1)} \rVert^2 \Big]+F\cdot\mathbb{E}\Big[ \lVert M^{(t+1)}\rVert^2\Big]+H\cdot \mathbb{E}\Big[\lVert \widetilde M^{(t+1)}\rVert^2\Big]$ and we use the assumption on $\eta$ and $\alpha$ to establish that $1 - \frac{\alpha}{4} \leq 1 - \gamma\mu$ and $1 - \frac{\eta}{2} \leq 1 - \gamma\mu$. Applying the inequality iteratively gives
\begin{equation*}
    \begin{split}
        \mathbb{E}\Big[\Psi^{T}\Big] &\leq (1-\gamma \mu)^T \mathbb{E}\Big[\Psi^{0}\Big]+\frac{4\eta\sigma^2(8\kappa+1)(\alpha^2+24\eta^2+2\alpha\eta^3)}{\mu\alpha^2}\\
        &\quad+\frac{16\eta\sigma^2\Big((8\kappa+1)(3\alpha^2+72\eta^2+6\eta^4)+2\kappa\alpha^2 \Big)}{\mu\alpha^2 }+\frac{16\kappa\zeta^2}{\mu} +\frac{4\eta \sigma^2}{\mu G}\\
        &\leq (1-\gamma \mu)^T \mathbb{E}\Big[\Psi^{0}\Big]+\frac{4\eta\sigma^2(8\kappa+1)}{\mu}+\frac{96\eta^3\sigma^2(8\kappa+1)}{\mu\alpha^2}+ 
 \frac{8\eta^4\sigma^2(8\kappa+1)}{\mu\alpha}\\
 &\quad+\frac{32\eta\sigma^2\kappa}{\mu}+\frac{16\kappa\zeta^2}{\mu} +\frac{4\eta \sigma^2}{\mu G}\\
        &\quad+\frac{48\eta\sigma^2(8\kappa+1)}{\mu}+\frac{96\eta\sigma^2(8\kappa+1)(12\eta^2+\eta^4) }{\mu\alpha^2 }\\
        &\leq (1-\gamma \mu)^T \mathbb{E}\Big[\Psi^{0}\Big]+\frac{4\eta\sigma^2(112\kappa G+13G+1)}{\mu G}\\
        &\quad+\frac{8\eta^3\sigma^2(8\kappa+1)( 12\eta^2+\eta\alpha+156)}{\mu\alpha^2}+\frac{16\kappa\zeta^2}{\mu}.
    \end{split}
\end{equation*}
Noting that $\mathbb{E}\Big[\Psi^T\Big]\geq \mathbb{E}\Big[ f(x^T) -f(x^*)\Big]$, we finish the proof.
    \end{proof}
\end{theorem}

\begin{corollary}\label{Dm21_PL}
    Suppose that assumptions from Theorem \ref{thm:general_pl_DM21} hold, momentum $\eta \leq \min\Big\{ \frac{\mu \varepsilon G}{(G(\kappa+1)+1)\sigma^2},\frac{\mu\alpha^2\varepsilon^{\nicefrac{1}{3}}}{(\kappa+1)\sigma^2},\frac{\mu\alpha\varepsilon^{1/4}}{(\kappa+1)\sigma^2}\Big\}$, and for all $i\in\mathcal{G}$, then Algorithm \ref{alg:SGD2M} needs
    \begin{equation}\label{complexity_DM21}
    \begin{split}
        T =& \widetilde{\mathcal{O}}\Bigg( \frac{(G(\kappa+1)+1)\sigma^2}{\mu\varepsilon G} + \frac{(\kappa+1)\sigma^2}{\mu\alpha^2\varepsilon^{\nicefrac{1}{3}}}+\frac{(\kappa+1)\sigma^2}{\mu\alpha\varepsilon^{1/4}}+\frac{L}{\mu}+\frac{L\sigma^2(G(\kappa+1)+1)\sqrt{(\kappa+1)}}{\mu^2\varepsilon G} \Bigg)
    \end{split}
    \end{equation}
    communication rounds to get an $\varepsilon$-solution.
\end{corollary}
\begin{proof}
    Considering the choice of $\eta$, we have $\frac{1}{\mu}\Big(\frac{4\eta(112\kappa G+13G+1)}{ G}+\frac{8\eta^3(8\kappa+1)(112\eta^2+\eta\alpha+156)}{\alpha^2}\Big)\sigma^2 = \mathcal{O}(\varepsilon)$, which guarantees that $\mathbb{E}\Big[f(x^{(T)})-f(x^*)\Big]\leq \varepsilon$ for $\varepsilon\geq \frac{32\kappa\zeta^2}{\mu }$. Therefore, is it sufficient to take the number of communication rounds equal \eqref{complexity_DM21} to get an $\varepsilon$-solution.
\end{proof}

\subsection{Byz-VR-DM21}
\begin{theorem}\label{thm:general_pl_VRDM21}
	Let Assumptions \ref{assump:smoothness}, \ref{assump:heterogeneity}, \ref{assump:bound_variance}, and \ref{as:pl} hold. Let us take $\eta \in (0,1]$ and 
    \begin{equation*}
        0 < \gamma \leq \min\left\{\frac{1}{L + \sqrt{A}}, \frac{\eta}{2\mu}, \frac{\alpha}{4\mu} \right\},
    \end{equation*}
    where $A = 32\Big( \frac{8(8\kappa+1)( L^2+7\eta\ell^2)}{\eta^2}+\frac{(8\kappa+1)(L^2(3\eta^2+24)+\ell^2(12\eta^3+\alpha\eta^2+156\eta))}{\alpha^2}\Big) $, Then for all $T \geq 0$ the iterates of \algname{Byz-VR-DM21} satisfy
    \begin{equation*}
    \begin{split}
        \mathbb{E}\Big[f(x^{T})-f(x^*)\Big]\leq &(1-\gamma \mu)^T \mathbb{E}\Big[\Psi^{0}\Big]+\frac{8\eta^3\sigma^2(8\kappa+1)(8\eta^2+2\alpha\eta+56)}{\mu\alpha^2}\\
        &\quad+\frac{4\eta\sigma^2(64\kappa G+7G+1)}{\mu G}+\frac{16\kappa\zeta^2}{\mu},
        \end{split}
    \end{equation*}
    where $\mathbb{E}\Big[\Psi^{0}\Big] = \mathbb{E}\Big[ f(x^{(0)})-f(x^{*})\Big]+\frac{8\gamma(8\kappa+1)}{\alpha}\mathbb{E}\Big[\lVert g^{(0)}-u^{(0)}\rVert^2\Big]+\frac{4\gamma(8\kappa+1)(\alpha^2+24\eta^2)}{\eta\alpha^2}\mathbb{E}\Big[\lVert  u^{(0)}-v^{(0)} \rVert^2 \Big]+\frac{16\gamma((8\kappa+1)(3\alpha^2+72\eta^2+6\eta^4)+2\kappa\alpha^2 )}{\alpha^2 \eta}\mathbb{E}\Big[ \lVert v^{(0)} -\nabla f(x^{0})\rVert^2\Big]+\frac{4\gamma}{\eta}  \mathbb{E}\Big[\lVert  \overline{v}^{(0)}-\nabla f(x^{0})\rVert^2\Big]$.
\end{theorem}
\begin{proof}
    Let
    \begin{align}\label{system_2}
    \begin{cases}
     D, E, F, H\geq 0\\
    D = \frac{8\gamma(8\kappa+1)}{\alpha}\\
    E = \frac{4\gamma(8\kappa+1)(\alpha^2+24\eta^2)}{\eta\alpha^2}\\
    F = \frac{16\gamma((8\kappa+1)(3\alpha^2+72\eta^2+6\eta^4)+2\kappa\alpha^2 )}{\alpha^2 \eta}\\
    H = \frac{4\gamma}{\eta} 
    \end{cases}
    \end{align}
    Using inequality \eqref{first_pl}, we obtain
    \begin{equation}
            \begin{split}
                 &\mathbb{E}\Big[ f(x^{(t+1)})-f(x^{*})\Big]\\
                 &\leq (1-\gamma \mu) \mathbb{E}\Big[f(x^{(t)})-f(x^{*})\Big]- (\frac{1}{2\gamma}-\frac{L}{2})\mathbb{E}\Big[\lVert x^{(t+1)}-x^{(t)}\rVert^2\Big]+2\gamma\mathbb{E}\Big[\lVert \widetilde M^{(t)}\rVert^2\Big]+16\gamma\kappa\zeta^2\\
                 &\quad +2\gamma(8\kappa+1)\mathbb{E}\Big[\lVert C^{(t)}\rVert^2\Big]+2\gamma(8\kappa+1)\mathbb{E}\Big[\lVert P^{(t)}\rVert^2\Big]+16\gamma\kappa\mathbb{E}\Big[\lVert M^{(t)}\rVert^2\Big].
            \end{split}
        \end{equation}
        Let $\delta^{(t+1)}\eqdef f(x^{(t+1)})-f(x^{*})$, $R^{(t)}\eqdef x^{(t+1)}-x^{(t)}$ and by adding $D \cdot\mathbb{E}\Big[\lVert C^{(t+1)}\rVert^2\Big]$, by using Lemma \ref{lem:compression_error_vr}, and substituting $ D = \frac{8\gamma(8\kappa+1)}{\alpha}$, we have
    \begin{equation*}
            \begin{split}
                &\mathbb{E}\Big[ \delta^{(t+1)}\Big]+D\cdot \mathbb{E}\Big[\lVert C^{(t+1)}\rVert^2\Big]\\
 &\leq (1-\gamma \mu) \mathbb{E}\Big[\delta^{(t)}\Big]- (\frac{1}{2\gamma}-\frac{L}{2})\mathbb{E}\Big[\lVert R^{(t)}\rVert^2\Big]+2\gamma\mathbb{E}\Big[\lVert \widetilde M^{(t)}\rVert^2\Big]+16\gamma\kappa\zeta^2\\
 &\quad +2\gamma(8\kappa+1)\mathbb{E}\Big[\lVert C^{(t)}\rVert^2\Big]+2\gamma(8\kappa+1)\mathbb{E}\Big[\lVert P^{(t)}\rVert^2\Big]+16\gamma\kappa\mathbb{E}\Big[\lVert M^{(t)}\rVert^2\Big]\\
 &\quad + D\cdot \Bigg( (1-\frac{\alpha}{2}) \mathbb{E}\left[\lVert  C^{(t)} \rVert^2 \right] +  \frac{6\eta^4}{\alpha}\mathbb{E}\left[\lVert   M^{(t)}\rVert^2 \right]   +  2\eta^4 \sigma^2+ 2\eta^2(\frac{3}{\alpha}L^2+\ell^2) \mathbb{E}\left[\lVert R^{(t)}\rVert^2 \right] \\
 &\quad+ \frac{6\eta^2}{\alpha} \mathbb{E}\left[\lVert  P^{(t)}\rVert^2 \right]\Bigg)\\
 &\leq (1-\gamma \mu) \mathbb{E}\Big[\delta^{(t)}\Big]- (\frac{1}{2\gamma}-\frac{L}{2}-\frac{16\gamma\eta^2 (8\kappa+1)(\frac{3}{\alpha}L^2+\ell^2)}{\alpha})\mathbb{E}\Big[\lVert R^{(t)}\Big]\rVert^2+2\gamma\mathbb{E}\Big[\lVert \widetilde M^{(t)}\rVert^2\Big]\\
 &\quad +16\gamma\kappa\zeta^2+ D\cdot(1-\frac{\alpha}{4})\mathbb{E}\left[\lVert  C^{(t)} \rVert^2 \right]+\frac{2\gamma(8\kappa+1)(\alpha^2+24\eta^2)}{\alpha^2}\mathbb{E}\left[\lVert  P^{(t)} \rVert^2 \right]\\
 &\quad + \frac{16\gamma(3\eta^4(8\kappa+1)+\kappa\alpha^2)}{\alpha^2}\mathbb{E}\left[\lVert  M^{(t)} \rVert^2 \right]+ \frac{16\gamma\eta^4\sigma^2(8\kappa+1)}{\alpha}.
            \end{split}
        \end{equation*}
    Next, by adding $E\cdot\mathbb{E}\Big[\lVert  P^{(t+1)} \rVert^2 \Big]$, using Lemma \ref{lem:second_momentum_vr}, and substituting $ E = \frac{4\gamma(8\kappa+1)(\alpha^2+24\eta^2)}{\eta\alpha^2}$, we get
        \begin{equation*}
    \begin{split}
         &\mathbb{E}\Big[ \delta^{(t+1)}\Big]+D\cdot\mathbb{E}\Big[\lVert C^{(t+1)}\rVert^2\Big]+E\cdot\mathbb{E}\left[\lVert  P^{(t+1)} \rVert^2 \right]\\
         &\leq  (1-\gamma \mu) \mathbb{E}\Big[\delta^{(t)}\Big]- (\frac{1}{2\gamma}-\frac{L}{2}-\frac{16\gamma\eta^2(8\kappa+1)(\frac{3}{\alpha}L^2+\ell^2)}{\alpha})\mathbb{E}\Big[\lVert R^{(t)}\rVert^2\Big]+2\gamma\mathbb{E}\Big[\lVert \widetilde M^{(t)}\rVert^2\Big]\\
         &\quad+16\gamma\kappa\zeta^2+ \frac{16\gamma\eta^4\sigma^2(8\kappa+1)}{\alpha}+ \frac{16\gamma(3\eta^4(8\kappa+1)+\kappa\alpha^2)}{\alpha^2}\mathbb{E}\left[\lVert  M^{(t)} \rVert^2 \right]\\
 &\quad + D\cdot(1-\frac{\alpha}{4})\mathbb{E}\left[\lVert  C^{(t)} \rVert^2 \right]+\frac{2\gamma(8\kappa+1)(\alpha^2+24\eta^2)}{\alpha^2}\mathbb{E}\left[\lVert  P^{(t)} \rVert^2 \right]\\
 &\quad +E\cdot\Big((1-\eta)\mathbb{E}\Big[ \lVert  P^{(t)}\rVert^2 \Big] +  6\eta\mathbb{E}\Big[ \lVert M^{(t)}\rVert^2 \Big]+ 2(\frac{2}{\eta}L^2+\ell^2)\mathbb{E}\Big[ \lVert R^{(t)}\rVert^2 \Big]+2\eta^2\sigma^2\Big)\\
 &\leq  (1-\gamma \mu) \mathbb{E}\Big[\delta^{(t)}\Big]+16\gamma\kappa\zeta^2+2\gamma\mathbb{E}\Big[\lVert \widetilde M^{(t)}\rVert^2\Big]\\
 &\quad- \Big(\frac{1}{2\gamma}-\frac{L}{2}-\frac{16\gamma\eta^2 (8\kappa+1)(\frac{3}{\alpha}L^2+\ell^2)}{\alpha}-\frac{8\gamma (8\kappa+1)(\frac{2}{\eta}L^2+\ell^2)(\alpha^2+24\eta^2)}{\eta\alpha^2}\Big)\mathbb{E}\Big[\lVert R^{(t)}\rVert^2\Big]\\
 &\quad + D\cdot(1-\frac{\alpha}{4})\mathbb{E}\left[\lVert  C^{(t)} \rVert^2 \right]+E\cdot(1-\frac{\eta}{2})\mathbb{E}\left[\lVert  P^{(t)} \rVert^2 \right]+\frac{8\eta\gamma\sigma^2(8\kappa+1)(\alpha^2+24\eta^2)}{\alpha^2}\\
 &\quad+\frac{8\gamma\Big((8\kappa+1)(3\alpha^2+72\eta^2+6\eta^4)+2\kappa\alpha^2 \Big)}{\alpha^2}\mathbb{E}\Big[ \lVert M^{(t)}\rVert^2\Big] + \frac{16\gamma\eta^4\sigma^2(8\kappa+1)}{\alpha}.
    \end{split}
\end{equation*}
Then, by adding $F \cdot \mathbb{E}\Big[ \lVert M^{(t+1)}\rVert^2\Big]$, using Lemma \ref{lem:momentum_vr}, and substituting $F = \frac{16\gamma((8\kappa+1)(3\alpha^2+72\eta^2+6\eta^4)+2\kappa\alpha^2 )}{\alpha^2 \eta}$, we obtain
\begin{equation*}
    \begin{split}
        &\mathbb{E}\Big[ \delta^{(t+1)}\Big]+D\cdot\mathbb{E}\Big[\lVert C^{(t+1)}\rVert^2\Big]+E\cdot\mathbb{E}\left[\lVert  P^{(t+1)} \rVert^2 \right] +F\cdot\mathbb{E}\Big[ \lVert M^{(t+1)}\rVert^2\Big]\\
        &\leq  (1-\gamma \mu) \mathbb{E}\Big[\delta^{(t)}\Big]+16\gamma\kappa\zeta^2+2\gamma\mathbb{E}\Big[\lVert \widetilde M^{(t)}\rVert^2\Big]\\
        &\quad- \Big(\frac{1}{2\gamma}-\frac{L}{2}-\frac{16\gamma\eta^2 (8\kappa+1)(\frac{3}{\alpha}L^2+\ell^2)}{\alpha}-\frac{8\gamma (8\kappa+1)(\frac{2}{\eta}L^2+\ell^2)(\alpha^2+24\eta^2)}{\eta\alpha^2}\Big)\mathbb{E}\Big[\lVert R^{(t)}\rVert^2\Big]\\
 &\quad + D\cdot(1-\frac{\alpha}{4})\mathbb{E}\left[\lVert  C^{(t)} \rVert^2 \right]+E\cdot(1-\frac{\eta}{2})\mathbb{E}\left[\lVert  P^{(t)} \rVert^2 \right]+\frac{8\eta\gamma\sigma^2(8\kappa+1)(\alpha^2+24\eta^2)}{\alpha^2}\\
 &\quad+\frac{8\gamma\Big((8\kappa+1)(3\alpha^2+72\eta^2+6\eta^4)+2\kappa\alpha^2 \Big)}{\alpha^2}\mathbb{E}\Big[ \lVert M^{(t)}\rVert^2\Big] + \frac{16\gamma\eta^4\sigma^2(8\kappa+1)}{\alpha} \\
 &\quad+F\cdot\Big((1-\eta)\mathbb{E} \left[  \lVert M^{(t)}  \rVert^2 \right]   + 2\ell^2 \mathbb{E} \left[  \lVert R^{(t)}  \rVert^2 \right] + 2\eta^2 \sigma^2 \Big)\\
 &\leq (1-\gamma \mu) \mathbb{E}\Big[\delta^{(t)}\Big]+ D\cdot(1-\frac{\alpha}{4})\mathbb{E}\left[\lVert  C^{(t)} \rVert^2 \right]+E\cdot(1-\frac{\eta}{2})\mathbb{E}\left[\lVert  P^{(t)} \rVert^2 \right]\\
 &\quad+\frac{8\eta\gamma\sigma^2(8\kappa+1)(\alpha^2+24\eta^2+2\alpha\eta^3)}{\alpha^2}+ F\cdot(1-\frac{\eta}{2})\mathbb{E}\Big[ \lVert M^{(t)}\rVert^2\Big]\\
 &\quad- \Bigg(\frac{1}{2\gamma}-\frac{L}{2}-\frac{16\gamma\eta^2 (8\kappa+1)(\frac{3}{\alpha}L^2+\ell^2)}{\alpha}-\frac{8\gamma (8\kappa+1)(\frac{2}{\eta}L^2+\ell^2)(\alpha^2+24\eta^2)}{\eta\alpha^2}\\
 &\quad-\frac{32\gamma\ell^2\Big((8\kappa+1)(3\alpha^2+72\eta^2+6\eta^4)+2\kappa\alpha^2 \Big)}{\eta\alpha^2}\Bigg)\mathbb{E}\Big[\lVert R^{(t)}\rVert^2\Big]+2\gamma\mathbb{E}\Big[\lVert \widetilde M^{(t)}\rVert^2\Big]\\
 &\quad+16\gamma\kappa\zeta^2+\frac{32\gamma\eta\sigma^2\Big((8\kappa+1)(3\alpha^2+72\eta^2+6\eta^4)+2\kappa\alpha^2 \Big)}{\alpha^2 }.
    \end{split}
\end{equation*}
Furthermore, by adding $H \cdot \mathbb{E}\Big[\lVert \widetilde M^{(t+1)}\rVert^2\Big]$, using Lemma \ref{lem:momentum_vr}, and substituting $H = \frac{4\gamma}{\eta}$, we arrive at
\begin{equation*}
    \begin{split}
        &\mathbb{E}\Big[ \delta^{(t+1)}\Big]+D\cdot\mathbb{E}\Big[\lVert C^{(t+1)}\rVert^2\Big]+E\cdot\mathbb{E}\left[\lVert  P^{(t+1)} \rVert^2 \right] +F\cdot\mathbb{E}\Big[ \lVert M^{(t+1)}\rVert^2\Big]+H\cdot \mathbb{E}\Big[\lVert \widetilde M^{(t+1)}\rVert^2\Big]\\
        &\leq (1-\gamma \mu) \mathbb{E}\Big[\delta^{(t)}\Big]+ D\cdot(1-\frac{\alpha}{4})\mathbb{E}\left[\lVert  C^{(t)} \rVert^2 \right]+E\cdot(1-\frac{\eta}{2})\mathbb{E}\left[\lVert  P^{(t)} \rVert^2 \right]+2\gamma\mathbb{E}\Big[\lVert \widetilde M^{(t)}\rVert^2\Big]\\
 &\quad- \Bigg(\frac{1}{2\gamma}-\frac{L}{2}-\frac{16\gamma\eta^2 (8\kappa+1)(\frac{3}{\alpha}L^2+\ell^2)}{\alpha}-\frac{8\gamma (8\kappa+1)(\frac{2}{\eta}L^2+\ell^2)(\alpha^2+24\eta^2)}{\eta\alpha^2}\\
 &\quad-\frac{32\gamma\ell^2\Big((8\kappa+1)(3\alpha^2+72\eta^2+6\eta^4)+2\kappa\alpha^2 \Big)}{\eta\alpha^2}\Bigg)\mathbb{E}\Big[\lVert R^{(t)}\rVert^2\Big]+ F\cdot(1-\frac{\eta}{2})\mathbb{E}\Big[ \lVert M^{(t)}\rVert^2\Big]+16\gamma\kappa\zeta^2\\
 &\quad+\frac{32\gamma\eta\sigma^2\Big((8\kappa+1)(3\alpha^2+72\eta^2+6\eta^4)+2\kappa\alpha^2 \Big)}{\alpha^2 }+\frac{8\eta\gamma\sigma^2(8\kappa+1)(\alpha^2+24\eta^2+2\alpha\eta^3)}{\alpha^2}\\
 &\quad +H\cdot\Big( (1-\eta)\mathbb{E} \left[  \lVert \widetilde M^{(t)}  \rVert^2 \right]   + 2\ell^2 \mathbb{E} \left[  \lVert R^{(t)}  \rVert^2 \right] + \frac{2\eta^2 \sigma^2}{G}\Big)\\
 &\leq (1-\gamma \mu) \mathbb{E}\Big[\delta^{(t)}\Big]+ D\cdot(1-\frac{\alpha}{4})\mathbb{E}\left[\lVert  C^{(t)} \rVert^2 \right]+E\cdot(1-\frac{\eta}{2})\mathbb{E}\left[\lVert  P^{(t)} \rVert^2 \right]+\frac{8\eta\gamma\sigma^2(8\kappa+1)(\alpha^2+24\eta^2+2\alpha\eta^3)}{\alpha^2}\\
 &\quad- \Bigg(\frac{1}{2\gamma}-\frac{L}{2}-\frac{16\gamma\eta^2 (8\kappa+1)(\frac{3}{\alpha}L^2+\ell^2)}{\alpha}-\frac{8\gamma (8\kappa+1)(\frac{2}{\eta}L^2+\ell^2)(\alpha^2+24\eta^2)}{\eta\alpha^2}-\frac{8\gamma \ell^2}{\eta}\\
 &\quad-\frac{32\gamma\ell^2\Big((8\kappa+1)(3\alpha^2+72\eta^2+6\eta^4)+2\kappa\alpha^2 \Big)}{\eta\alpha^2}\Bigg)\mathbb{E}\Big[\lVert R^{(t)}\rVert^2\Big]+ F\cdot(1-\frac{\eta}{2})\mathbb{E}\Big[ \lVert M^{(t)}\rVert^2\Big]\\
 &\quad+16\gamma\kappa\zeta^2+\frac{32\gamma\eta\sigma^2\Big((8\kappa+1)(3\alpha^2+72\eta^2+6\eta^4)+2\kappa\alpha^2 \Big)}{\alpha^2 } +H\cdot(1-\frac{\eta}{2})\mathbb{E} \left[  \lVert \widetilde M^{(t)}  \rVert^2 \right]   + \frac{8\gamma\eta \sigma^2}{G}.
    \end{split}
\end{equation*}
Finally, we bound $R^{(t)}$,
\begin{equation*}
    \begin{split}
        &\Bigg(\frac{1}{2\gamma}-\frac{L}{2}-\frac{8\gamma \ell^2}{\eta}-\frac{16\gamma\eta^2 (8\kappa+1)(\frac{3}{\alpha}L^2+\ell^2)}{\alpha}-\frac{8\gamma (8\kappa+1)(\frac{2}{\eta}L^2+\ell^2)(\alpha^2+24\eta^2)}{\eta\alpha^2}\\
 &\quad-\frac{32\gamma\ell^2\Big((8\kappa+1)(3\alpha^2+72\eta^2+6\eta^4)+2\kappa\alpha^2 \Big)}{\eta\alpha^2}\Bigg) \mathbb{E}\Big[\lVert R^{(t)}\rVert^2\Big]\geq 0.
    \end{split}
\end{equation*}
Let 
\begin{equation*}
\begin{split}
A \eqdef &\Bigg(\frac{16 \ell^2}{\eta}+\frac{16(8\kappa+1)(\frac{2}{\eta} L^2+\ell^2)(\alpha^2+24\eta^2)}{\eta\alpha^2}+\frac{64 \ell^2\Big((8\kappa+1)(3\alpha^2+72\eta^2+6\eta^4)+2\kappa\alpha^2 \Big)}{\eta\alpha^2 }\\
&\quad+\frac{32\eta^2(8\kappa+1)(\frac{3}{\alpha}L^2+\ell^2)}{\alpha}\Bigg)\\
&=16\Big(\frac{\ell^2}{\eta}+ \frac{(8\kappa+1)(\frac{2}{\eta}L^2+\ell^2)}{\eta}+\frac{24\eta(8\kappa+1)( \frac{2}{\eta}L^2+\ell^2)}{\alpha^2}+\frac{8\kappa\ell^2}{\eta}\\
&\quad+\frac{4 \ell^2(8\kappa+1)(3\alpha^2+72\eta^2+6\eta^4)}{\eta\alpha^2}+\frac{2\eta^2(8\kappa+1)(\frac{3}{\alpha}L^2+\ell^2)}{\alpha}\Big)\\
&=16\Big(\frac{2(8\kappa+1)\ell^2}{\eta}+\frac{2(8\kappa+1)L^2}{\eta^2}+\frac{24\eta(8\kappa+1)\ell^2}{\alpha^2}+\frac{48(8\kappa+1)L^2}{\alpha^2}\\
&\quad+\frac{12(8\kappa+1)\ell^2}{\eta}+\frac{4 \ell^2(8\kappa+1)(72\eta+6\eta^3)}{\alpha^2}+\frac{6\eta^2(8\kappa+1)L^2}{\alpha^2}+\frac{2\eta^2(8\kappa+1)\ell^2}{\alpha}\Big)\\
&=16\Big( \frac{(8\kappa+1)( 2 L^2+14\eta\ell^2)}{\eta^2}+\frac{2(8\kappa+1)((156\eta+12\eta^3+\alpha\eta^2)\ell^2+(3\eta^2+24)L^2)}{\alpha^2}\Big).
\end{split}
\end{equation*}
Taking $0 < \gamma \leq \frac{1}{L+\sqrt{A}}$ and applying Lemma \ref{lemma:step_lemma} gives
\begin{equation*}
    \frac{1}{2\gamma}-\frac{L}{2}-\frac{\gamma A}{2}\geq 0.
\end{equation*}

Let $\mathbb{E}\Big[\Psi^{(t+1)}\Big] \eqdef\mathbb{E}\Big[ \delta^{(t+1)}\Big]+D\cdot\mathbb{E}\Big[\lVert C^{(t+1)}\rVert^2\Big]+E\cdot\mathbb{E}\Big[\lVert  P^{(t+1)} \rVert^2 \Big]+F\cdot\mathbb{E}\Big[ \lVert M^{(t+1)}\rVert^2\Big]+H\cdot \mathbb{E}\Big[\lVert \widetilde M^{(t+1)}\rVert^2\Big]$ and we use the assumption on $\eta$ and $\alpha$ to establish that $1 - \frac{\alpha}{4} \leq 1 - \gamma\mu$ and $1 - \frac{\eta}{2} \leq 1 - \gamma\mu$. Applying the inequality iteratively gives
\begin{equation*}
    \begin{split}
        &\mathbb{E}\Big[\Psi^{T}\Big] \\
        &\leq (1-\gamma \mu)^T \mathbb{E}\Big[\Psi^{0}\Big]+\frac{8\eta\sigma^2(8\kappa+1)(\alpha^2+24\eta^2+2\alpha\eta^3)}{\mu\alpha^2}+\frac{32\eta\sigma^2(8\kappa+1)(3\alpha^2+72\eta^2+6\eta^4)}{\mu\alpha^2 }\\
        &\quad+\frac{64\eta\sigma^2\kappa}{\mu}+\frac{16\kappa\zeta^2}{\mu} +\frac{8\eta \sigma^2}{\mu G}\\
        &\leq (1-\gamma \mu)^T \mathbb{E}\Big[\Psi^{0}\Big]+\frac{8\eta\sigma^2(8\kappa+1)}{\mu}+\frac{16\eta^3\sigma^2(8\kappa+1)(\alpha\eta+12)}{\mu\alpha^2}+\frac{96\eta\sigma^2(8\kappa+1)}{\mu}+\frac{64\eta\sigma^2\kappa}{\mu}\\
        &\quad+\frac{192\eta^3\sigma^2(8\kappa+1)(\eta^2+12)}{\mu\alpha^2}+\frac{16\kappa\zeta^2}{\mu} +\frac{8\eta \sigma^2}{\mu G}\\
        &\leq (1-\gamma \mu)^T \mathbb{E}[\Psi^{0}]+\frac{8\eta\sigma^2(108\kappa G+13G+1)}{\mu G}+\frac{16\eta^3\sigma^2(8\kappa+1)(12\eta^2+\alpha\eta+156)}{\mu\alpha^2}+\frac{16\kappa\zeta^2}{\mu}.
    \end{split}
\end{equation*}
Noting that $\mathbb{E}\Big[\Psi^T\Big]\geq \mathbb{E}\Big[ f(x^T) -f(x^*)\Big]$, we finish the proof.
    \end{proof}

\begin{corollary}\label{VRDm21_PL}
    Suppose that assumptions from Theorem \ref{thm:general_pl_DM21} hold, momentum $\eta \leq \min\Big\{ \frac{\mu G\varepsilon}{(G(\kappa+1)+1)\sigma^2},\frac{\mu\alpha^2\varepsilon^{\nicefrac{1}{3}}}{(\kappa+1)\sigma^2},\frac{\mu\alpha\varepsilon^{1/4}}{(\kappa+1)\sigma^2}\Big\}$, and for all $i\in\mathcal{G}$, then Algorithm \ref{alg:SGD2M} needs
    \begin{equation}\label{complexity_VRDM21}
        T = \widetilde{\mathcal{O}}\Bigg( \frac{(G(\kappa+1)+1)\sigma^2}{\mu\varepsilon G} + \frac{(\kappa+1)\sigma^2}{\mu\alpha^2\varepsilon^{\nicefrac{1}{3}}}+\frac{(\kappa+1)\sigma^2}{\mu\alpha\varepsilon^{\nicefrac{1}{4}}}+\frac{L}{\mu}+\frac{\sigma\sqrt{(\ell^2+L^2)(G(\kappa+1)+1)(\kappa+1)}}{\mu^{\nicefrac{3}{2}} \varepsilon^{\nicefrac{1}{2}}G^{\nicefrac{1}{2}}} \Bigg)
    \end{equation}
    communication rounds to get an $\varepsilon$-solution.
\end{corollary}
\begin{proof}
    Considering the choice of $\eta$, we have $\frac{1}{\mu}\Big(\frac{8\eta(108\kappa G+13G+1)}{ G}+\frac{16\eta^3(8\kappa+1)( 12\eta^2+\alpha\eta+156)}{\alpha^2}\Big)\sigma^2 = \mathcal{O}(\varepsilon)$, which guarantees that $\mathbb{E}\Big[f({x}^{(T)})-f(x^*)\Big]\leq \varepsilon$ for $\varepsilon\geq \frac{32\kappa\zeta^2}{\mu }$. Therefore, is it sufficient to take the number of communication rounds equal \eqref{complexity_VRDM21} to get an $\varepsilon$-solution.
\end{proof}

\end{document}